\documentclass{article}
\usepackage{arxiv}

\usepackage[utf8]{inputenc}
\usepackage{textcomp}
\usepackage{microtype}
\usepackage{graphicx}
\usepackage{booktabs} 
\usepackage{hyperref}

\usepackage[utf8]{inputenc}
\usepackage{xcolor}
\usepackage{fancybox}
\usepackage{graphicx}
\usepackage{amsmath}
\usepackage{amsfonts}
\usepackage{amssymb}
\usepackage{gensymb}
\usepackage{mathtools}
\usepackage{eurosym, booktabs, units}
\usepackage{pgfgantt}
\usepackage{enumitem}
\usepackage{colortbl}
\usepackage{tabularx}
\usepackage{caption}
\usepackage{subcaption}
\usepackage{float}
\usepackage{algorithmicx}
\usepackage{algorithm}
\usepackage{algpseudocode}
\usepackage{siunitx}
\usepackage{ bbold }
\usepackage{nicefrac}       
\usepackage{amsthm}

\theoremstyle{definition}
\newtheorem{definition}{Definition}[section]
\theoremstyle{remark}
\newtheorem*{remark}{Remark}

\newtheorem{theorem}{Theorem}
\newtheorem{lemma}{Lemma}

\newcommand{\displaycomments}
\ifdefined\displaycomments

\newif\ifAddInBlue
\AddInBluetrue   

\newif\ifShowDel

\newtheorem{corollary}[theorem]{Corollary}

\title{Dissipative Deep Neural Dynamical Systems}

\author{J\'an Drgo\v na, Aaron Tuor, Soumya Vasisht, Draguna Vrabie \\
    Pacific Northwest National Laboratory\\
	Richland, Washington USA\\
	\{first.last\}@pnnl.gov
}

\begin{document}

\maketitle
\begin{abstract}
In this paper, we provide sufficient conditions for dissipativity 
and local asymptotic stability of discrete-time dynamical systems parametrized by deep neural networks.
We leverage the representation of neural networks as pointwise affine maps, thus exposing their local linear operators and making them accessible to classical system analytic and design methods. This allows us to ``crack open the black box'' of the neural dynamical system's behavior by evaluating their dissipativity, and estimating their stationary points and state-space partitioning. We relate the norms of these local linear operators to the energy stored in the dissipative system with supply rates represented by their aggregate bias terms.
Empirically, we analyze the variance in dynamical behavior and eigenvalue spectra of these local linear operators with varying weight factorizations, activation functions, bias terms, and depths.
\end{abstract}

\maketitle

\section{Introduction}

In recent years, deep neural networks (DNN) have become ever more integrated into safety-critical and high-performance systems, including robotics, autonomous driving, and process control, where formal verification methods are desired to ensure safe operation.
Thus the need for guarantees on the behavior of DNNs fuels the recent research on safe learning-based systems. 
Dissipativity is an important property of non-linear dynamical systems that have been shown to provide connections with stability, reachability, and controllability properties, that have paramount importance in the control theory~\cite{Byrnes94,WILLEMS2007134}. 

In this work, we provide an analytical method that allows us to interpret discrete-time neural dynamical systems through the
dissipativity theory perspective.
In particular, we leverage the equivalence of DNNs with pointwise affine maps (PWA) that allows us to compute the operator norms and aggregate bias terms of deep neural dynamics with arbitrary activation functions and unlimited depth. We show that we can interpret the norms of the local linear operators as energy stored in the system with norms of the aggregate bias terms representing the supply rates.
In summary, we report  the following contributions:
\begin{itemize}
    \item Method for obtaining local linear operators of deep neural networks.
    \item Sufficient conditions for dissipativity of discrete-time deep neural dynamical system.
    \item A set of design methods for enforcing dissipativity of deep neural dynamical systems.
    \item Case study analyzing the influence of network depth, constituent linear maps, and activation functions on the dissipativity and overall dynamics of deep neural networks.
\end{itemize}

\section{Related Work}

In the last decades, dissipativity theory~\cite{WILLEMS2007134} has been extensively used in control theory to analyze the stability of a wide range of dynamical systems~\cite{HILL1980327} including
nonlinear feedback systems~\cite{Chellaboina2005}, 
stochastic dynamical systems~\cite{Rajpurohit2017},
and passive  systems~\cite{Kottenstette2010}.
Moreover, dissipativity plays a crucial role in proving the stability of economic model predictive control~\cite{Rawlings2012,Diehl2011,Muller2015}.
Most recently, authors in~\cite{Sosanya2022} in the deep learning community proposed a dissipative neural network architecture based on the Hamiltonian dynamics. 
In this work, we provide a connection between dissipativity theory and discrete-time dynamical systems parametrized by fully connected deep neural networks.

An alternative path to tractable analysis of nonlinear systems is to cast them in a linear formulation. 
Prior to DNNs coming of age,
linear parameter-varying systems (LPV)~\cite{shamma1990LPV} and linear differential inclusions (LDI)~\cite{boyd1994LDI} have been used frequently in analysis of nonlinear systems. Stability analysis of LDI  was previously applied to neural networks for system and control design problems ~\cite{tanaka1996ldi, matusik2020ldiNN3, he2014ldiNN,limanond1998ldiNN2}.
Linear analysis of neural networks has principally focused on networks with ReLU activations. \cite{arora2016ReLuDNN} prove the equivalence of deep ReLU networks with piecewise affine (PWA) maps.
\cite{hanin2019complexity, hanin2019deep} show how to compute the number of linear regions of DNNs with piecewise linear activation functions. 
\cite{Wang2016,gehr2018ai2} interpret ReLU networks as pointwise linearizations, enabling the exploration of their spectral properties.
\cite{Robinson2020} present an algorithm for computing PWA forms of deep ReLU networks.
In contrast, our presented work provides exact pointwise affine forms of DNNs for a general class of activation functions opening linear analysis methods to a wider range of DNN architectures.

A host of works have begun to view neural networks from a dynamical systems perspective leading to new regularizations \cite{ludwig2014eigenvalue,Schmidt2021}, architectures~\cite{IMEXnet2019,NAISnet2018,HamiltonianDNNs2019,cranmer2020lagrangian,Rusch2021}, analysis methods \cite{engelken2020lyapunov, vogt2020lyapunov, Wang2016, gler2019robust}, and stability guarantees for some architectures \cite{haber2017stable}. 
\cite{NIPS2019_9292} propose jointly learning a non-linear dynamics model and Lyapunov function that guarantees non-expansiveness of the dynamics.
Authors in~\cite{Wang2016,gehr2018ai2} interpret ReLU networks as pointwise linearizations allowing them to explore the spectral properties. 
Other works perform eigenvalue analysis of neural network's Hessian \cite{ghorbani2019investigation, le1991eigenvalues} and Gram \cite{goel2017eigenvalue} matrices, providing insight into the optimization dynamics which they use to develop more efficient learning algorithms.

Stability of DNNs has been studied for several years in the context of neuro-controllers  \cite{vrabie2009neural, vamvoudakis2010neuroCntrl, vamvoudakis2015NNfbCntrl}. 
More recently, authors in~\cite{Pennington46342,louart2017random,liao2018dynamics} study the eigenvalues of the data covariance matrix propagation through a single layer neural network from a perspective of random matrix theory.
 \cite{Kozachkov2020} provide a
control-theory perspective by applying contraction analysis to the Jacobians of recurrent neural networks (RNN) for studying emergence of stable dynamics in neural circuits.
 \cite{Revay2019, Revay2021} use contraction analysis to
design an implicit model structure that allows for a convex parametrization of stable RNN models.
\cite{Pauli9319198} propose a method for training deep feedforward neural network with bounded Lipschitz constants.
\cite{Fazlyab2019} 
pose the Lipschitz constant estimation problem for deep neural networks as a semidefinite program (SDP).
While~\cite{erichson2021lipschitz} shows how to train continuous-time RNN with constrained Lipschitz constants to guarantee stability.
The stability and attractors of RNNs have been studied in continuous time~\cite{Zhang2014}
as well as from a neuroscience perspective~\cite{laje_robust_2013}.
\cite{bonassi2020lstm} analyze Input-to-State (ISS)  stability of LSTM networks by recasting them in the state space form.


Various parametrizations and auxilliary loss terms have been proposed to restrict the eigenvalues of a neural network's  layer weights, $\mathbf{W}$. 
Some authors use regularization to minimize eigenvalues of $\mathbf{W}\mathbf{W^{\intercal}}$, e.g. \cite{ludwig2014eigenvalue}. Other works bound the singular values of layer weights directly via orthogonal \cite{mhammedi2017efficient}, spectral \cite{zhang2018stabilizing}, Perron-Frobenius~\cite{tuor2020constrained}, or symplectic (antisymmetric)~\cite{haber2017stable,chang2019antisymmetricrnn} parametrizations. 
 \cite{rajan2006random}  present a family of matrices with eigenvalues constrained within a circle with prescribed radius. 
\cite{lechner2020gershgorin} introduce a Gershgorin disc based regularization to ensure negative eigenvalues on the weights and prove that this regularization ensures stability.

\section{Methods} 

Our primary objective is to design provably stable yet expressive
discrete-time dynamical systems parametrized by deep neural networks.
We show that representing nonlinear activation functions as  state-dependent diagonal matrices allows us to decompose the neural network into a composition of pointwise affine maps (PWA).
We leverage this equivalence for computing the operator norms of DNNs that allows us to analyze and constrain the dissipativity leading to fixed point stability of the proposed discrete-time neural dynamics.

\subsection{Deep Neural Networks as Pointwise Affine Maps} \label{sec:dnn_lpvm}
In this section we give a formulation of a deep neural network as a pointwise affine (PWA) map. 
Consider a deep neural network
 $\mathbf{f}_{\mathbf{\theta}}: \mathbb{R}^m \rightarrow \mathbb{R}^n$ parametrized by $\theta = \{\mathbf{W}_0, \mathbf{b}_0, \ldots, \mathbf{W}_{L}, \mathbf{b}_L\}$ with hidden layers $1\leq l\leq L$  given by: 
 \label{def:dnn}
\begin{subequations}
    \label{eq:dnn}
    \begin{align}
    \mathbf{f}_{\theta}(\mathbf{x}) & =  \mathbf{W}_{L} \mathbf{h}_L + \mathbf{b}_L \\
    \mathbf{h}_{l} &= \boldsymbol\sigma(\mathbf{W}_{l-1} \mathbf{h}_{l-1} + \mathbf{b}_{l-1})  \\
    \mathbf{h}_0 &= \mathbf{x}
 \end{align}
\end{subequations}
where $\boldsymbol\sigma: \mathbb{R}^n \rightarrow \mathbb{R}^n$ represents an elementwise application of a univariate scalar activation function $\sigma: \mathbb{R} \rightarrow \mathbb{R}$ to vector elements such that $\mathbf{\boldsymbol\sigma}(\mathbf{z}) = \begin{bmatrix}\sigma(z_1)\hdots \sigma(z_n)\end{bmatrix}^{\intercal}$, where $\mathbf{z}_l = \mathbf{W}_{l} \mathbf{h}_{l} + \mathbf{b}_{l}$.

\begin{lemma}~\cite{drgona2021stochastic}
Let $\mathbf{f}_{\mathbf{\theta}}$~\eqref{eq:dnn} be a deep neural network with activation function $\boldsymbol\sigma$, then there exists a pointwise affine map $\mathbf{A}^{\star}(\mathbf{x})\mathbf{x} +  \mathbf{b}^{\star}(\mathbf{x})$ parametrized by $\mathbf{x}$ which satisfies the following:
\label{lem:lpv}
\begin{equation}
\label{eq:lpv_dnn}
        \mathbf{f}_{\theta}(\mathbf{x}) := \mathbf{A}^{\star}(\mathbf{x})\mathbf{x} +  \mathbf{b}^{\star}(\mathbf{x})
\end{equation}
where $\mathbf{A}^{\star}(\mathbf{x})$ is a state-dependent  matrix given as:
\begin{equation}
  \label{eq:dnn_LPV}
    \mathbf{A}^{\star}(\mathbf{x})\mathbf{x}= \mathbf{W}_L \boldsymbol\Lambda_{\mathbf{z}_{L-1}} \mathbf{W}_{L-1} \ldots \boldsymbol \Lambda_{\mathbf{z}_{0}} \mathbf{W}_{0} \mathbf{x}
 \end{equation}
and $\mathbf{b}^{\star}(\mathbf{x})$ is a state-dependent  vector given as:
  \begin{eqnarray}
  \label{eq:dnn_bias_recurence}
  \mathbf{b}^{\star}(\mathbf{x}) =  \mathbf{b}^{\star}_{L}, \ \
  \mathbf{b}^{\star}_{l} := \mathbf{W}_i \boldsymbol\Lambda_{\mathbf{z}_{l-1}} \mathbf{b}^{\star}_{l-1}  + \\ \mathbf{W}_i\boldsymbol\sigma_{l-1}(\mathbf{0}) +\mathbf{b}_{l}, \ \ l \in \mathbb{N}_1^{L}
 \end{eqnarray}
  with $ \mathbf{b}^{\star}_{0} = \mathbf{b}_{0}$, and  $i$ representing index of the network layer.
 Here $\boldsymbol \Lambda_{\mathbf{z}_{l}}$ represents a diagonal matrix of activation patterns dependent on a hidden states $\mathbf{z}_l$ at $l$-th layer defined as:
   \begin{subequations}
\label{eq:lambda_matrix}
\begin{align}
\boldsymbol\sigma(\mathbf{z})  & =   \boldsymbol \Lambda_{\mathbf{z}}  \mathbf{z} + \boldsymbol\sigma(\mathbf{0})  \\ 
\boldsymbol\sigma(\mathbf{z}) & =   \begin{bmatrix}
   \frac{\sigma(z_1) -\sigma(0)}{z_1} &  & \\
    & \ddots & \\
     &  &  \frac{\sigma(z_n)-\sigma(0)}{z_n} 
  \end{bmatrix}\mathbf{z} + \begin{bmatrix}\sigma(0)\\\vdots \\\sigma(0)\end{bmatrix} 
  \end{align}
   \end{subequations}
\end{lemma}

\subsection{Dissipative Deep Neural Dynamical Systems }
Consider the following discrete-time autonomous deep neural dynamical system:
  \begin{equation}
  \label{eq:neural_dynamics}
        \mathbf{x}_{t+1} = \mathbf{f}_{\theta}(\mathbf{x}_t)  
 \end{equation}
where $\mathbf{x}_t \in \mathbf{R}^{n_x}$ are system states,
and $\mathbf{f}_{\theta}: \mathbb{R}^{n_x} \to  \mathbb{R}^{n_x}$ is a deep neural network~\eqref{eq:dnn}.
For the simplicity of the analysis throughout this paper we assume  fully observable dynamics.

Next, we leverage the pointwise affine (PWA) reformulation~\eqref{eq:lpv_dnn}  for stability analysis of the neural dynamics~\eqref{eq:neural_dynamics} through the perspective of local linear operators and  dissipative systems theory.
In particular, using the PWA form~\eqref{eq:lpv_dnn} leads to a square state-dependent matrix $\mathbf{A}^{\star}(\mathbf{x}) \in \mathbb{R}^{n_x \times n_x}$ and bias term $\mathbf{b}^{\star}(\mathbf{x}) \in \mathbb{R}^{n_x}$ that are amenable to local linear dynamics analysis of the nonlinear dynamical system~\eqref{eq:neural_dynamics}. 


\begin{remark} 
Note, that the hidden layers of the network $\mathbf{f}_{\theta}(\mathbf{x}) \in  \mathbb{R}^{n_x \times n_x}$ may still be constructed using non-square weights $\mathbf{W}_l \in  \mathbb{R}^{n_x \times m}$, $n_x \ne m$, thus allowing for increased expressivity of the map $\mathbf{f}_{\theta}(\mathbf{x}_t)$. \end{remark}

\begin{definition} \textit{Dissipative Discrete-time Dynamical System~\cite{Byrnes94}.}
A discrete-time dynamical system~\eqref{eq:neural_dynamics} is said to be \textit{dissipative} if the following condition holds:
\begin{equation}
\label{eq:dissipativity}
  \mathbf{V}(\mathbf{x}_{t+1}) -  \mathbf{V}(\mathbf{x}_t) \le 
   \mathbf{s}(\mathbf{x}_t), \ \forall t \in \{0, 1, 2, \ldots \}
\end{equation}
Where $ \mathbf{V}(\mathbf{x}_t): \mathbb{R}^{n_x} \to  \mathbb{R}$
such that $ \mathbf{V}(0) = 0$, and
$ \mathbf{V}(\mathbf{x}_t) \ge 0 $ represents a non-negative storage function quantifying the energy stored internally in the system, and $\mathbf{s}(\mathbf{x}_t) : \mathbb{R}^{n_x} \to  \mathbb{R}$
is the so-called supply rate representing energy supplied to the system from the external environment.
\end{definition}

\begin{remark} 
The discrete-time dissipativity condition~\eqref{eq:dissipativity} is an extension of the discrete-time Lyapunov stability condition:
\begin{equation}
\label{eq:lyapunov_discrete}
     \mathbf{V}(\mathbf{x}_{t+1}) -  \mathbf{V}(\mathbf{x}_t) \le 0
\end{equation}
 defined for closed systems, i.e., with zero supply rate.
\end{remark}

Now we formulate the main results of the paper.
\begin{theorem}
\label{thm:dissipative_dnn}
\textit{Dissipative Deep Neural Dynamical Systems:}
 the neural dynamical system~\eqref{eq:neural_dynamics}  is dissipative over a state-space region $\mathcal{X} \subseteq \mathbb{R}^{n_x}$ with respect to the supply rate $\mathbf{s}(\mathbf{x}_t)  =  ||\mathbf{b}^{\star}(\mathbf{x}_t) ||_2  $ if 
 the local linear dynamics $  \|\mathbf{A}^{\star}(\mathbf{x}) \|_2$ of the equivalent PWA form~\eqref{eq:lpv_dnn} is a contractive map over the entire region $\mathcal{X} $. Or more formally the following must hold:
\begin{equation}
        \label{eq:dissipatitvity_condition_dnn}
        \|\mathbf{A}^{\star}(\mathbf{x}) \|_2 < 1, \  \
      \forall \mathbf{x} \in \mathcal{X} \subseteq \mathbb{R}^{n_x}.
    \end{equation} 
\end{theorem}
\begin{proof}
Consider the dissipativity condition~\eqref{eq:dissipativity} with a choosen storage function $\mathbf{V}(\mathbf{x}) = \sqrt{\mathbf{x}^T \mathbf{x}}$ and supply rate $\mathbf{s}(\mathbf{x}_t)  =  ||\mathbf{b}^{\star}(\mathbf{x}_t) ||_2  $
we will prove the following dissipativity condition:
\begin{equation}
\label{eq:dnn_dissipativity_quadratic_bias_supply}
||\mathbf{x}_{t+1}||_2 - ||\mathbf{x}_t||_2  \le   ||\mathbf{b}^{\star}(\mathbf{x}_t) ||_2 
\end{equation}
Leveraging the equivalence of DNN with PWA~\eqref{eq:lpv_dnn} we get:
\begin{equation}
\label{eq:PWA_dnn_bias}
      \mathbf{x}_{t+1}  =   \mathbf{A}^{\star}(\mathbf{x}_t)\mathbf{x}_t +  \mathbf{b}^{\star}(\mathbf{x}_t) 
\end{equation}
Applying $2$-norms to~\eqref{eq:PWA_dnn_bias} we get:
\begin{equation}
\label{eq:PWA_dnn_bias_norm}
       ||\mathbf{x}_{t+1}||_2   =   ||\mathbf{A}^{\star}(\mathbf{x}_t)\mathbf{x}_t +  \mathbf{b}^{\star}(\mathbf{x}_t) ||_2
\end{equation}
Then applying norm subadditivity~\eqref{eq:operator_norm_subadd} and  submultiplicativity~\eqref{eq:operator_norm_submultiplicative} of the norms we have:
\begin{equation}
\label{eq:PWA_dnn_bias_norm_2}
      ||\mathbf{x}_{t+1}||_2   \le   ||\mathbf{A}^{\star}(\mathbf{x}_t)||_2 ||\mathbf{x}_t||_2 +  ||\mathbf{b}^{\star}(\mathbf{x}_t) ||_2 
\end{equation}
We can substitute~\eqref{eq:PWA_dnn_bias_norm_2} into the 
dissipativity condition~\eqref{eq:dnn_dissipativity_quadratic_bias_supply}:
\begin{equation}
\label{eq:dnn_dissipativity_quadratic_bias_supply_1}
||\mathbf{A}^{\star}(\mathbf{x}_t)||_2 ||\mathbf{x}_t||_2 +  ||\mathbf{b}^{\star}(\mathbf{x}_t) ||_2  - ||\mathbf{x}_t||_2  \le   ||\mathbf{b}^{\star}(\mathbf{x}_t) ||_2 
\end{equation}
Leading to:
\begin{equation}
\label{eq:dnn_dissipativity_quadratic_bias_supply_2}
||\mathbf{A}^{\star}(\mathbf{x}_t)||_2 ||\mathbf{x}_t||_2   - ||\mathbf{x}_t||_2  \le   0
\end{equation}
Now its clear that the
condition~\eqref{eq:dissipatitvity_condition_dnn} must hold $\forall \mathbf{x}_t \in \mathcal{X}$ to satisfy the  dissipativity~\eqref{eq:dnn_dissipativity_quadratic_bias_supply_2}  locally over the set $\mathcal{X}$.
\end{proof}

\begin{corollary}
\label{theorem:glbl_stable}
\textit{Local Asymptotic Stability of Deep Neural Dynamics:}
 system~\eqref{eq:neural_dynamics} parametrized by deep neural networks is  
 locally asymptotically stable towards the origin $\bar{\mathbf{x}} = \mathbf{0}$ belonging to the interior of $\mathcal{X}$, if the equivalent PWA form~\eqref{eq:lpv_dnn} of the system~\eqref{eq:neural_dynamics} is strictly contractive over the region $\mathcal{X}$. Or more formally the following must hold:
\begin{equation}
        \label{eq:dissipatitvity_condition_dnn_2}
        \|\mathbf{A}^{\star}(\mathbf{x}) \|_2 < 1 - \frac{ \|\mathbf{b}^{\star}(\mathbf{x}) \|_2}{\|\mathbf{x}\|_2} , \  \
      \forall \mathbf{x}  \in \mathcal{X} \backslash  \{0\}, \ \bar{\mathbf{x}}  \in \mathcal{X}
    \end{equation} 
\end{corollary}
\begin{proof}
For dynamics~\eqref{eq:neural_dynamics} to be
asymptotically stable, the state must converge to a fixed-point steady state:
\begin{equation}
\label{eq:equilibrium}
  \bar{\mathbf{x}} = \mathbf{f}_{\theta}(\bar{\mathbf{x}}) =  \lim_{t \to \infty} \mathbf{f}_{\theta}({\mathbf{x}}_t)
\end{equation} 
To guarantee asymptotic stability towards the origin $\bar{\mathbf{x}} = \mathbf{0} \in \mathcal{X}$  we can consider the following contraction condition:
\begin{equation}
\label{eq:dnn_contraction_1}
  ||\mathbf{x}_{t+1}||_2  \le c ||\mathbf{x}_t||_2 
\end{equation}
With the contraction constant $c \in [0, 1)$. As shown in~\cite{BFb0109870} it is straightforward to see that the contraction condition~\eqref{eq:dnn_contraction_1} for $c < 1$
is equivalent to a discrete time Lyapunov condition~\eqref{eq:lyapunov_discrete} with  Lyapunov function $\mathbf{V}(\mathbf{x}) = \sqrt{\mathbf{x}^T \mathbf{x}}$ leading to:
\begin{equation}
\label{eq:dnn_contraction_2}
||\mathbf{x}_{t+1}||_2 - ||\mathbf{x}_t||_2  < 0
\end{equation}


To satisfy the contraction condition~\eqref{eq:dnn_contraction_2} for a neural dynamical system~\eqref{eq:neural_dynamics}, lets take the form~\eqref{eq:PWA_dnn_bias_norm_2} and divide the expression
by $||\mathbf{x}_t||_2$ leading to:
\begin{equation}
\label{eq:PWA_dnn_bias_norm_3}
   c = \frac{||\mathbf{x}_{t+1}||_2}{||\mathbf{x}_t||_2}  \le   ||\mathbf{A}^{\star}(\mathbf{x}_t)||_2  +   \frac{||\mathbf{b}^{\star}(\mathbf{x}_t) ||_2 }{||\mathbf{x}_t||_2}
\end{equation}
Now it is clear that to  to satisfy the contraction~\eqref{eq:dnn_contraction_2} with the steady state $\bar{\mathbf{x}} = \mathbf{0} \in \mathcal{X}$
 the following must hold: 
\begin{equation}
\label{eq:PWA_dnn_bias_norm_4}
   c = \frac{||\mathbf{x}_{t+1}||_2}{||\mathbf{x}_t||_2}  \le   ||\mathbf{A}^{\star}(\mathbf{x}_t)||_2  +   \frac{||\mathbf{b}^{\star}(\mathbf{x}_t) ||_2 }{||\mathbf{x}_t||_2} <1, \  \forall \mathbf{x}  \in \mathcal{X} \backslash \{ 0 \}.
\end{equation}
\end{proof}

\begin{remark}
The condition~\eqref{eq:PWA_dnn_bias_norm_4} implies that for asymptotically stable deep neural dynamics~\eqref{eq:neural_dynamics} 
with steady state at the origin $\lim_{t \to \infty} \mathbf{f}_{\theta}({\mathbf{x}}_t) = \bar{\mathbf{x}} = \mathbf{0}$ the energy of the bias term  needs to vanish to zero,
i.e., $\lim_{t \to \infty} ||\mathbf{b}^{\star}(\mathbf{x}_t) ||_2 = 0$.
\end{remark}


\begin{remark}
Without loss of generality, the condition~\eqref{eq:dissipatitvity_condition_dnn_2} assumes steady state at the origin $\bar{\mathbf{x}} = \mathbf{0} \in \mathcal{X}$.
One can apply simple coordinate transformation
$\mathbf{y} = \mathbf{x} + \mathbf{x}_{\text{s}}$
to model deep neural dynamics~\eqref{eq:neural_dynamics} with arbitrary valued fixed points $\mathbf{x}_{\text{s}} \neq 0  \subseteq \mathbb{R}$.
\end{remark}

\begin{corollary}
\label{thm:DNN_eq_bounded}
Assume the conditions of Theorem~\ref{thm:dissipative_dnn} with bounded supply rate and deep neural dynamics~\eqref{eq:neural_dynamics} that converge to equilibrium $\bar{\mathbf{x}}$~\eqref{eq:equilibrium}. 
Then given the system~\eqref{eq:neural_dynamics}, there exists an equilibrium point $\mathbf{x}_{\text{lb}} \le || \bar{\mathbf{x}} ||_p \le \mathbf{x}_{\text{ub}}$ with the bounds:
\begin{equation}
\label{eq:stable_eq}
  \mathbf{x}_{\text{lb}} = \frac{ \|\mathbf{b}^{\star}(\mathbf{x})\|_p }{||\mathbf{I} - \mathbf{A}^{\star}(\mathbf{x})||_p }, \ \  \mathbf{x}_{\text{ub}} = \frac{  \|\mathbf{b}^{\star}(\mathbf{x})\|_p }{1- \|\mathbf{A}^{\star}(\mathbf{x}) \|_p}.
\end{equation}
\end{corollary}
%

\begin{proof}
Now by substitution of the PWA form~\eqref{eq:lpv_dnn}
of the DNN~\eqref{eq:dnn} into~\eqref{eq:equilibrium}
for $\mathbf{x} = \bar{\mathbf{x}}$ we get:
\begin{equation}
\label{eq:DMM_equilibrium_pwa}
  {\mathbf{x}} = \mathbf{A}^{\star}(\mathbf{x}) {\mathbf{x}} + \mathbf{b}^{\star}(\mathbf{x})
\end{equation}
Where the matrix  $\mathbf{A}^{\star}(\mathbf{x})$ and bias vector $\mathbf{b}^{\star}(\mathbf{x})$ uniquely define the affine equilibrium dynamics.

To obtain the upper bound of the equilibrium,
we apply operator norm
to equation~\eqref{eq:DMM_equilibrium_pwa} leading to:
\begin{equation}
\label{eq:DMM_eq_norm}
 || {\mathbf{x}} ||_p = ||  \mathbf{A}({\mathbf{x}}) {\mathbf{x}} + \mathbf{b}({\mathbf{x}}) ||_p
\end{equation}
Then applying triangle inequality and operator upper bound 
$ ||\mathbf{A} \mathbf{x}||_p \le  ||\mathbf{A}||_p ||\mathbf{x}||_p $
we get:
\begin{equation}
\label{eq:DMM_eq_norm_subadd}
 || {\mathbf{x}} ||_p \le ||  \mathbf{A}^{\star}(\mathbf{x}) ||_p  ||{\mathbf{x}} ||_p + || \mathbf{b}^{\star}(\mathbf{x}) ||_p
\end{equation}
By applying straightforward algebra we have:
\begin{equation}
\label{eq:DMM_eq_norm_subadd_2}
 (1-||  \mathbf{A}^{\star}(\mathbf{x})||_p ) || {\mathbf{x}} ||_p \le   || \mathbf{b}^{\star}(\mathbf{x}) ||_p
\end{equation}
With resulting equilibrium upper bound given as:
\begin{equation}
\label{eq:DMM_eq_upper}
    ||{\mathbf{x}} ||_p \le \frac{||\mathbf{b}^{\star}(\mathbf{x}) ||_p}{ 1- || \mathbf{A}^{\star}(\mathbf{x}) ||_p  }
\end{equation}

For deriving the lower bound, we start  
with straightforward algebraic operations on~\eqref{eq:DMM_equilibrium_pwa} to obtain:
\begin{equation}
\label{eq:DMM_equilibrium_dyn}
 (\mathbf{I} -\mathbf{A}^{\star}(\mathbf{x})) {\mathbf{x}} =  \mathbf{b}^{\star}(\mathbf{x})) 
\end{equation}
For two equivalent vectors their  norms must be equal:
\begin{equation}
\label{eq:DMM_equ_norm}
  ||(\mathbf{I} - \mathbf{A}^{\star}(\mathbf{x})) {\mathbf{x}} ||_p  = ||\mathbf{b}^{\star}(\mathbf{x}) ||_p
\end{equation}
Now applying operator norm upper bound inequality 
$ ||\mathbf{A} \mathbf{x}||_p \le  ||\mathbf{A}||_p ||\mathbf{x}||_p $ to~\eqref{eq:DMM_equ_norm}
we have:
\begin{align}
\label{eq:DMM_eq_lower}
      ||(\mathbf{I} - \mathbf{A}^{\star}(\mathbf{x})) ||_p ||{\mathbf{x}} ||_p  \ge ||\mathbf{b}^{\star}(\mathbf{x}) ||_p \\
      ||{\mathbf{x}} ||_p \ge \frac{||\mathbf{b}^{\star}(\mathbf{x}) ||_p}{  ||\mathbf{I} - \mathbf{A}^{\star}(\mathbf{x})||_p  }
\end{align}
If the conditions of Corollary~\ref{thm:DNN_eq_bounded} are satisfied then
the  conditions~\eqref{eq:DMM_eq_upper} and~\eqref{eq:DMM_eq_lower}  hold.
\end{proof}

\begin{corollary} 
\label{theorem:dnn_stable}
Neural neural dynamics~\eqref{eq:neural_dynamics} satisfies the dissipativity condition~\eqref{eq:dissipatitvity_condition_dnn}  if the norms of all the weights $\mathbf{W}_i$ and activation  matrices $\boldsymbol\Lambda_{\mathbf{z}_j}$~\eqref{eq:lambda_matrix} of $\mathbf{f}_{\theta}(\mathbf{x})$ are contractive:
           \begin{equation}
        \label{eq:sufficient_2}
        \|\mathbf{W}_i \|_2 < 1, \ i \in \mathbb{N}_0^L, \ \|\boldsymbol\Lambda_{\mathbf{z}_j}\|_2 \le 1, \ \forall j \in \mathbb{N}_1^L
    \end{equation}
\end{corollary}

\begin{proof}
To prove the general case with  weights $\mathbf{W}_i \in \mathbf{R}^{n_i \times m_i}$ we apply the  submultiplicativity of the induced $2$-norms~\eqref{eq:operator_norm_submultiplicative} 
to upper bound the norm of products of $m$ matrices $\mathbf{A}_i$ given as:
\begin{equation}
\label{eq:Gelfand_norm}
     \|\mathbf{A}^{\star}(\mathbf{x})\|_2 = \|\mathbf{A}_1  \ldots \mathbf{A}_m  \|_2 \le
     \| \mathbf{A}_1  \|_2  \ldots \| \mathbf{A}_m  \|_2
\end{equation}
 By applying~\eqref{eq:Gelfand_norm}  
 to~\eqref{eq:dnn_LPV} with 
 $ \|\mathbf{W}_i\|_2 < 1, \ \forall i \in \mathbb{N}_0^L$, $\|\boldsymbol\Lambda_{\mathbf{z}_j}\|_2  \le 1,  \ \forall j \in \mathbb{N}_1^L$, it yields $\|\mathbf{A}^{\star}(\mathbf{x})\|_2 < 1$, $\forall \mathbf{x} \in \mathbb{R}^{n_x}$.
Now for arbitrary point $\mathbf{x}$, the local linear operator $\mathbf{A}^{\star}(\mathbf{x})$ is a contractive map and satisfies the sufficient condition~\eqref{eq:dissipatitvity_condition_dnn}. 
\end{proof}

\begin{remark}
The norm upper bound~\eqref{eq:Gelfand_norm} implies the sufficiency of a relaxed condition~\eqref{eq:sufficient_2}
where at least one of the matrix norms is strictly below one $\| \mathbf{A}_i  \|_2 <1$, while the others are less or equal to one $\| \mathbf{A}_j  \|_2 \le 1, \forall j \in \mathbb{N}_1^L \setminus i$.
\end{remark}


\subsection{Practical Design of Dissipative Deep Neural Dynamics} \label{sec:maps}

As given in Corollary~\ref{theorem:dnn_stable}, if the product of weight and activation scaling matrices is a contraction, the global stability of deep neural dynamics~\eqref{eq:neural_dynamics} is guaranteed. The following discussion summarizes some practical methods for enforcing these conditions in deep neural networks.

\paragraph{Lipschitz Continuous Activation Functions}
As part of the sufficient stability conditions in Theorem~\ref{theorem:dnn_stable}, the scaling matrices $\boldsymbol\Lambda_{\mathbf{z}}$ generated by activation functions
must yield non-expanding maps $||\boldsymbol\Lambda_{\mathbf{z}}||_2 \le 1$ for any $\mathbf{z}$. 
Activation function $\mathbf{\boldsymbol\sigma}(\mathbf{z})$
in the network $\mathbf{f}_{\theta}(\mathbf{x})$  is   stable over the entire domain $\text{dom}(\mathbf{\boldsymbol\sigma}(\mathbf{z}))$, 
if it is Lipschitz continuous with constant $K \le 1$.
Lipschitz continuity~\eqref{eq:Lipschitz} with $K \le 1$ implies 
 contraction which  implies asymptotic stability~\eqref{eq:asymptotic} via Banach fixed point Theorem~\ref{thm:banch}.
 Observe that these conditions force all diagonal entries of
 the activation scaling matrix $\boldsymbol\Lambda_{\mathbf{z}}$~\eqref{eq:lambda_matrix} to satisfy $|\frac{ \sigma(z_i)}{z_i}| < 1$.  And because  $\boldsymbol\Lambda_{\mathbf{z}}$ is a diagonal matrix, its diagonal entries represent its real eigenvalues with bounded spectral norm $||\boldsymbol\Lambda_{\mathbf{z}}||_2 <1$.
     Fig.~\ref{fig:activations}
plots common activation functions with guaranteed  stability\footnote{Stable activations: \texttt{SoftExponential}, \texttt{BLU}, \texttt{PReLU}, \texttt{ReLU}, \texttt{GELU}, \texttt{RReLU},  \texttt{Hardtanh}, \texttt{ReLU6}, \texttt{Tanh}, \texttt{ELU}, \texttt{CELU}, \texttt{Hardshrink}, \texttt{LeakyReLU}, \texttt{Softshrink}, \texttt{Softsign}, \texttt{Tanhshrink}},  and activations with unstable regions\footnote{Activations with unstable regions: \texttt{APLU}, \texttt{PELU}, \texttt{Hardswish}, \texttt{SELU}, \texttt{LogSigmoid}, \texttt{Softplus}, \texttt{Hardswish}}. Despite having locally unstable regions, some activations\footnote{Unstable activations with stable regions of attraction: \texttt{Sigmoid}, \texttt{Hardsigmoid}} are clamping the edges of the domain, thus generating contractive maps towards regions the central region of attraction with non-zero volume. 
    \begin{figure*}[htbp!]
        \centering
        \includegraphics[width=0.21\textwidth, trim=60 10 80 0, clip]{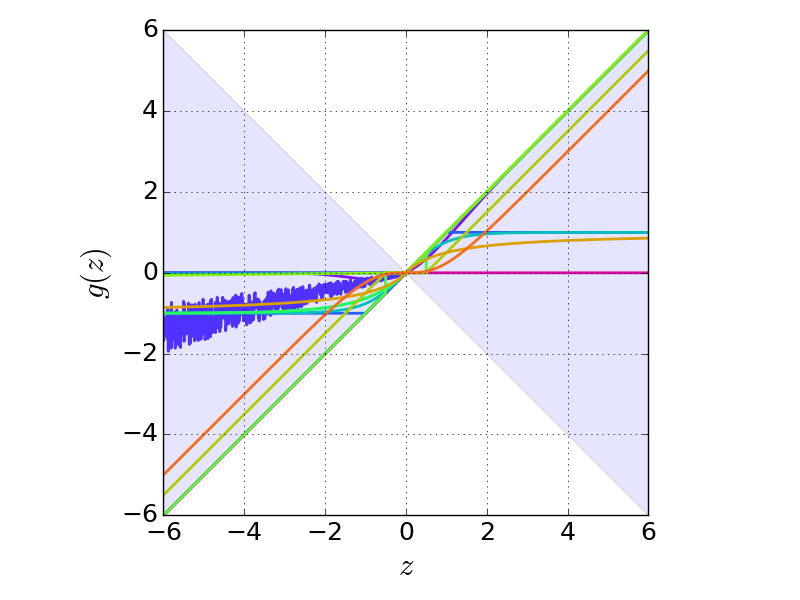}
        \includegraphics[width=0.21\textwidth, trim=60 10 80 0, clip]{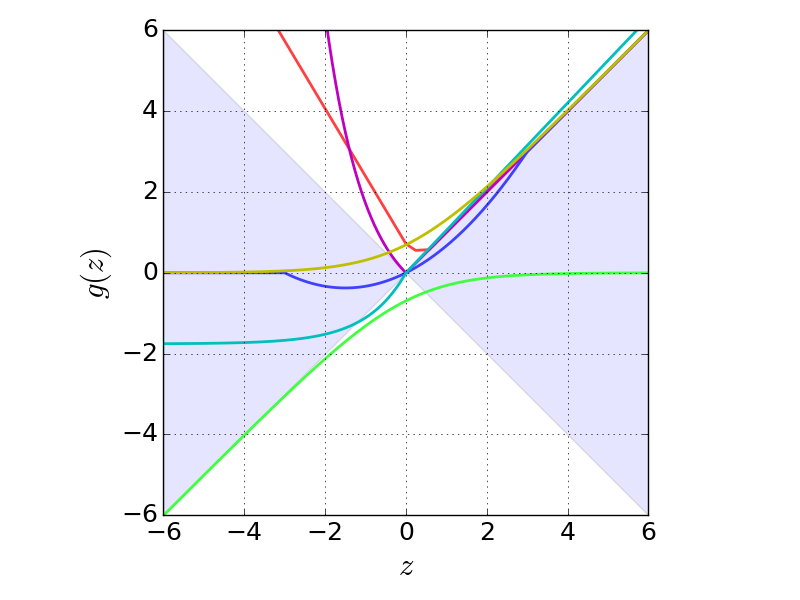}
        \includegraphics[width=0.21\textwidth, trim=60 10 80 0, clip]{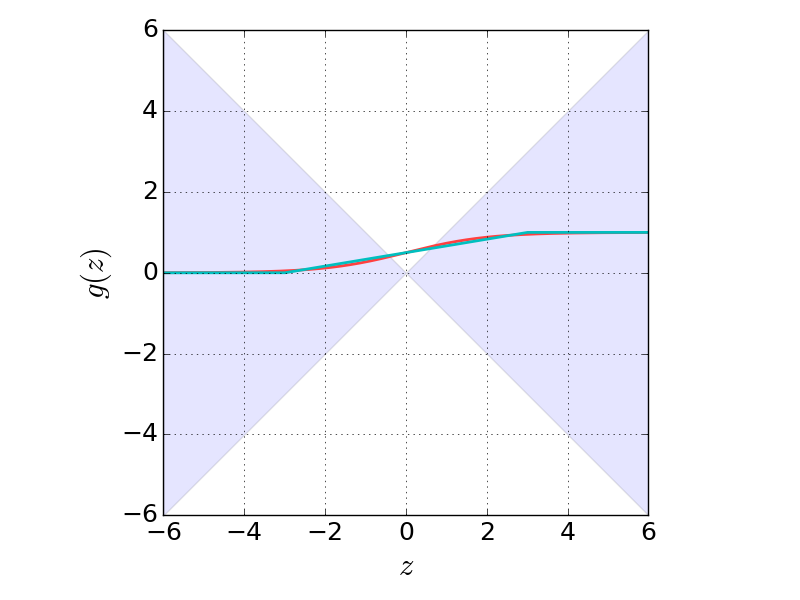}
        \caption{Activation functions with stability guarantees (left), with unstable regions (middle), and with unstable central regions but contractive elsewhere (right), respectively. Blue areas represent functions with trivial null space and Lipschitz constant $K \le 1$.}
        \label{fig:activations}
    \end{figure*}

\paragraph{Normed weights}
Common regularizations linked with improving numerical stability of neural networks are norm penalties of layer weights. Examples include
$L_0$~\cite{WangLH15}, $L_1$~\cite{Ng_norms2004,Duchi2008,L1_dnns2018},
 $L_2$~\cite{Ng_norms2004,NeyshaburTS15},  $L_{\infty}$~\cite{hoffer2019norm}, or spectral norm~\cite{specNorm2018,farnia2018generalizable}. Minimizing weight norms results in tightening the upper bound of the norm of the overall function composition~\eqref{eq:lpv_dnn}, which has a stabilizing effect on the resulting deep neural dynamics.
  However, all of these standard norm regularizations effectively drive the upper bound of the norm towards zero, which  might eventually result in the vanishing gradients problem for deeper networks~\cite{Pascanu2012}.
 Some of the following parametrizations might alleviate this problem by constraining the lower bounds of the operator norms.
 
%

\paragraph{Perron-Frobenius weights~\cite{tuor2020constrained}} This approach 
 uses the Perron-Frobenius theorem for imposing bounds on the dominant eigenvalue of a square nonnegative matrix $\mathbf{W}$ given as:
\begin{subequations}
\label{eq:pf}
\begin{align}
\mathbf{M} &= \lambda_{\text{max}} - (\lambda_{\text{max}} - \lambda_{\text{min}})  g(\mathbf{M'}) \\
\mathbf{W}_{i,j} &= \frac{\text{exp}(\mathbf{A'}_{ij})}{\sum_{k=1}^{n_x} \text{exp}(\mathbf{A'}_{ik})}\mathbf{M}_{i,j}
\end{align}
\end{subequations}
where matrix $\mathbf{M}$ represents damping parameterized by the matrix $\mathbf{M'} \in \mathbb{R}^{n_x \times n_x}$.
We apply a row-wise softmax to another parameter matrix $\mathbf{A'} \in \mathbb{R}^{n_x \times n_x}$, 
then elementwise multiply by $\mathbf{M}$ to obtain the stable weight $\mathbf{W}$ with eigenvalues lower and upper bounds $\lambda_{\text{min}}$ and $\lambda_{\text{max}}$.

\paragraph{Spectral weights~\cite{zhang2018stabilizing,skomski2021constrained}}
This method parametrizes a weight matrix as a factorization via singular value decomposition (SVD). The weight is defined as two unitary matrices $\mathbf{U}$ and $\mathbf{V}$ initialized as orthogonal matrices, and singular values $\Sigma$ initialized randomly. 
The advantage of the SVD factorization is that it supports non-square matrices.
This regularization enforces boundary constraints on the singular values $\Sigma$ by:
\begin{subequations}
\begin{align}
    \mathbf{{\Sigma}} &= \text{diag}(\lambda_{\text{max}} - (\lambda_{\text{max}} - \lambda_{\text{min}}) \cdot \sigma(\Sigma))\\
   \mathbf{W} &= \mathbf{U{\Sigma}V} \label{eq:SVD}
\end{align}
\end{subequations}
where $\lambda_{\text{min}}$ and $\lambda_{\text{max}}$ are the lower and upper singular value bounds, respectively.
For enforcing orthogonal structure,~\cite{zhang2018stabilizing} used Householder reflectors  to represent unitary matrices $\mathbf{U}$ and $\mathbf{V}$, an alternative approach is to use soft constraint penalties on unitary matrices~\cite{skomski2021constrained}.

\paragraph{Gershgorin discs weights~\cite{lechner2020gershgorin}}
This factorization based on \textit{Gershgorin Circle Theorem}~\cite{Varga_Gersgorin2004} confines the
 eigenvalues $\lambda_i$ of a square weight $\mathbf{W}$ within a circle with center $\lambda$ and radius $r$ and is given as:
\begin{equation}
\label{eq:Gershgorin}
 \mathbf{W} = \texttt{diag}\begin{pmatrix}\frac{r}{s_1}, ..., \frac{r}{s_n}\end{pmatrix}\mathbf{M} +  \texttt{diag}\begin{pmatrix}\lambda, ..., \lambda\end{pmatrix}
\end{equation}
Here the parameters of a matrix $\mathbf{M} \in  \mathbb{R}^{n\times n}$ belongs to $m_{i,j} \sim \mathcal{U}(0,1)$, with $m_{i, i} = 0$.
Then each row's elements are divided by it's sum $s_j = \sum_{i\neq j} m_{i,j}$ and multiplied by a radius $r$, finally adding a diagonal matrix of values where the matrix eigenvalues should be centered.

\paragraph{Dissipativity penalties}
An alternative approach to learning dissipative neural dynamics is to penalize the dissipativity conditions~\eqref{eq:dissipatitvity_condition_dnn} in the loss function via regularization term:
\begin{equation}
\label{eq:reg:stable}
    \mathcal{L}_{\texttt{s}} = 
 \max\big(1, \|\mathbf{A}^{\star}(\mathbf{x})\|_2\big)
\end{equation}
The disadvantage of~\eqref{eq:reg:stable}
is that the penalties do not guarantee dissipativity by design. On the other hand, they may be more expressive compared to
the design methods discussed above.

\section{Numerical Case Studies}
\label{sec:case_studies}

\subsection{Stability Analysis of Autonomous Neural Dynamics}

This section presents empirical analysis of the effect of different components of deep neural dynamics on their operator norm bounds, state space trajectories, and stability.
Investigated components include weight factorizations from Section~\ref{sec:maps}, types of activations, bias terms, and network depth.
For the sake of intuitive visualizations we perform experiments on neural dynamical system~\eqref{eq:neural_dynamics} with two states $n_x = 2$.
For each configuration, we compute PWA form~\eqref{eq:lpv_dnn} of the DNN over a 2D grid ranging [-6, 6] in both dimensions to compute their operator norms,   state space regions, and state space trajectories. Additional details on the configurations used can be found in Appendix~\ref{sec:appendix_f}.

\paragraph{Dynamical Effects of Weights}
To empirically verify implications of Corollary~\ref{theorem:dnn_stable}, we construct 2D $8$-layer neural models with layer-wise  eigenvalues constrained between zero and one, close to one, and greater than one.
To do so we randomly generate a set of constrained weight matrices using methods from Section~\ref{sec:maps}.
Fig.~\ref{fig:weights} demonstrates dynamical properties and associated eigenvalue spectra of DNNs with  stable\footnote{Spectral norm $||\mathbf{A}^{\star}(\mathbf{x})||_2 \le 1$, \texttt{Tanh} activation, Gershgorin disc factorized weight.}, marginally stable\footnote{Spectral norm $||\mathbf{A}^{\star}(\mathbf{x})||_2 \approx 1$, \texttt{Tanh} activation, Spectral SVD factorized weight.}, and unstable\footnote{Spectral norm $||\mathbf{A}^{\star}(\mathbf{x})||_2 \ge 1$, \texttt{Softplus} activation, Gershgorin disc factorized weight.} dynamics, respectively.
 The first row visualizes state space trajectories, while second row plots the complex plane with eigenvalues of the DNN's PWA forms~\eqref{eq:lpv_dnn}. 
The left hand side example on Fig.~\ref{fig:weights}  demonstrates the effect of the contraction condition of Corollary~\ref{theorem:glbl_stable} leading to asymptotic stability.
\begin{figure*}[h]
    \centering
    \includegraphics[width=0.20\textwidth, trim=70 10 80 40, clip]{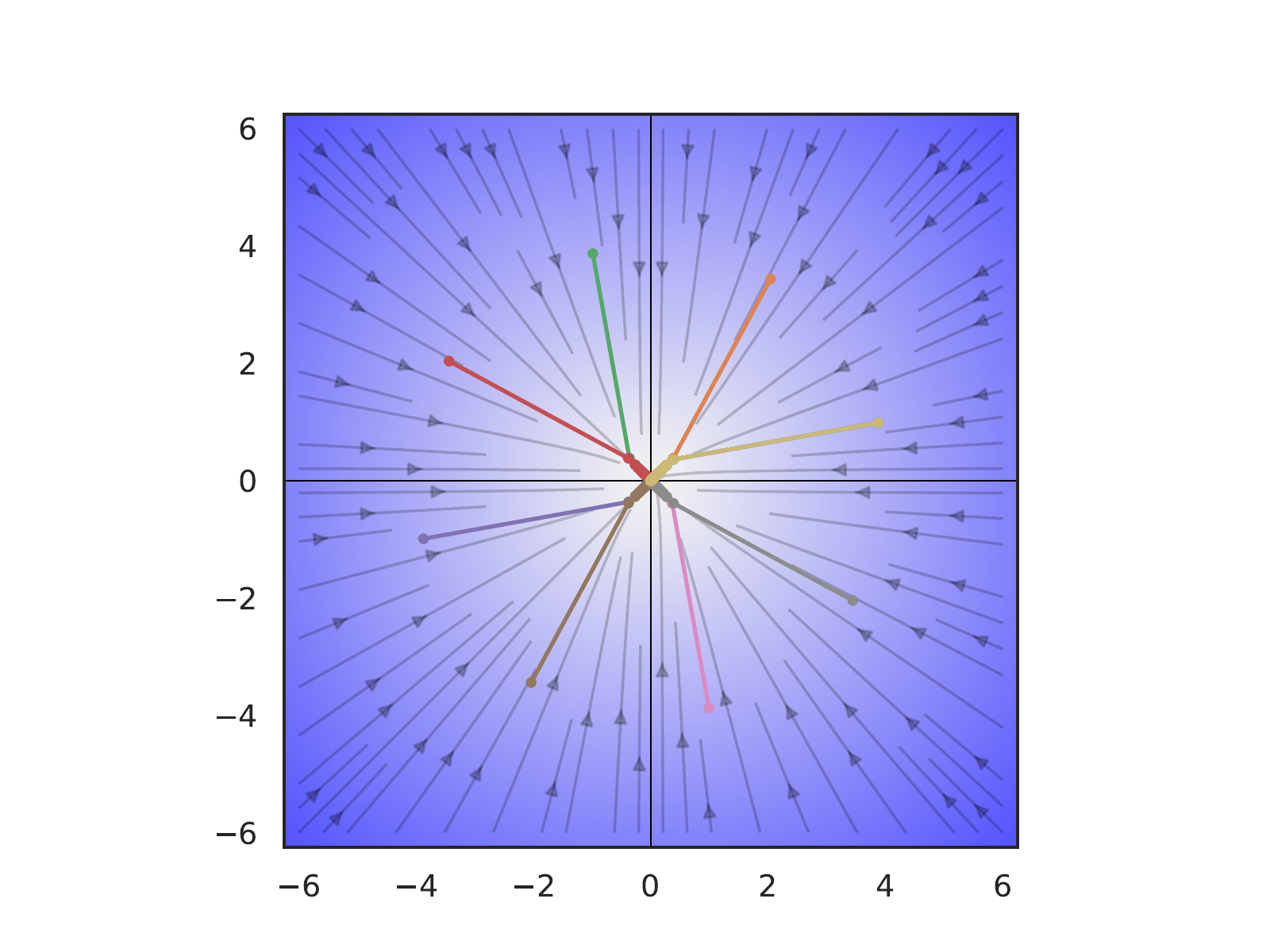}
    \includegraphics[width=0.20\textwidth, trim=70 10 80 40, clip]{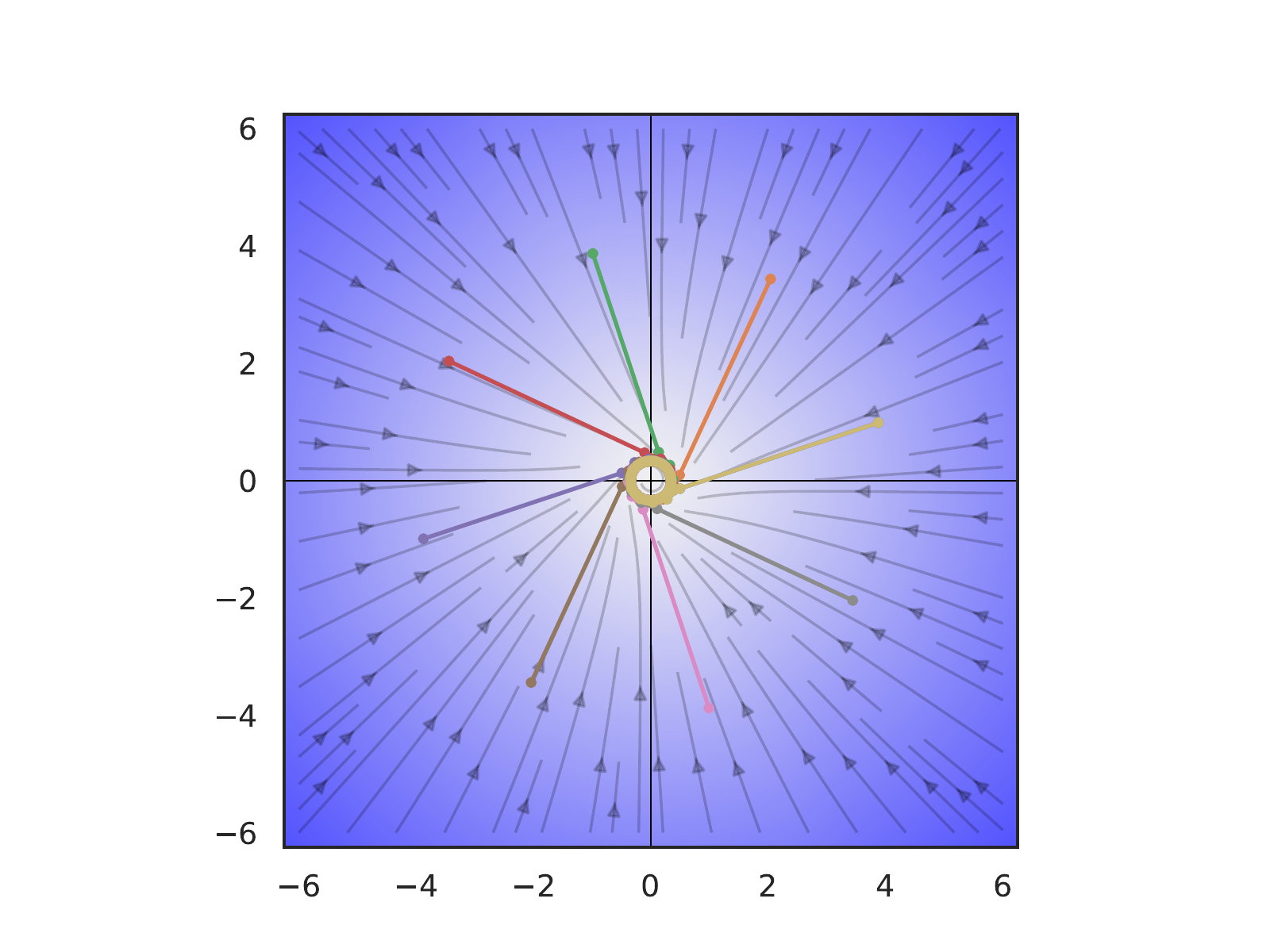}
        \includegraphics[width=0.20\textwidth, trim=70 10 80 40, clip]{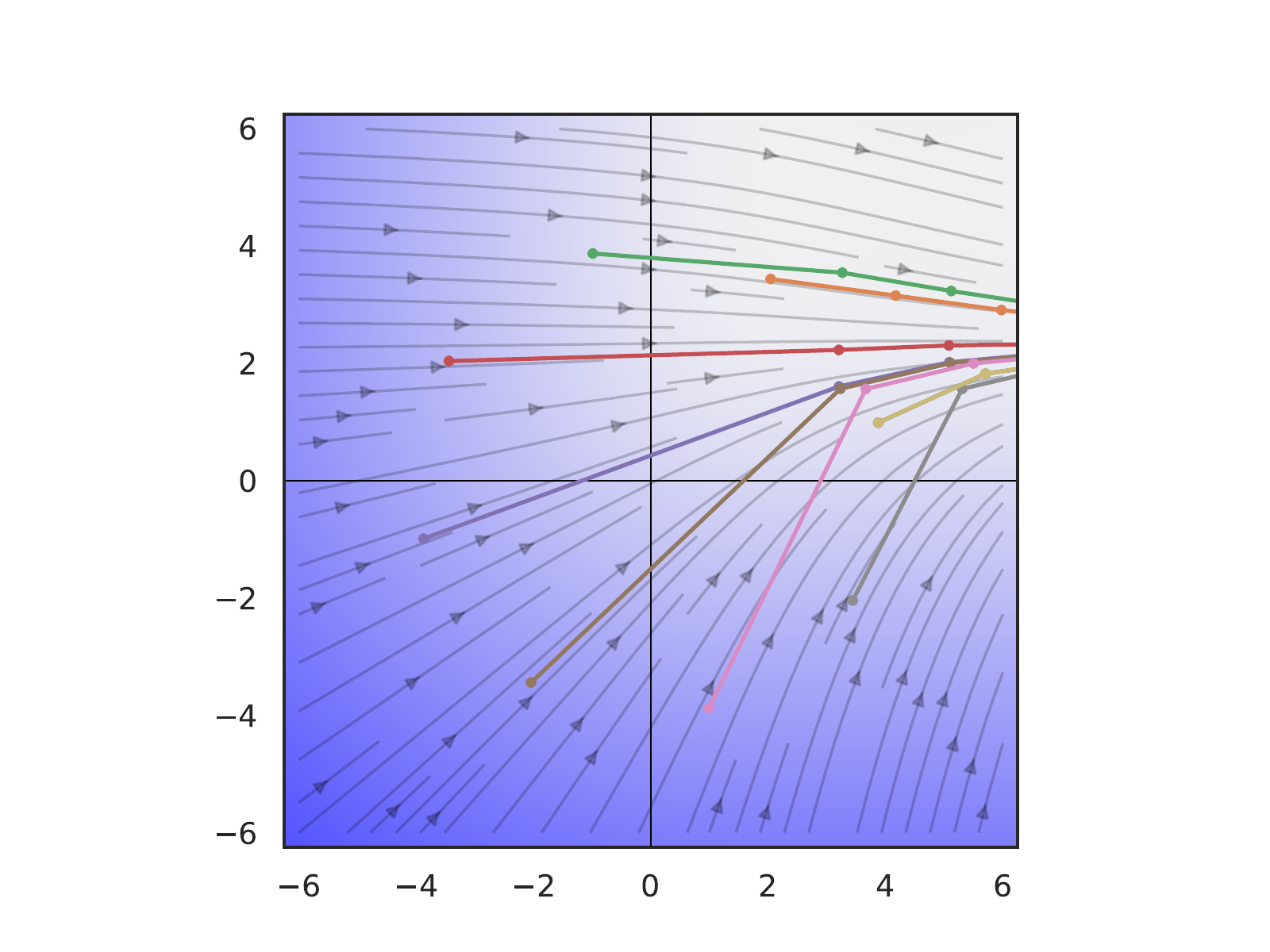}  \\
        \includegraphics[width=0.20\textwidth, trim=45 10 70 0, clip]{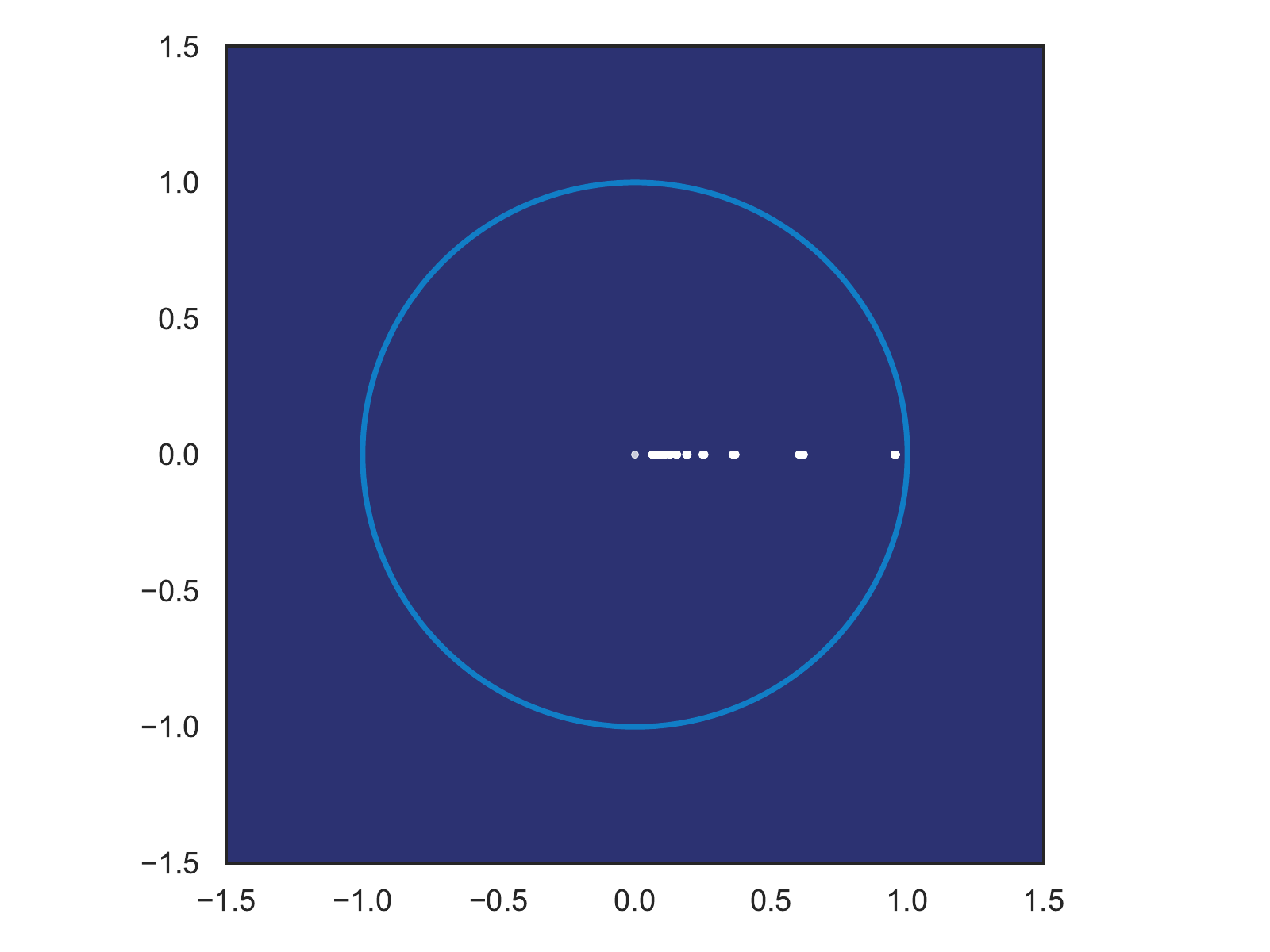}
    \includegraphics[width=0.20\textwidth, trim=45 10 70 0, clip]{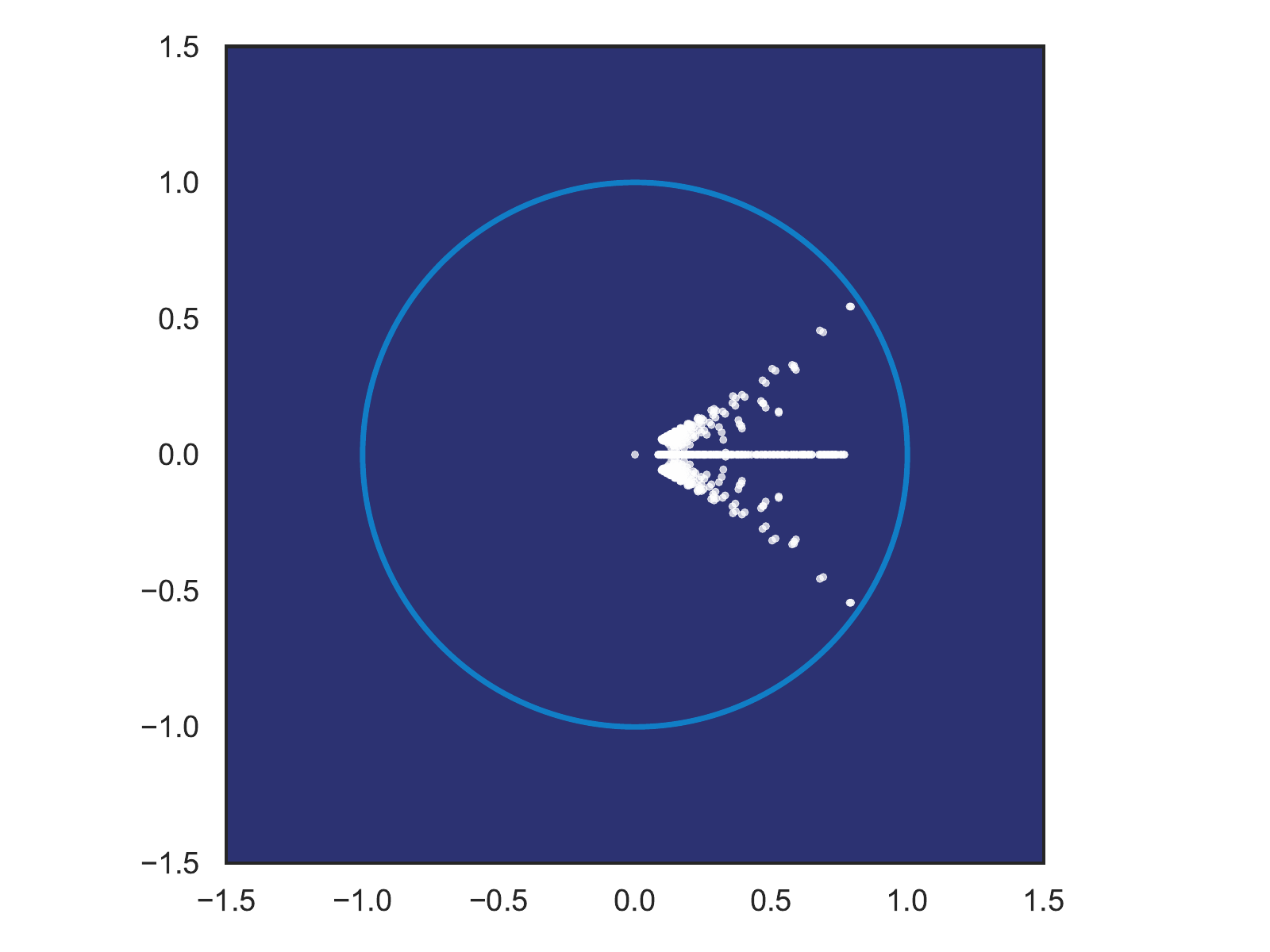}
        \includegraphics[width=0.20\textwidth, trim=45 10 70 0, clip]{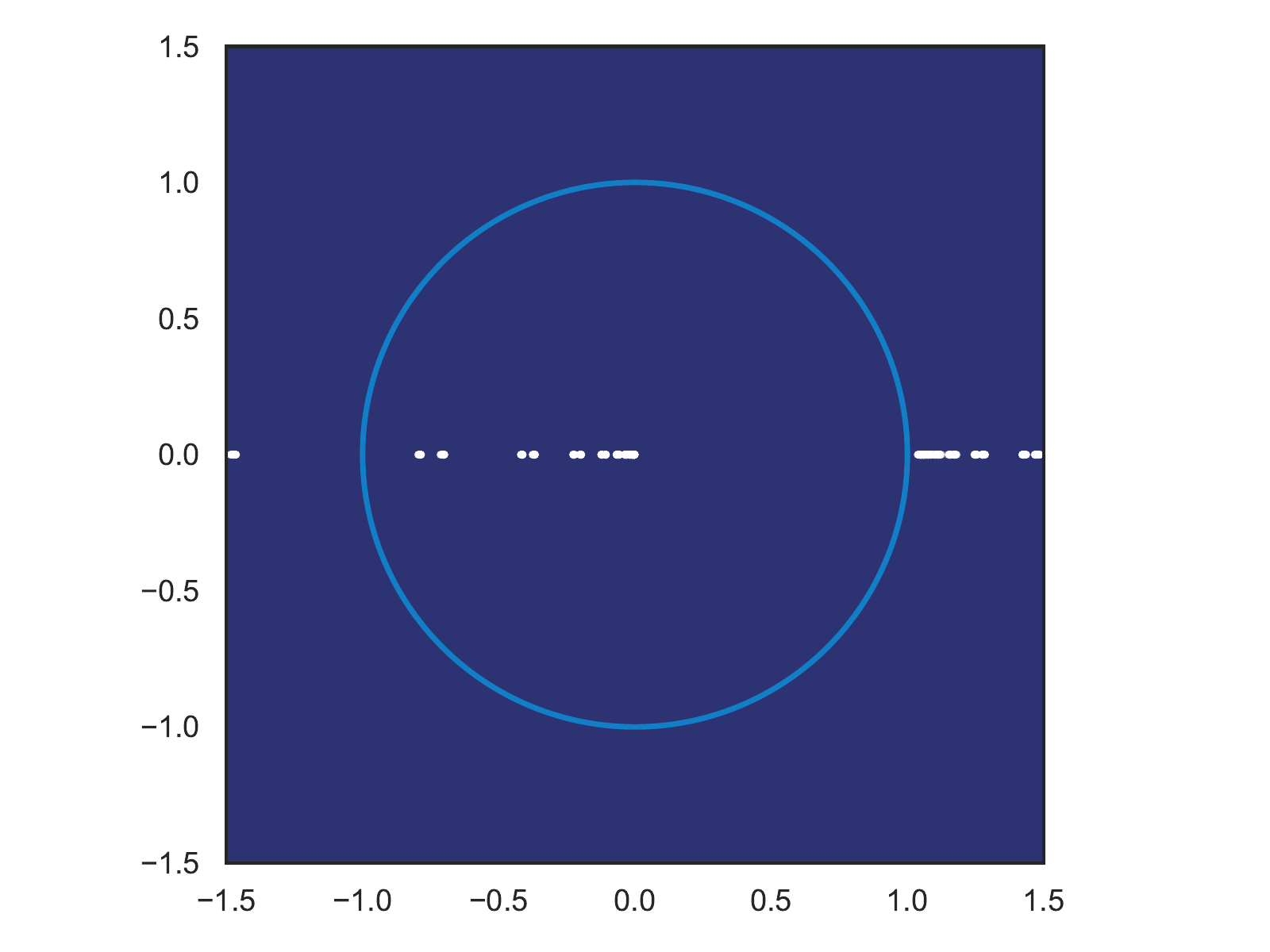} 
    \caption{Top: DNNs with  stable $||\mathbf{A}^{\star}(\mathbf{x})||_2 \le 1$ (left),   marginally stable $||\mathbf{A}^{\star}(\mathbf{x})||_2 \approx 1$ (center), and unstable dynamics $||\mathbf{A}^{\star}(\mathbf{x})||_2 \ge 1$ (right). Bottom: Eigenvalues of $\mathbf{A}^{\star}(\mathbf{x})$.}
    \label{fig:weights}
\end{figure*}

\paragraph{Dynamical Effects of Activation Functions}
Here we leverage the PWA form~\eqref{eq:lpv_dnn}
to study the dynamical properties of DNNs with different activation functions. 
Fig.~\ref{fig:spectral_radii} displays state space regions associated with spectral radii of $4$-layer DNNs with four activations: \texttt{ReLU},  \texttt{Tanh},   \texttt{SELU}, and \texttt{Sigmoid}. All weights are randomly generated to have stable eigenvalues, i.e. $||\mathbf{W}_i||_2 \le 1, \ \forall i \in \mathbb{N}_1^L$.
As expected, \texttt{ReLU} networks generate linear regions, while the exponential part of \texttt{SELU} networks make the resulting pattern of linear regions more complex. 
On the other hand, smooth activations \texttt{Tanh}, and \texttt{Sigmoid} generate continuous gradient fields.
Thanks to the  contractivity of \texttt{ReLU} and \texttt{Tanh} as given by their Lipschitz constant $K \le 1$
the whole state space is guaranteed to be dissipative with $||\mathbf{A}^{\star}(\mathbf{x})||_p \le 1$.  
Because both \texttt{SELU} and \texttt{Sigmoid} activations 
have  Lipschitz constant $K \ge 1$
 the dissipativity of their state space is not guaranteed by design.
Even though in this case, \texttt{SELU} network generated  stable dynamics. 
On the other hand, the \texttt{Sigmoid} network generated a large unstable region surrounded by stable regions. This is   caused by its nontrivial null space and clamped tails. 
\begin{figure*}[h]
    \centering
    \includegraphics[width=0.23\textwidth, trim=60 20 20 38, clip]{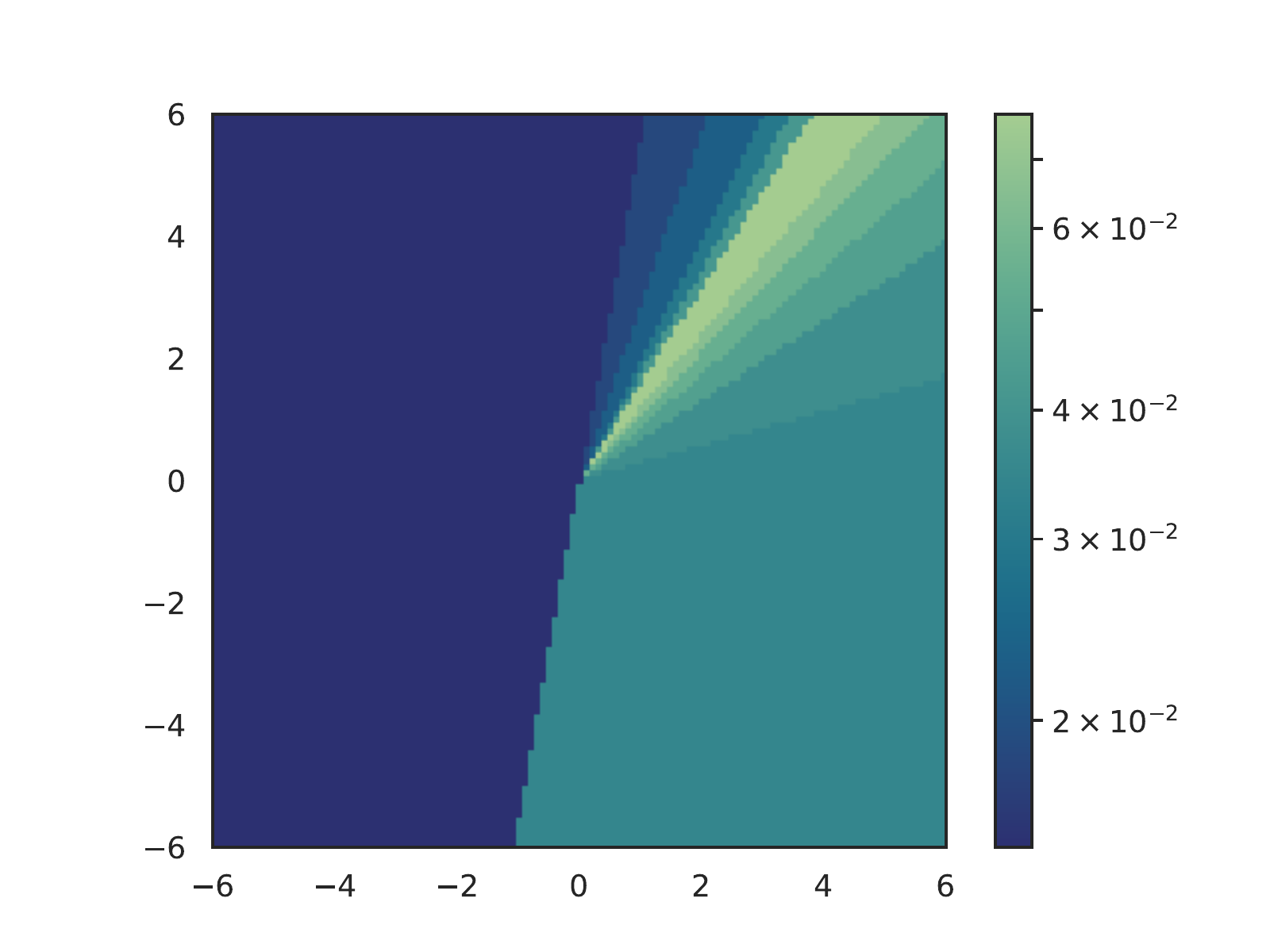}
    \includegraphics[width=0.23\textwidth, trim=60 20 20 38, clip]{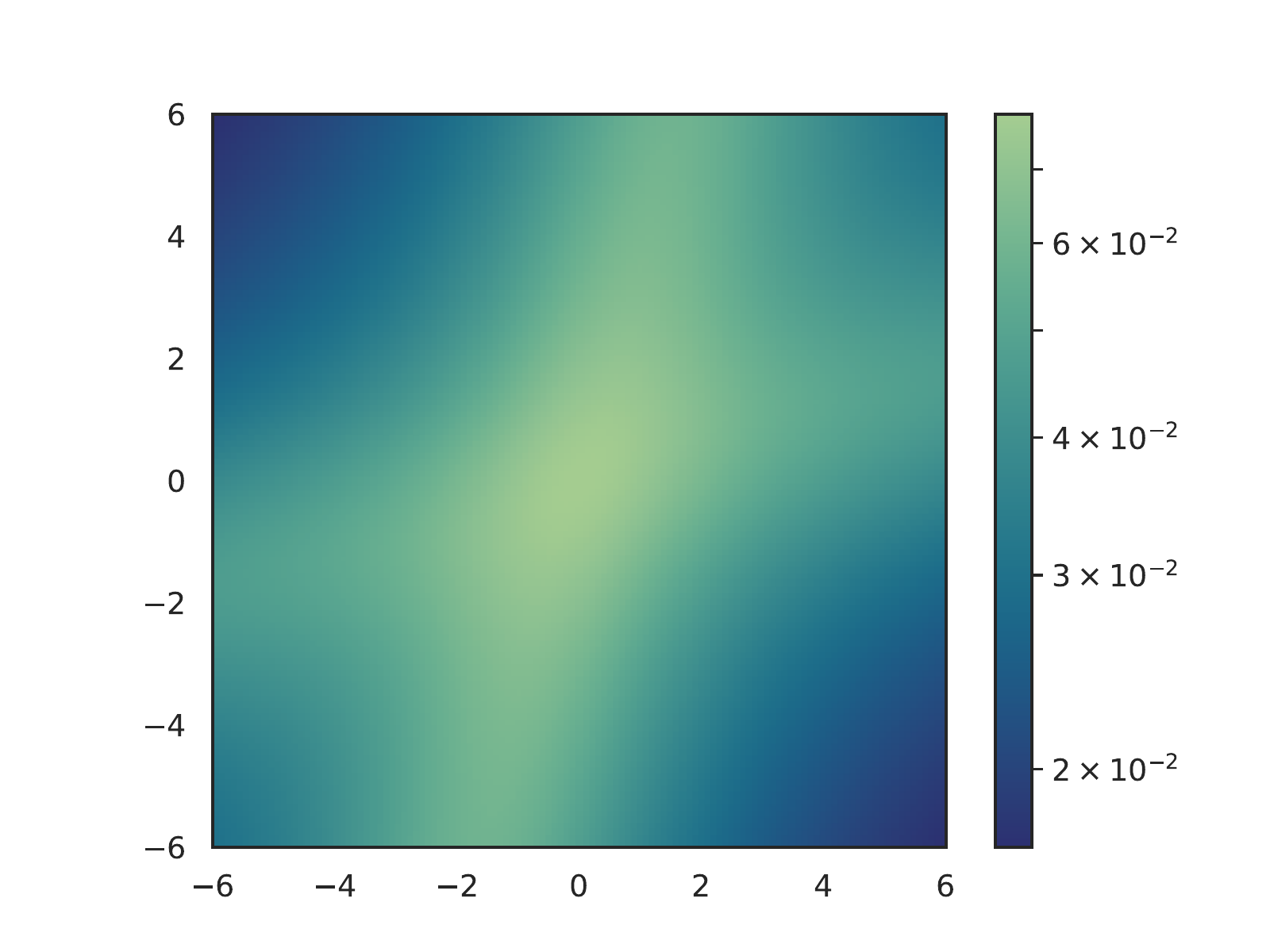}
    \includegraphics[width=0.23\textwidth, trim=60 20 20 38, clip]{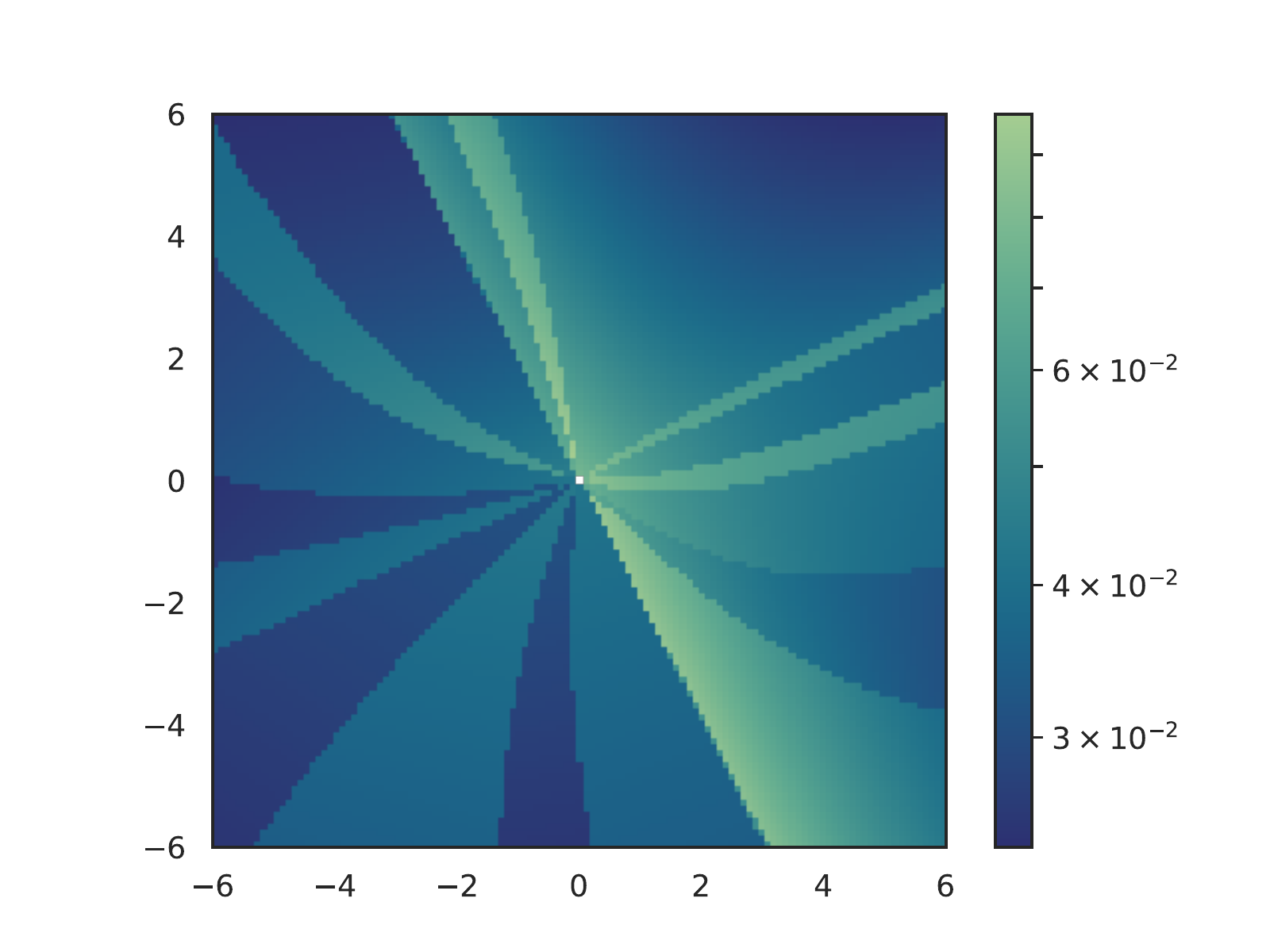} 
    \includegraphics[width=0.23\textwidth, trim=60 20 20 38, clip]{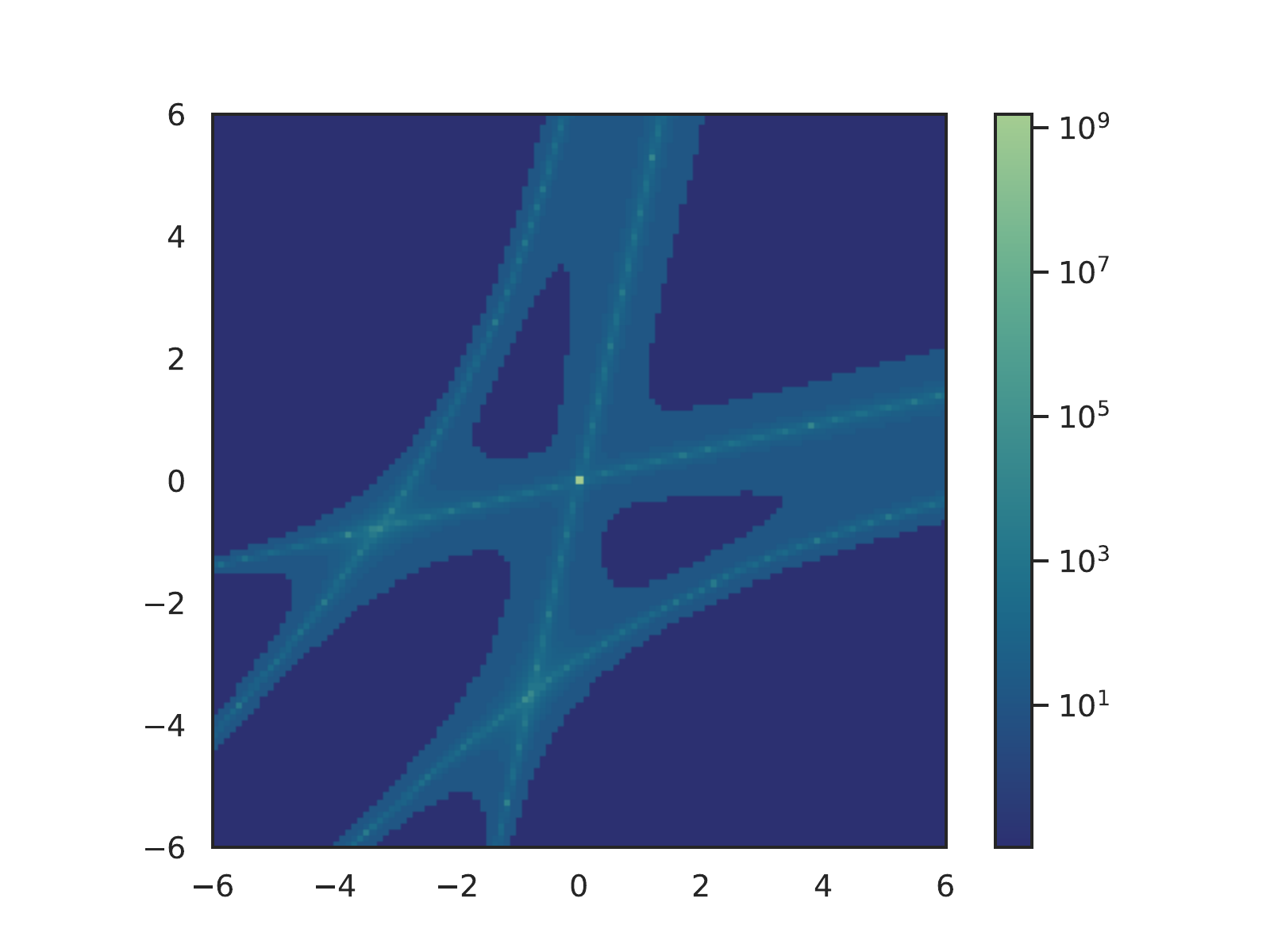} 
    \caption{State space regions associated with norms $\|\mathbf{A}^{\star}\|_2$ of two dimensional neural dynamics~\eqref{eq:neural_dynamics} parametried with $4$-layer DNN with different activation functions initialized with stable weights. From left to right \texttt{ReLU},  \texttt{Tanh},  \texttt{SELU}, and  \texttt{Sigmoid}.}
    \label{fig:spectral_radii}
\end{figure*}

\paragraph{Dynamical effects of bias terms}
If the conditions of Corollary~\ref{theorem:glbl_stable} are satisfied, the neural dynamics~\eqref{eq:neural_dynamics} is  asymptotically stable~\eqref{eq:asymptotic} with the equilibrium at the origin $\bar{\mathbf{x}} = \mathbf{0}$.
In the more general case of non-zero biases at the steady state, the equilibrium points of the dissipative neural dynamics~\eqref{eq:dissipatitvity_condition_dnn}  are shifted from the origin by the aggregate bias terms~\eqref{eq:dnn_bias_recurence}. 
Corollary~\ref{thm:DNN_eq_bounded} provides the lower and upper bound of the equilibrium norms for stable deep neural dynamics~\eqref{eq:neural_dynamics}.  
Here we demonstrate the effect of the bias term on equilibrium points $\bar{\mathbf{x}}$ using stable single-layer ReLU network as shown in Fig.~\ref{fig:bias}.
As given by Corollary~\ref{thm:DNN_eq_bounded}, for dissipative systems, the non-zero steady state bias terms shift the centroid of attractors confined in the compact subspace of the state space. 
This analysis reveals the significant role of the bias terms in learning correct steady states of dynamical systems.
This observation can inspire new research in the design of  more sophisticated parametrizations for learning dynamical systems, e.g., with non-trivial equilibria such as multi-point attractors. 
    \begin{figure*}[htb!]
        \centering
        \includegraphics[width=0.20\textwidth, trim=70 20 80 40, clip]{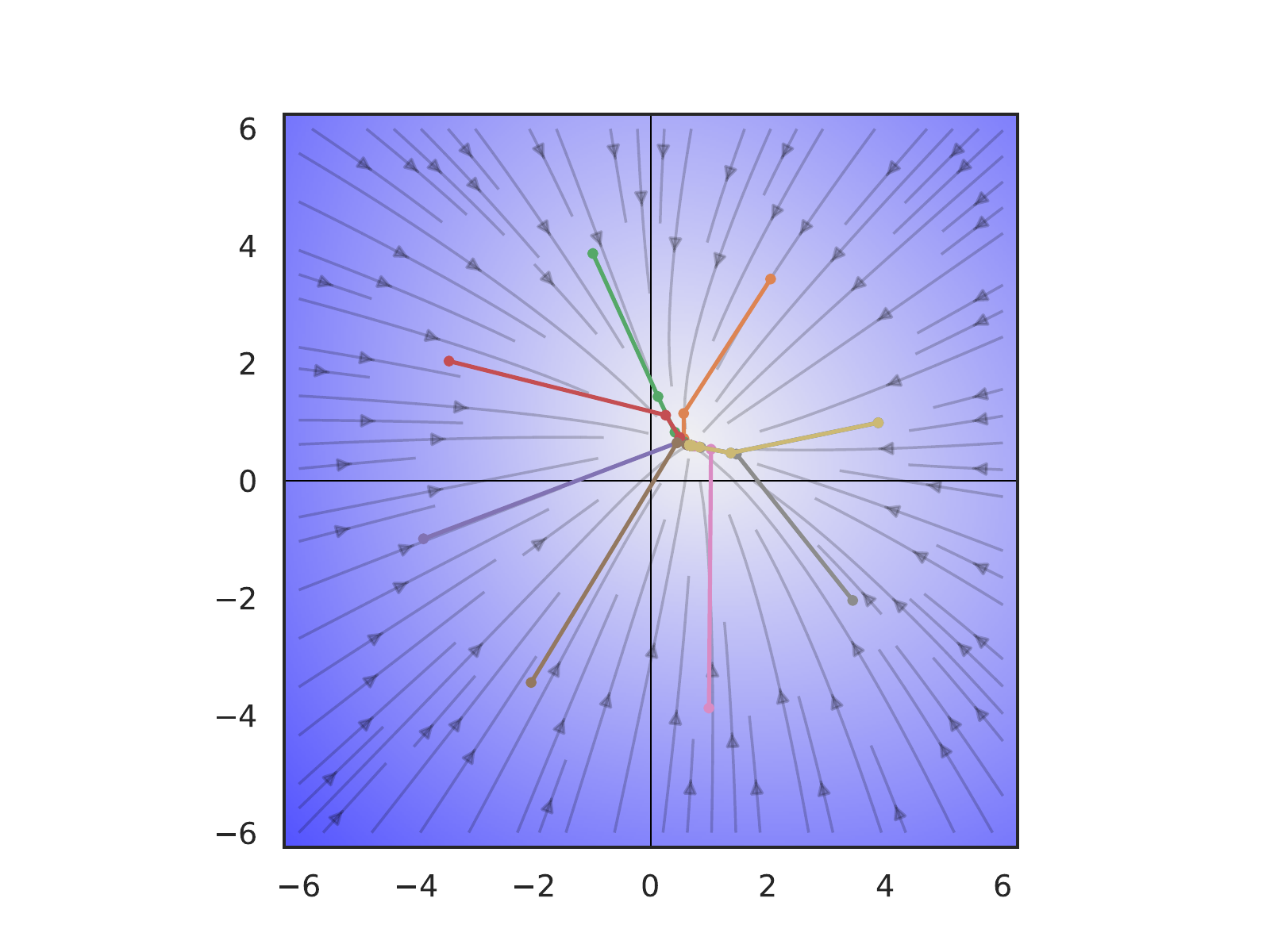}\hspace{16pt}
        \includegraphics[width=0.20\textwidth, trim=70 20 80 40, clip]{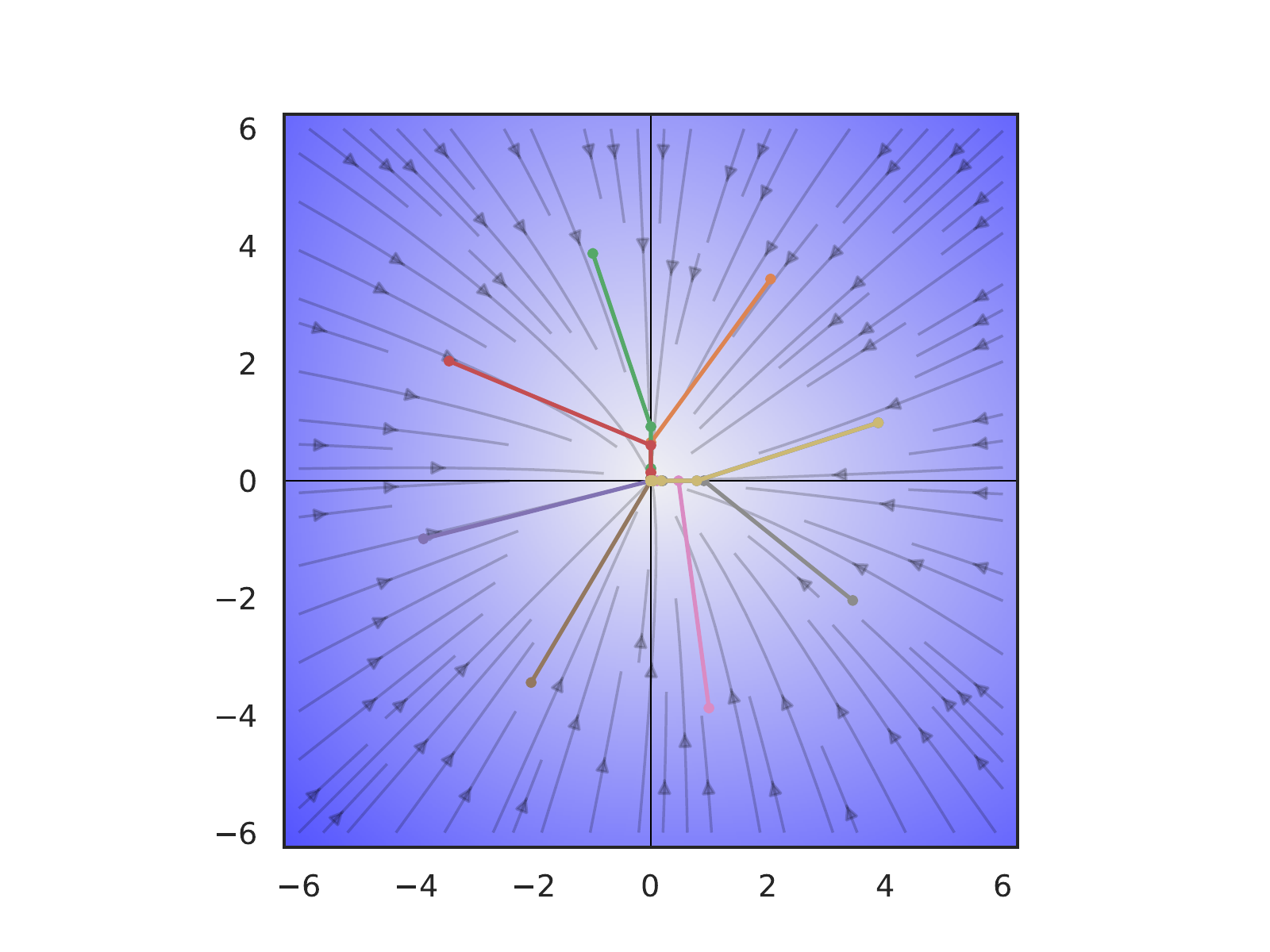} 
        \caption{Dynamics of stable two dimensional neural dynamical system~\eqref{eq:neural_dynamics} with single \texttt{ReLU} layer with bias $\bar{\mathbf{x}} \neq \mathbf{0}$ (left), and without bias $\bar{\mathbf{x}} = \mathbf{0}$ (right).}
        \label{fig:bias}
    \end{figure*}

\paragraph{Dynamical effects of network depth}
Very deep recurrent neural networks (RNNs) are notoriously difficult to train due to the vanishing and exploding gradient problems~\cite{Pascanu2012,Kolen5264952}. 
\cite{Pascanu2012,haber2017stable} linked these problems with the spectral properties of RNNs.
In the same spirit we leverage Lemma~\ref{lem:lpv} and Corollary~\ref{theorem:dnn_stable} to analyze the spectral norms and eigenvalue distribution of the forward propagation of RNNs with varying depth.
Fig.~\ref{fig:depth} displays a visualization of the spectral plots with increasing depth in deep neural network.
As given by Lemma~\ref{lem:lpv}, we can equivalently cast DNNs as a  product of pointwise affine representations of its layers. RNNs are trained by unrolling them to $L$-layer deep feedforward DNNs, where all layers share the weights and activations and hence by definition also their spectral properties.
Hence by having shared layer weights $\mathbf{A}_i = \mathbf{A}_j, \forall (i, j) \in \mathbb{N}_1^L$ it is clear that applying operator norm bounds~\eqref{eq:Gelfand_norm} yields norm bounded matrix power series.
Therefore increasing the depth of RNNs with stable layers necessarily shrinks the spectral norm with each additional layer, thus increasing the dissipativity of the system.
For unstable layers the opposite is true as the spectral norm expands exponentially.
%
 \begin{figure*}[htb!]
    \centering
    \includegraphics[width=0.27\textwidth]{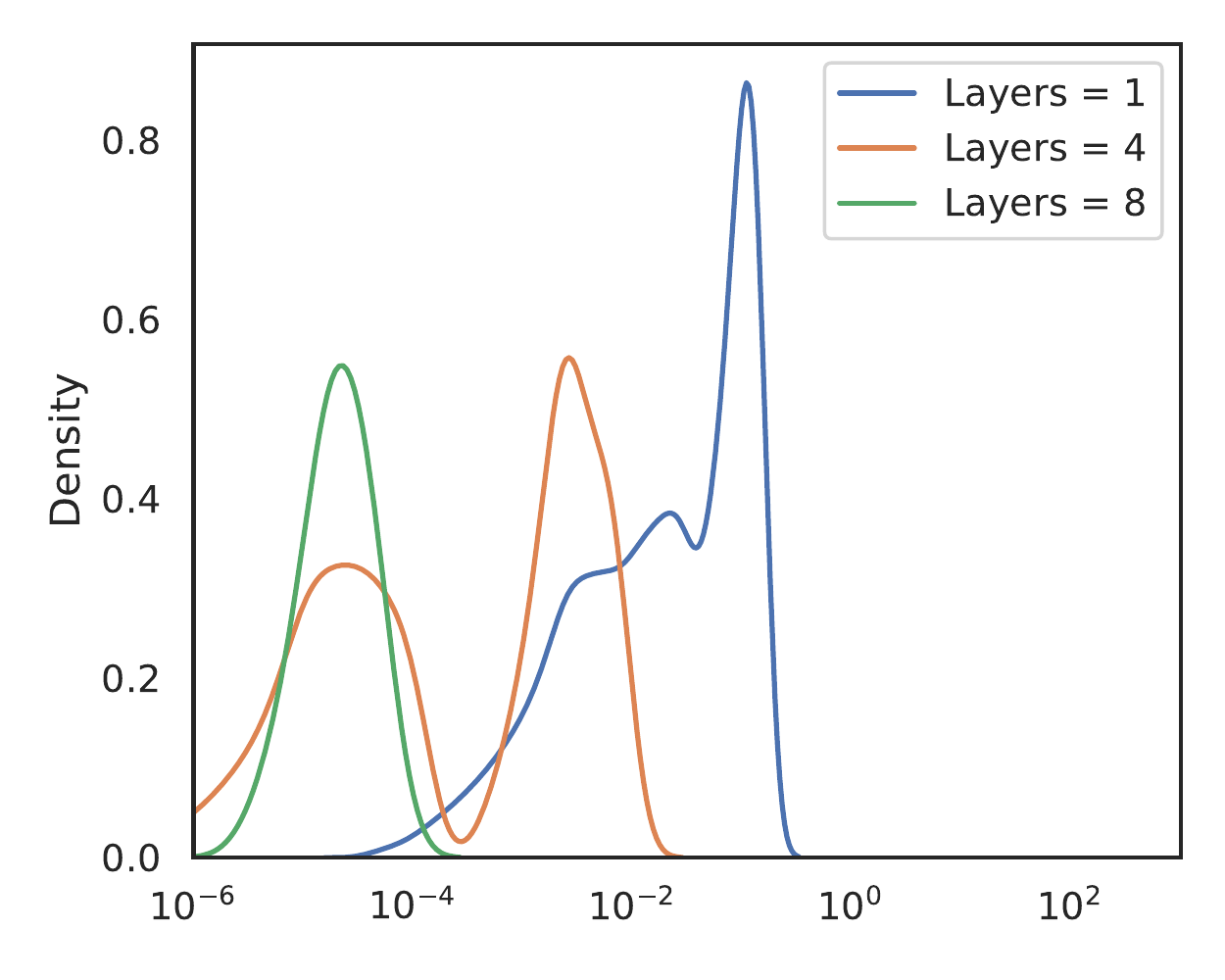}
    \includegraphics[width=0.27\textwidth]{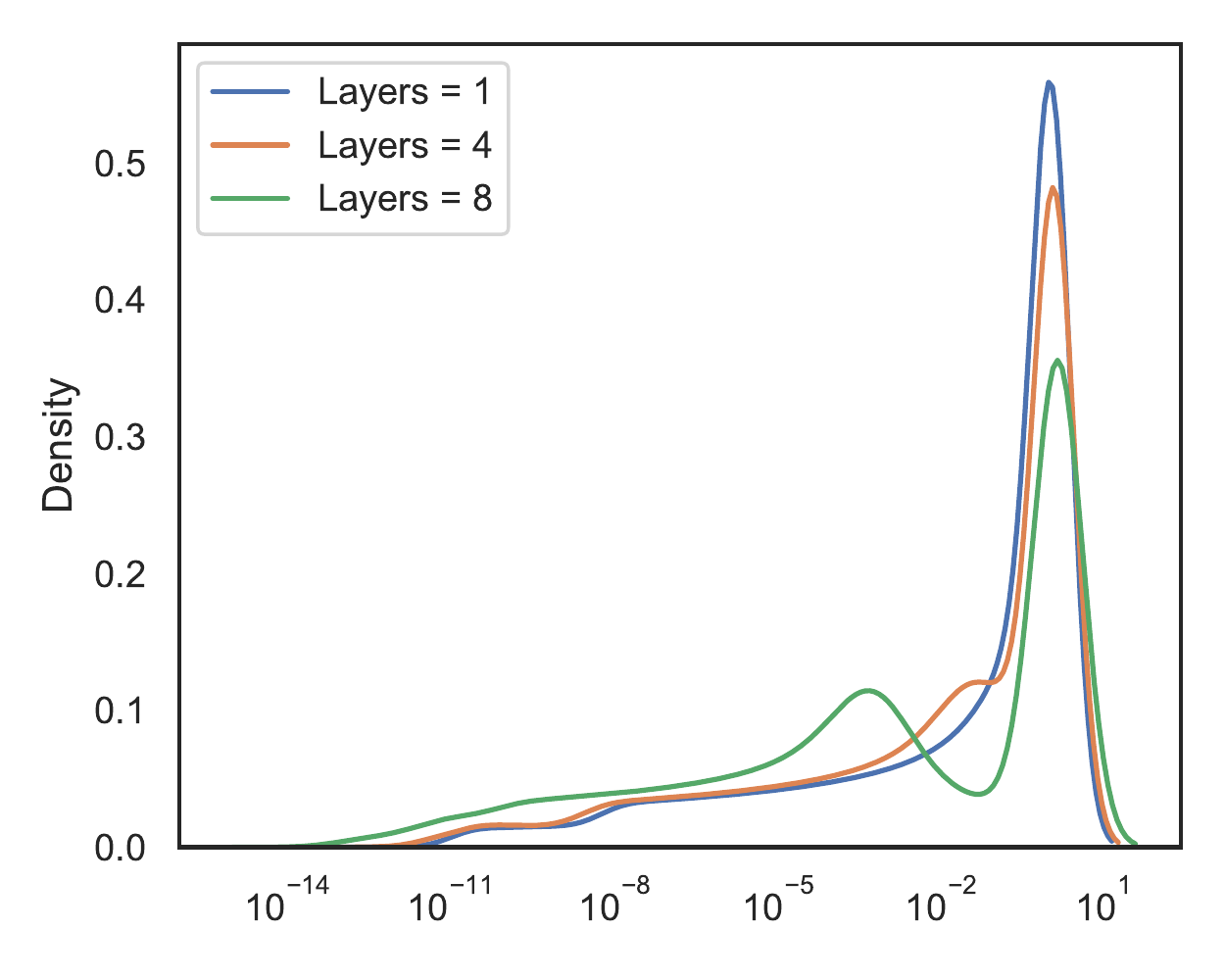}
    \includegraphics[width=0.27\textwidth]{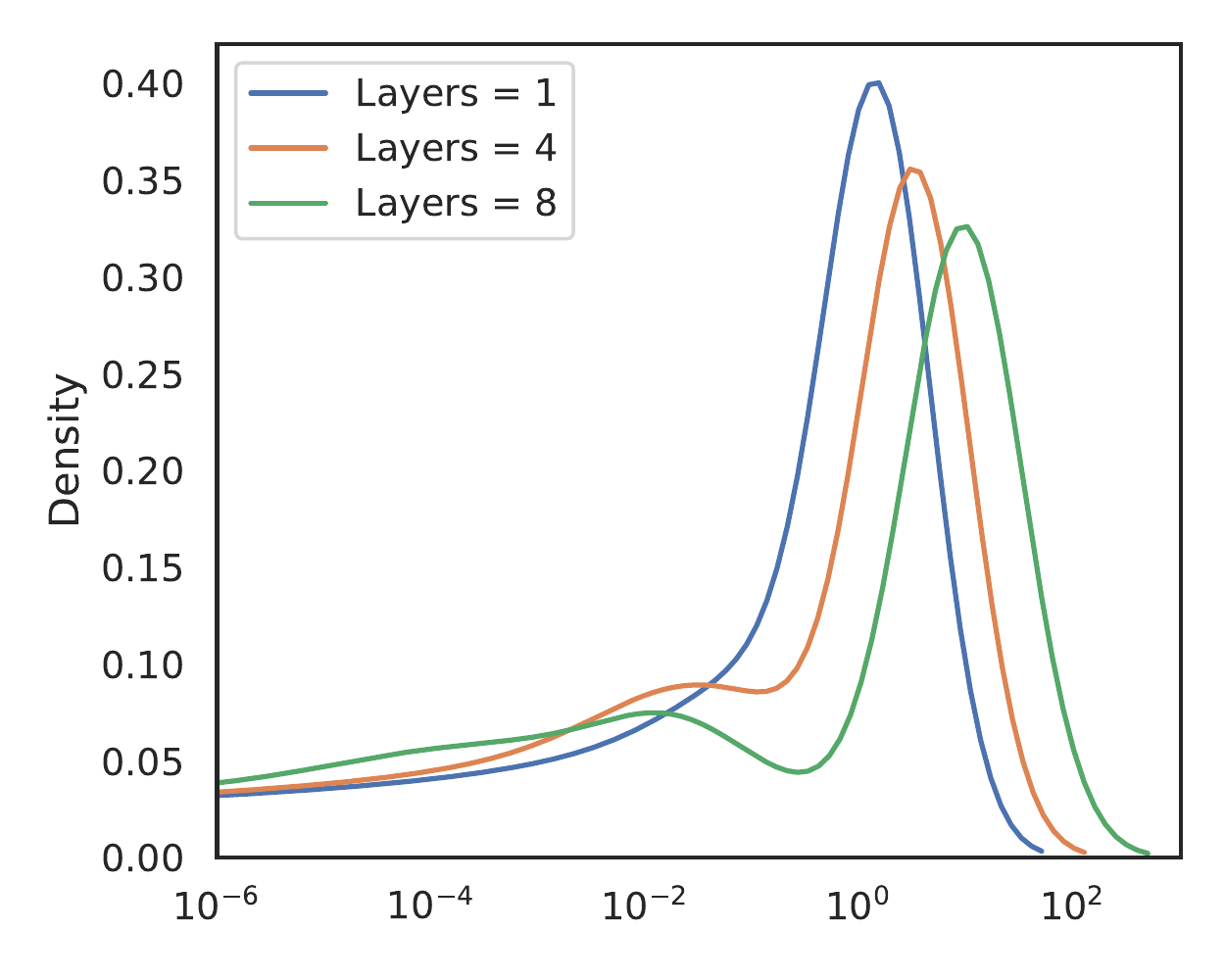}
    \caption{Empirical eigenvalue density distributions of neural networks with varying depth using \texttt{GELU} layers with  stable (first column), on the edge of stability (second column), and unstable dynamics (third column), respectively. $x$ axis represents magnitude of the eigenvalues.}
    \label{fig:depth}
 \end{figure*}

\subsection{Stability Analysis of Neural State Space Models}
Here we explore the practical application of the proposed  dissipativity analysis to neural dynamics trained to model two non-autonomous dynamical systems from process control.

The Continuous Stirred Tank Reactor (CSTR) model is a common simplified mathematical representation a chemical reactor that is equipped with a mixing device to provide efficient mixing of materials. 
The model is described as:
\begin{subequations}
\label{eq:CSTR}
\begin{align}
    & r = k_0  e^{-\frac{E}{R\mathbf{x}_2}}  \mathbf{x}_1 \\
    & \dot{\mathbf{x}}_1 = \frac{q}{V}  (C_{af} - \mathbf{x}_1) - r \\
    &  \dot{\mathbf{x}}_2 = \frac{q}{V}  (T_f - \mathbf{x}_2)  
   + \frac{H}{\rho  c_p}  rA + \frac{A}{V \rho c_p  (\mathbf{u} - \mathbf{x}_2)} 
\end{align}
\end{subequations}
where the measured system states $\mathbf{x}_1$, and $\mathbf{x}_2$ are the concentration of the product and temperature, respectively. The input $\mathbf{u}$ is the temperature of a cooling jacket. The system is inherently non-dissipative and exhibits nonlinear behavior during exothermic reactions with unstable, oscillatory modes that makes the identification process non-trivial.
We assume that the structure of the differential equations and the parameters $\{k_0, E, R, q, V, H, \rho, c_p, A, T_f, C_{af}\}$ are unknown. 

Another system we consider is the Two Tank system (2TS). It consists of two water tanks in series that are connected by a valve. The two inputs include a pump that controls the liquid inflow to the first tank and a valve opening controls the flow between the tanks and can either be fully open or fully closed. The system can be described by the ordinary differential equations:
\begin{subequations}
\label{eq:2tank}
\begin{align}
    & \dot{\mathbf{x}}_1 = 
    \begin{cases}
   (1 - \mathbf{u}_1) c_1 \mathbf{u}_2  - c_2 \sqrt{\mathbf{x}_1},  & \text{if } \mathbf{x}_1 \leq 1\\ 
    0 &  \text{otherwise}
    \end{cases} \\
    &  \dot{\mathbf{x}}_2 = 
            \begin{cases}
     c_1  \mathbf{u}_1  \mathbf{u}_2 + c_2  \sqrt{\mathbf{x}_1} - c_2 \sqrt{\mathbf{x}_2} ,  & \text{if } \mathbf{x}_2 \leq 1\\ 
    0 &  \text{otherwise}
    \end{cases} 
\end{align}
\end{subequations}
where the measured system states $\mathbf{x}_1$, and $\mathbf{x}_2$ denote the liquid levels in first and second tank, respectively.

To identify the unknown system dynamics from data, we consider the following non-autonomous block structured neural state space model:
\begin{equation}
  \label{eq:ssm}
      \mathbf{x}_{t+1} = \mathbf{f}(\mathbf{x}_t) + \mathbf{g}(\mathbf{u}_t)
 \end{equation}
With nonlinear maps $\mathbf{f}(\mathbf{x}_t): \mathbb{R}^{n_x} \to \mathbb{R}^{n_x} $, $\mathbf{g}(\mathbf{u}_t): \mathbb{R}^{n_u} \to \mathbb{R}^{n_x} $ parametrized by deep neural networks~\ref{eq:dnn}.  Here $\mathbf{u}_t$ is an exogenous control signal at time step $t$. In both cases we use GELU activation functions and standard unconstrained weights.
To train the model~\eqref{eq:ssm} for each system we generated 3000 measurements via simulation of the governing differential equations using the Scipy ODEint solver.
We use 1000 time steps for training, model selection, and testing.
To demonstrate generalization we show the traces of the learned neural dynamics given only an initial condition and the sequence of control inputs compared to the ground truth simulations (right of Fig.~\ref{fig:system_1}, and Fig.~\ref{fig:system_2}). 

After training we perform the dissipativity analysis  on the trained neural networks composing the model~\eqref{eq:ssm} by evaluating their local linear operator norms
$||\mathbf{A}^{\star}(\mathbf{x})||_2$ obtained via Lemma~\ref{lem:lpv}.  
Left hand sides of Fig.~\ref{fig:system_1} and Fig.~\ref{fig:system_2} visualize the state space regions 
of the state transition dynamics $\mathbf{f}(\mathbf{x}_t)$ associated with its spectral norm. 
Middle plots of Fig.~\ref{fig:system_1} and Fig.~\ref{fig:system_2}
show the spectral norms of $\mathbf{g}(\mathbf{u}_t)$ as a function of the input space showing the nonlinearity of the input dynamics.
Here the values below $1$ represent dissipative regions of the neural networks.
For the CSTR system (Fig.~\ref{fig:system_1}), the analysis reveals state transition maps with unstable regions $||\mathbf{A}^{\star}_{f}(\mathbf{x})||_2>1$ for states in the lower triangular region of the state space. This is in accordance to the unstable dynamics of the exothermic CSTR system, whose dynamics is stabilized by the inputs representing cooling.
In the two tank case (Fig.~\ref{fig:system_2} ), the best performing model learned globally dissipative state transition dynamic maps $||\mathbf{A}^{\star}_{f}(\mathbf{x})||_2<1$, which is in line with the stable nature of the underlying physical system.
Therefore as demonstrated the proposed method can be used to analyze the dissipativity of data-driven neural models~\eqref{eq:ssm}. 
\begin{figure*}[htb!]
    \centering
    \includegraphics[width=0.25\textwidth, trim=50 20 20 38, clip]{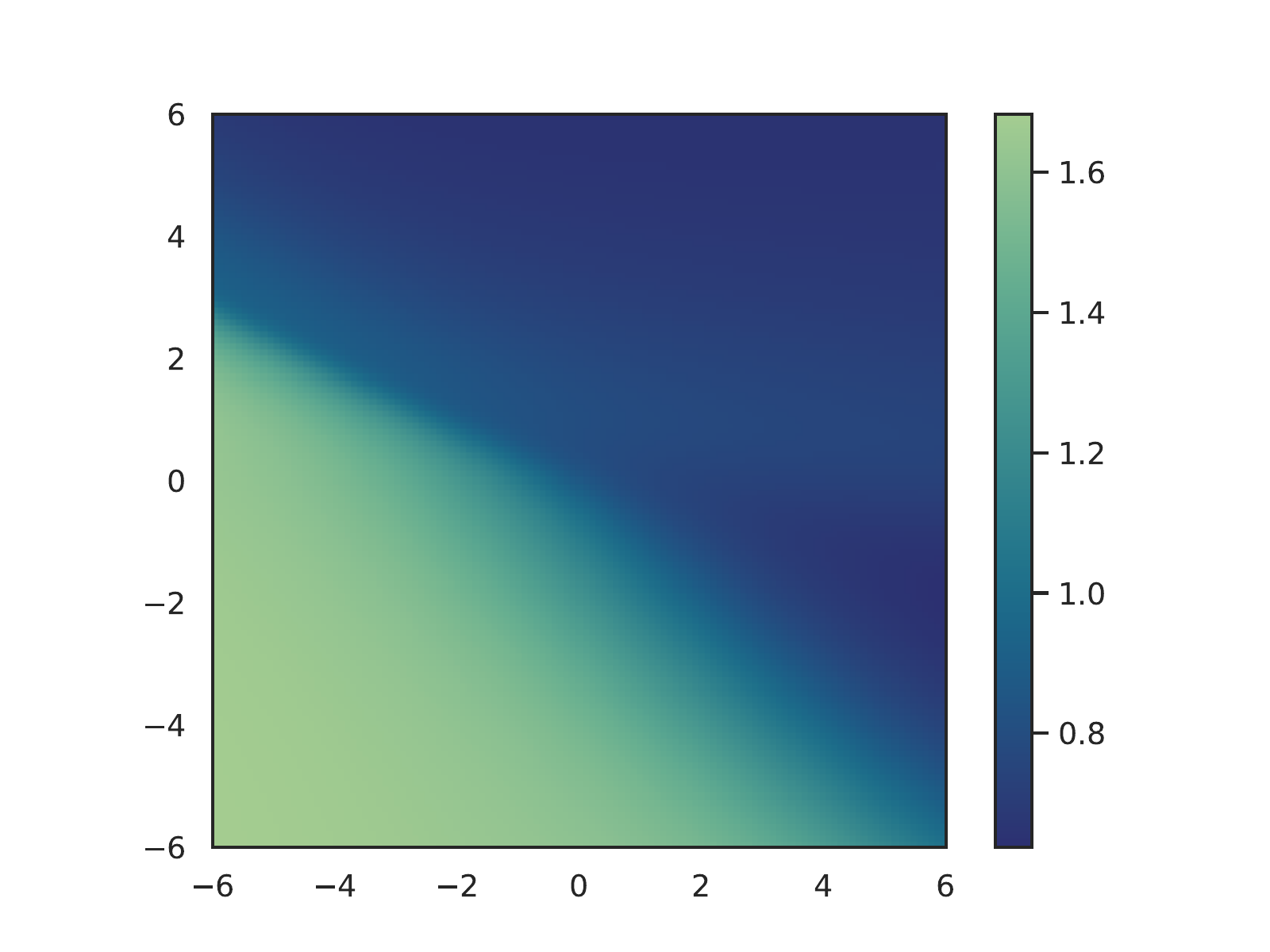}
    \includegraphics[width=0.25\textwidth, trim=30 20 20 38, clip]{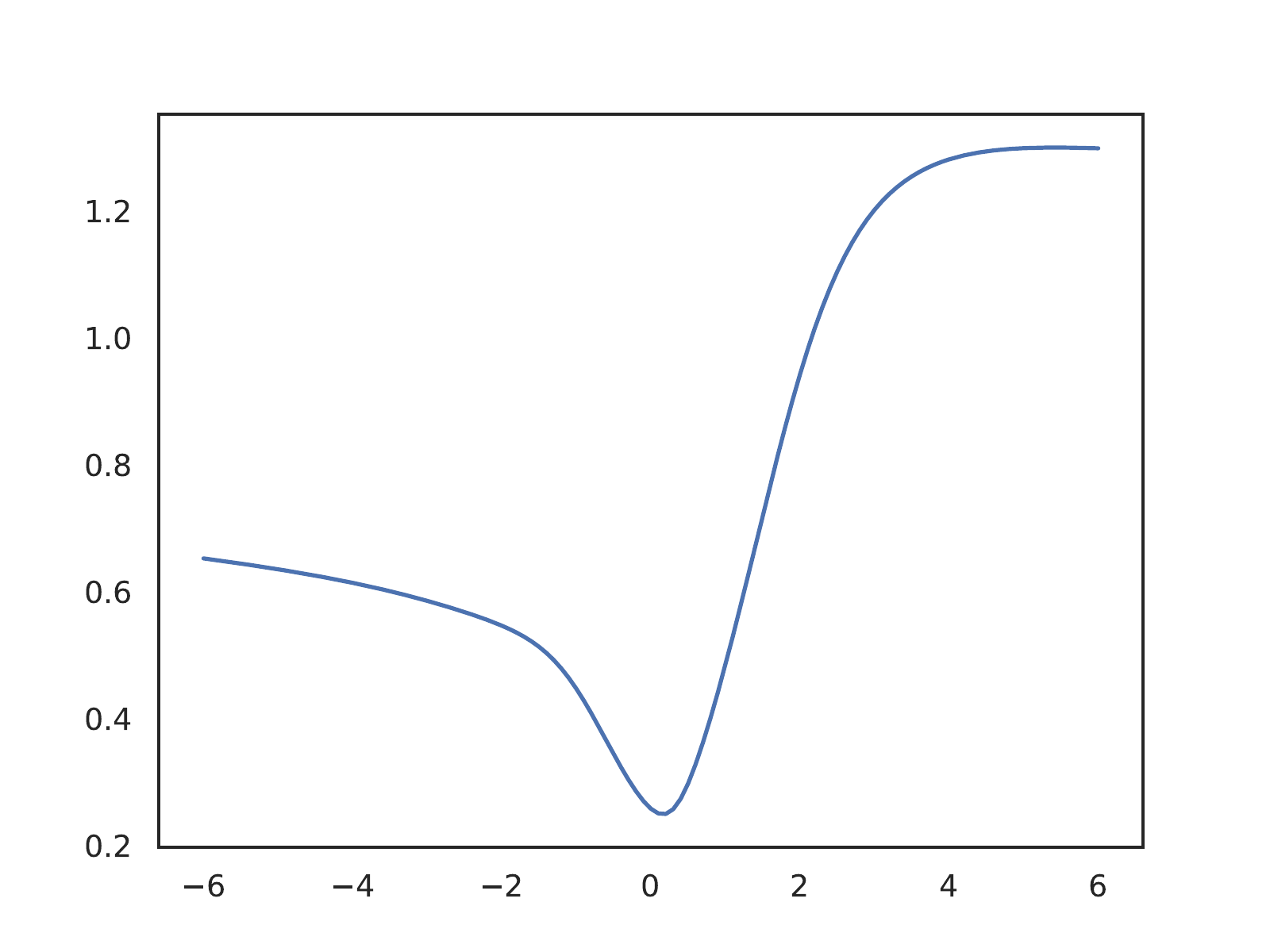}
    \includegraphics[width=0.32\textwidth, trim=30 20 20 38, clip]{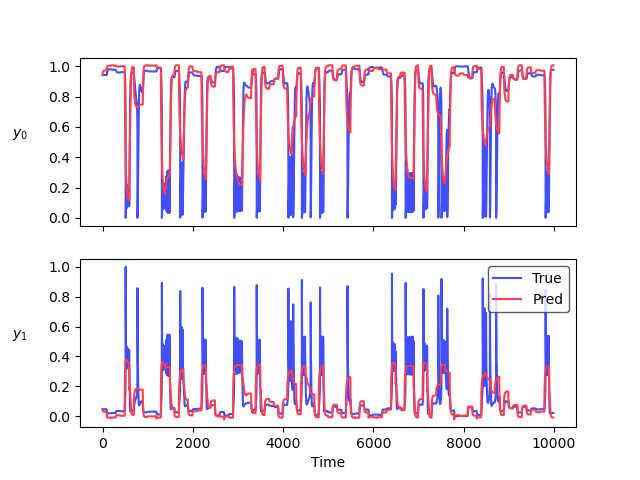}
    \caption{Dissipativity analysis of learned neural surrogate~\eqref{eq:ssm} of the CSTR system.
    Left: state space regions associated with spectral norms of two dimensional state dynamics map $\mathbf{f}(\mathbf{x})$. Middle: spectral norms of $\mathbf{g}(\mathbf{u})$ as a function of inputs $\mathbf{u}$. Right: simulation traces of the learned model (red) compared to ground truth (blue).}
    \label{fig:system_1}
\end{figure*}
\begin{figure*}[htb!]
    \centering
    \includegraphics[width=0.25\textwidth, trim=50 20 20 38, clip]{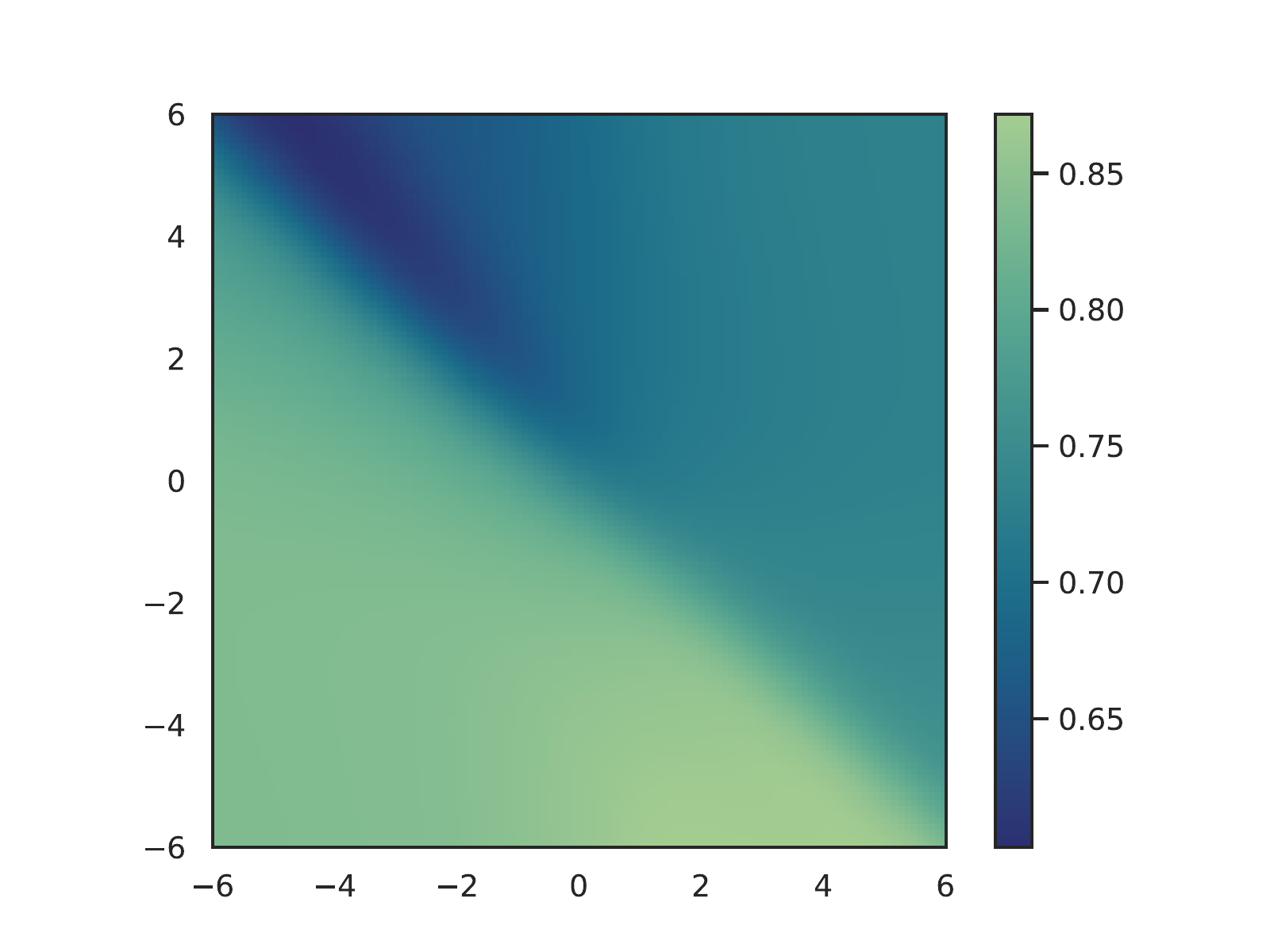}
    \includegraphics[width=0.25\textwidth, trim=50 20 20 38, clip]{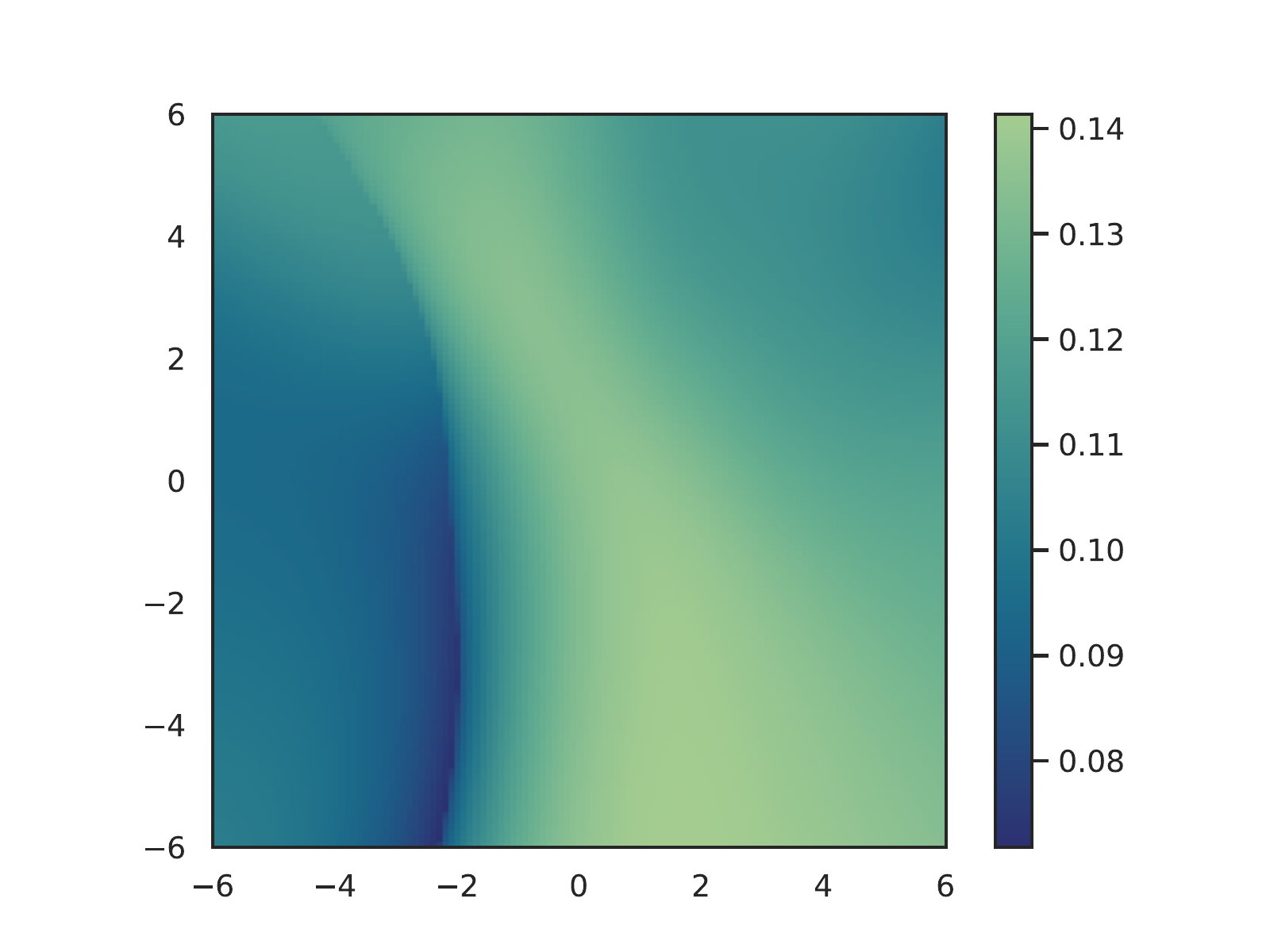}
    \includegraphics[width=0.32\textwidth, trim=30 20 20 38, clip]{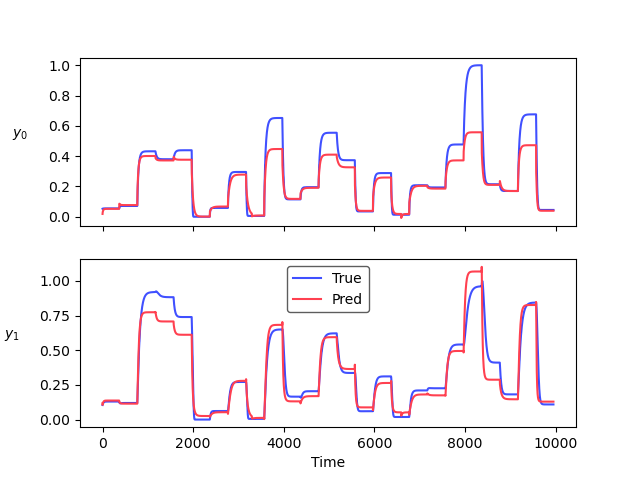}
    \caption{Dissipativity analysis of learned neural surrogate~\eqref{eq:ssm} of the Two Tank system. Left: state space regions associated with spectral norms of two dimensional state dynamics map  $\mathbf{f}(\mathbf{x})$. Middle: spectral norms of $\mathbf{g}(\mathbf{u})$ as a function of inputs $\mathbf{u}$. Right: simulation traces of the learned model (red) compared to ground truth (blue).}
    \label{fig:system_2}
\end{figure*}

 \subsection{Attractors of Deep Neural Dynamics}

Understanding dynamical effects of individual components of deep neural networks allow us to analyze neural dynamics with different attractors.
 In neuroscience, different types of attractor networks have been associated with different brain functions~\cite{Eliasmith2007}.
For instance, it is known that cyclic attractors 
can describe repetitive behaviors such as walking, line attractors have been linked with oculomotor control and integrators in control theory, while point attractors have been linked with associative memory, pattern completion, noise reduction, and classification tasks~\cite{Eliasmith2007}.
Here we empirically demonstrate the expressive capacity of deep neural network for dynamical systems by generating six different types of attractors and analyze their dissipativity, eigenvalue spectra, and state space partitioning. 
Fig.~\ref{fig:attractors} plots state space trajectories of different neural networks generating:
single equilibrium, multiple equilibria, line attractor, limit cycle, quasi-periodic attractor, and unstable attractor, respectively.
Fig.~\ref{fig:AstarRegions}  plots the scalar field of the local operator norms  $||\mathbf{A}^{\star}||$, thus visualizes the network sensitivity to perturbations of the state space, where warmer regions correspond to higher and colder regions to lower sensitivity, respectively. The norm values below $1$ represent dissipative regions of the state space.
Fig.~\ref{fig:attractors_spectra} displays associated eigenvalue spectra. We can observe that the number of eigenvalues roughly corresponds to the number of unique state space regions in shown Fig.~\ref{fig:AstarRegions}. Moreover, eigenvalues with larger dispersion in the complex plane of Fig.~\ref{fig:attractors_spectra}
can be linked with more complex state space trajectories shown in Fig.~\ref{fig:attractors}.
   \begin{figure*}[htbp!]
        \centering
        \hspace{-1.5cm}
        \includegraphics[width=0.20\textwidth, trim=70 20 80 40, clip]{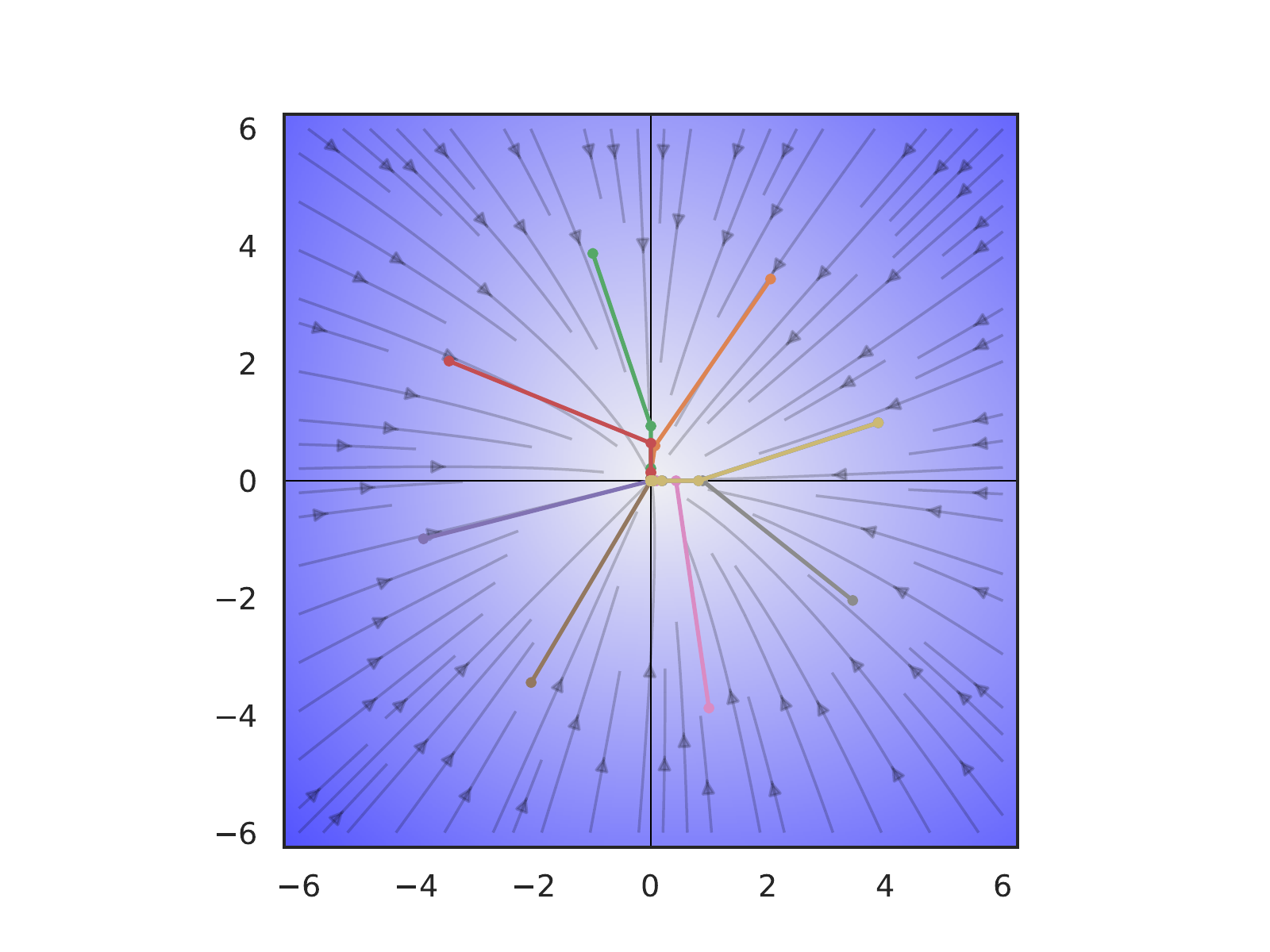}
        \includegraphics[width=0.20\textwidth, trim=70 20 80 40, clip]{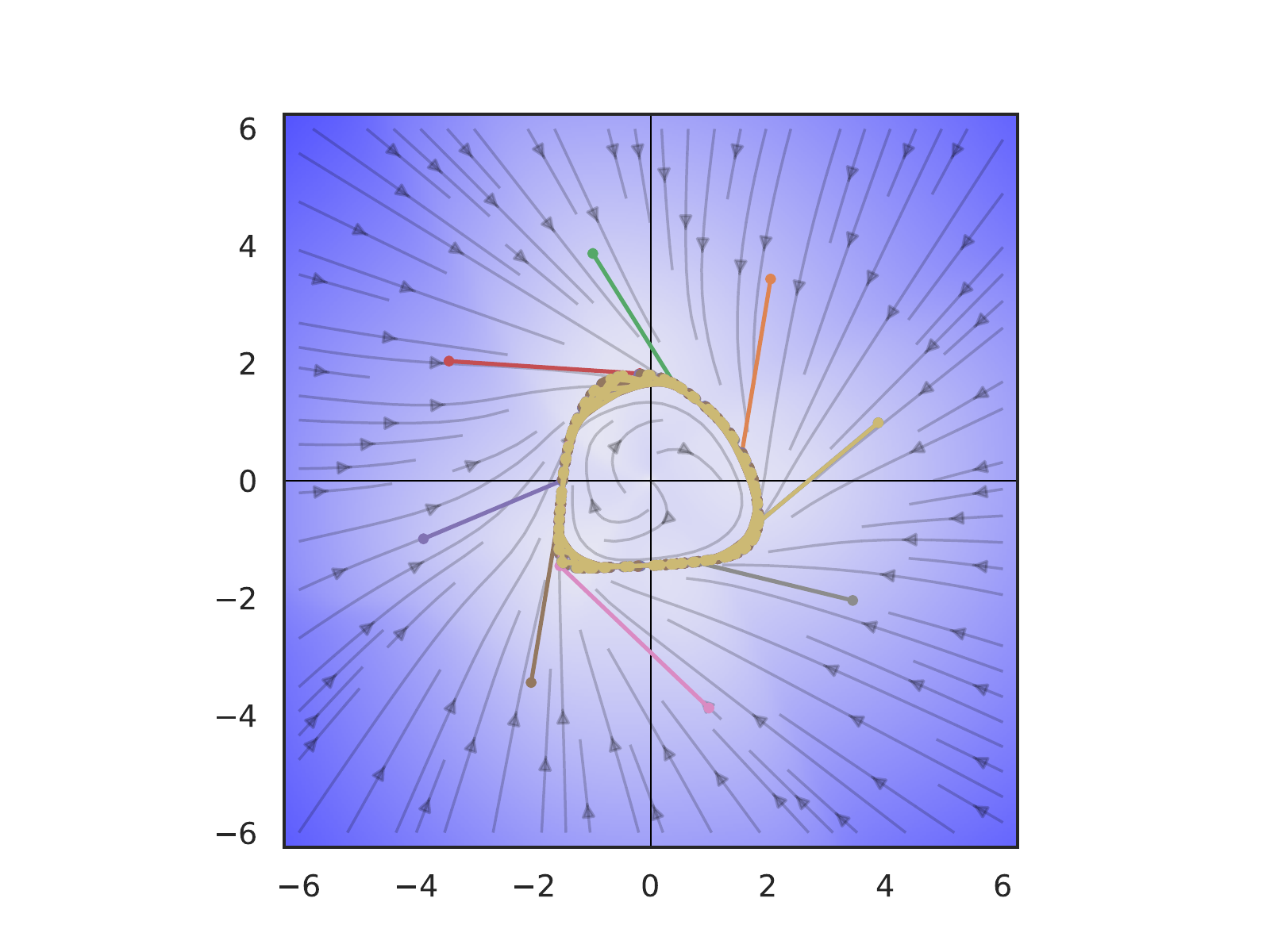} 
        \includegraphics[width=0.20\textwidth, trim=70 20 80 40, clip]{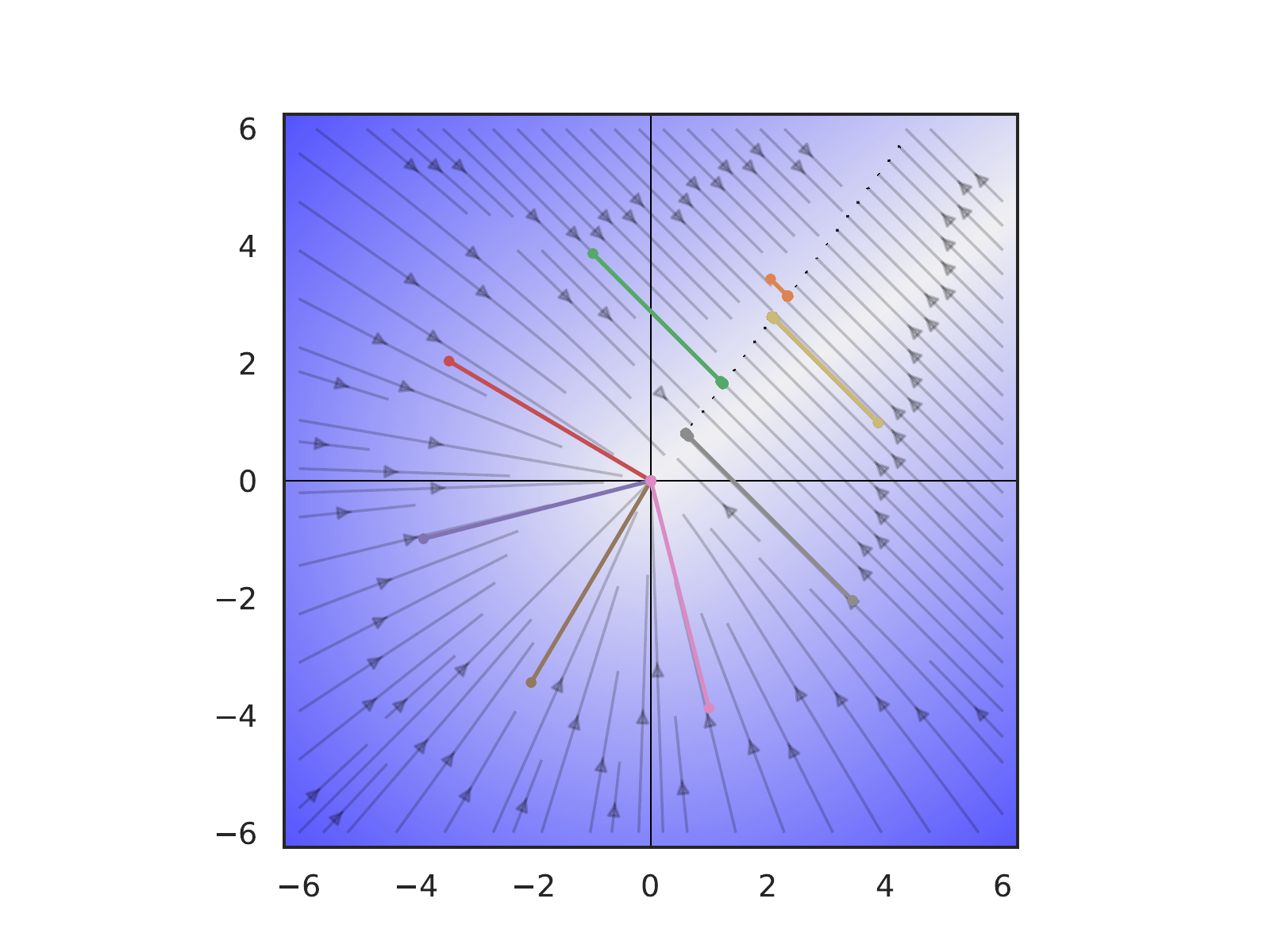} \\ \vspace{16pt}
         \hspace{-1.5cm}
            \includegraphics[width=0.20\textwidth, trim=70 20 80 40, clip]{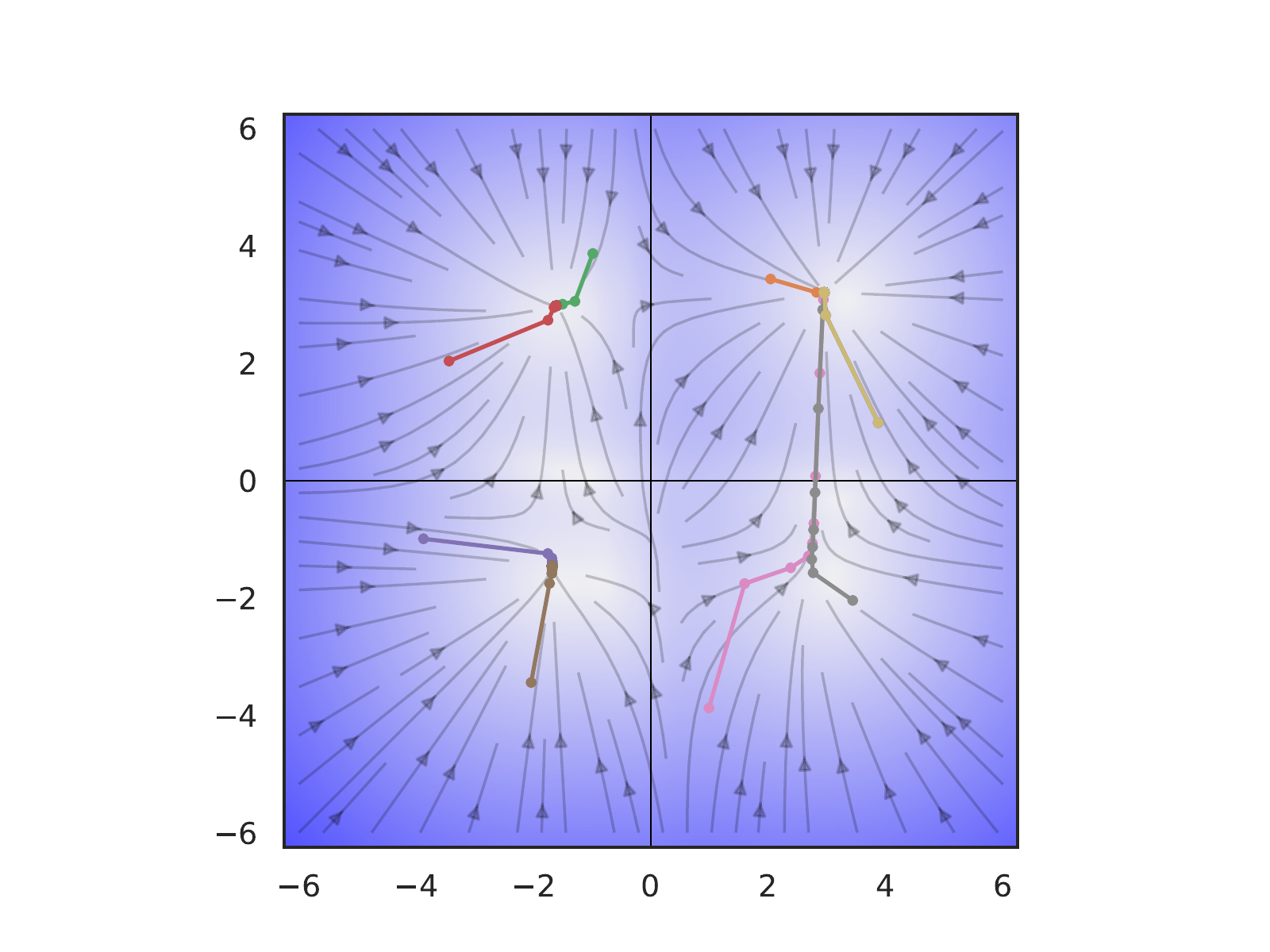} 
        \includegraphics[width=0.20\textwidth, trim=70 20 80 40, clip]{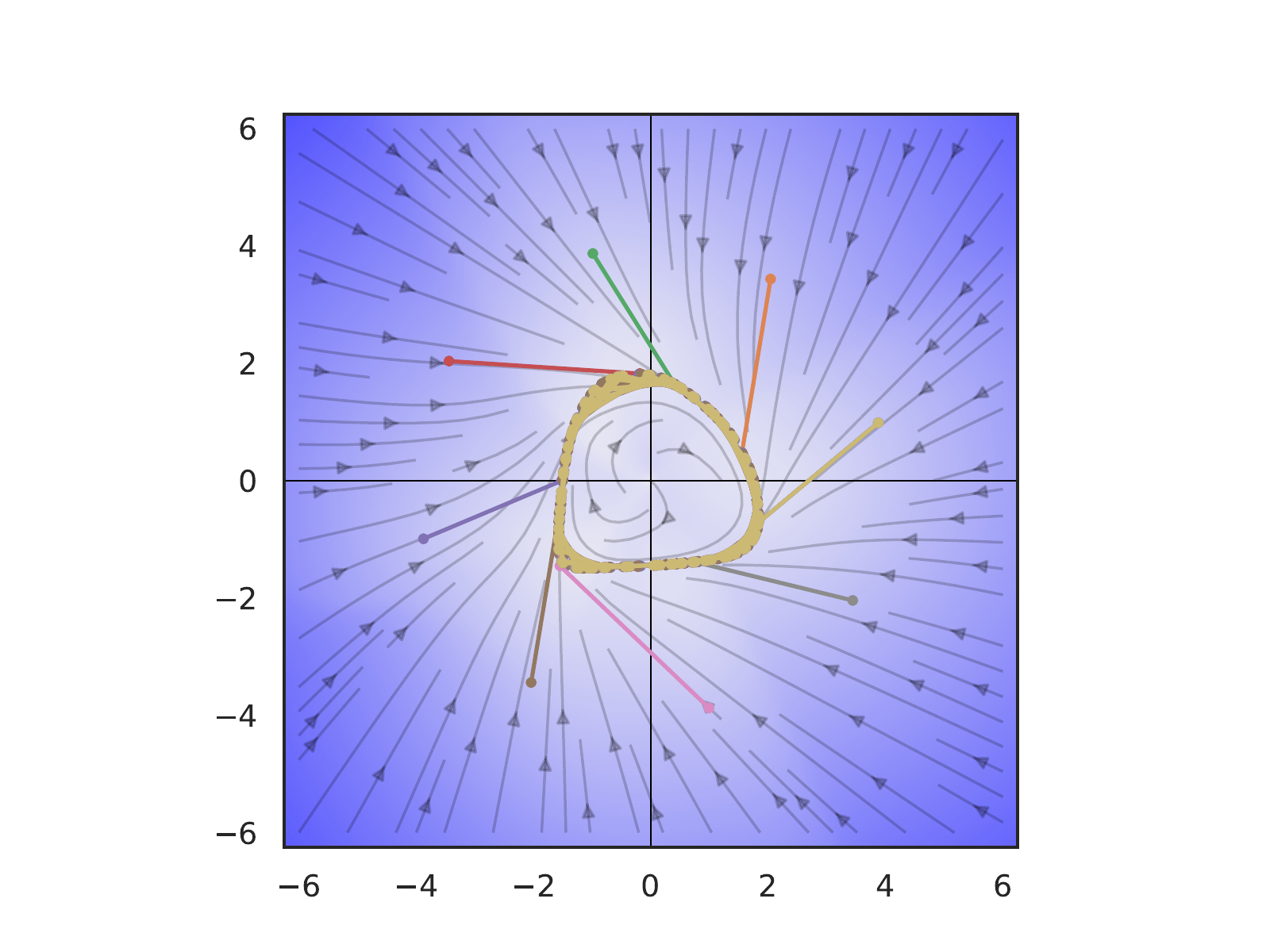} 
        \includegraphics[width=0.20\textwidth, trim=70 20 80 40, clip]{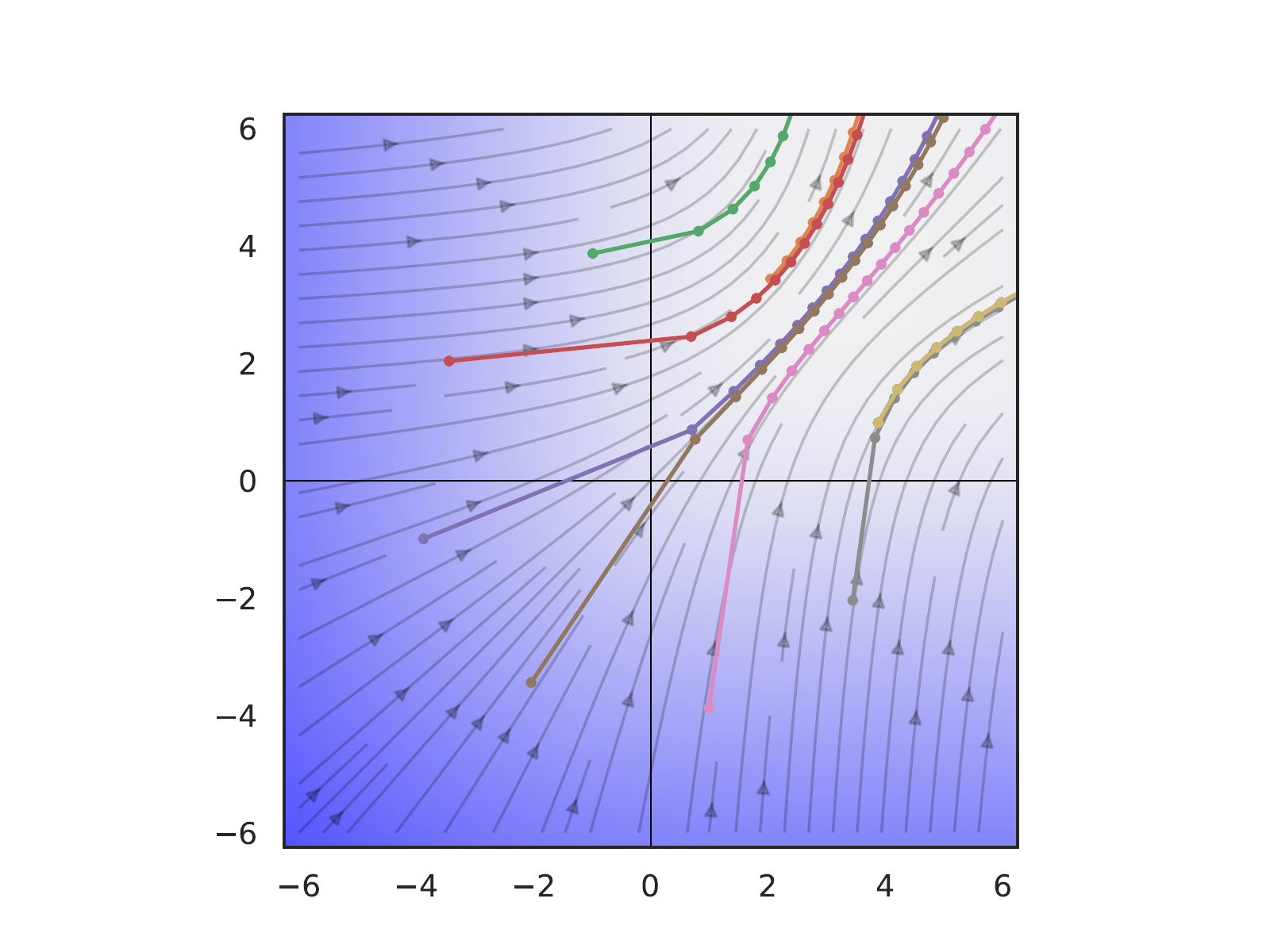} 
        \caption{Different attractor types generated by deep neural networks. From left to right: single point,  limit cycle,
        line attractor, multiple points, quasi-periodic attractor, unstable dynamics.}
        \label{fig:attractors}
    \end{figure*}
    \begin{figure*}[htbp!]
        \centering
        \includegraphics[width=0.20\textwidth, trim=60 20 20 38, clip]{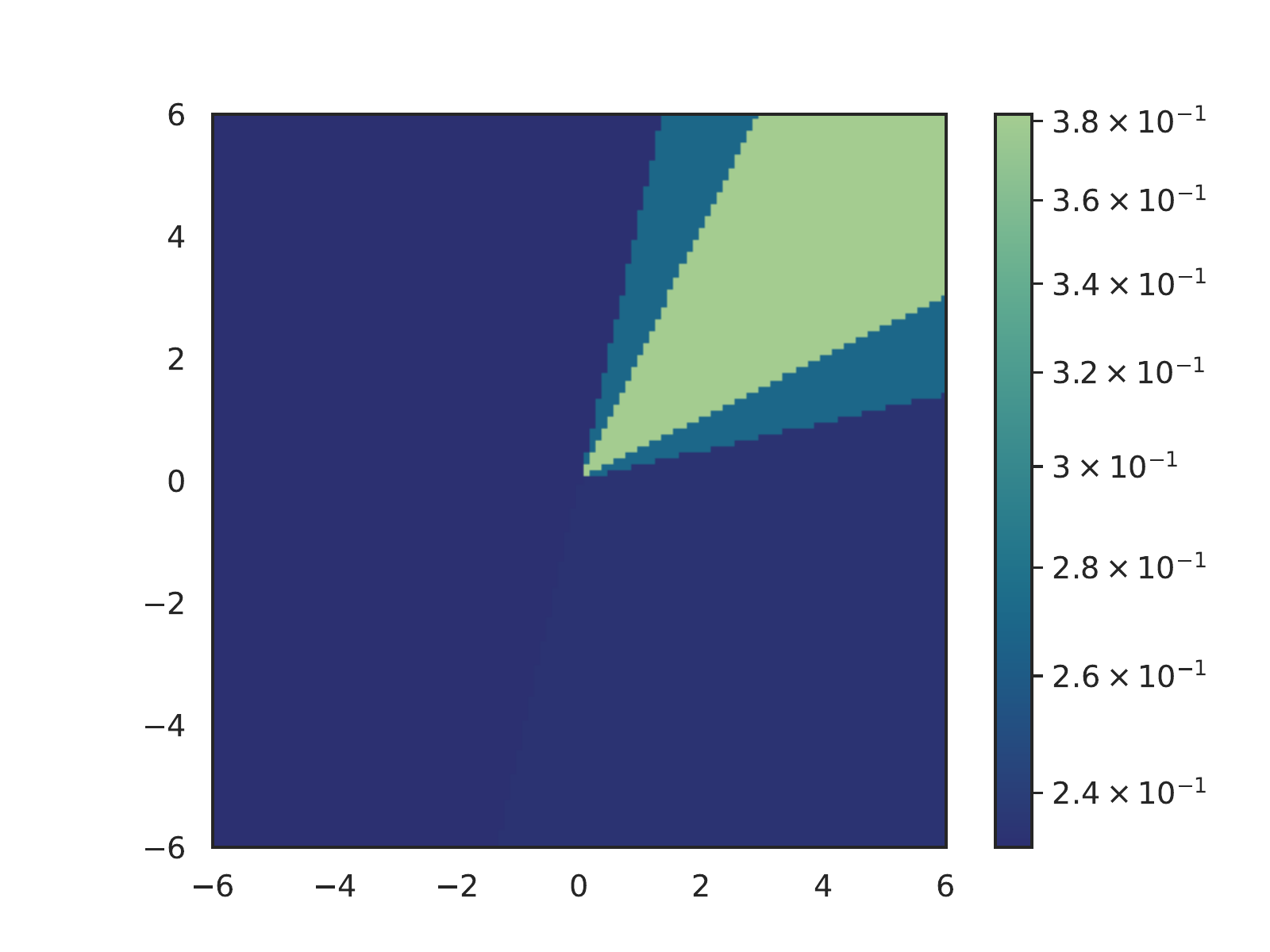} 
        \includegraphics[width=0.20\textwidth, trim=60 20 20 38, clip]{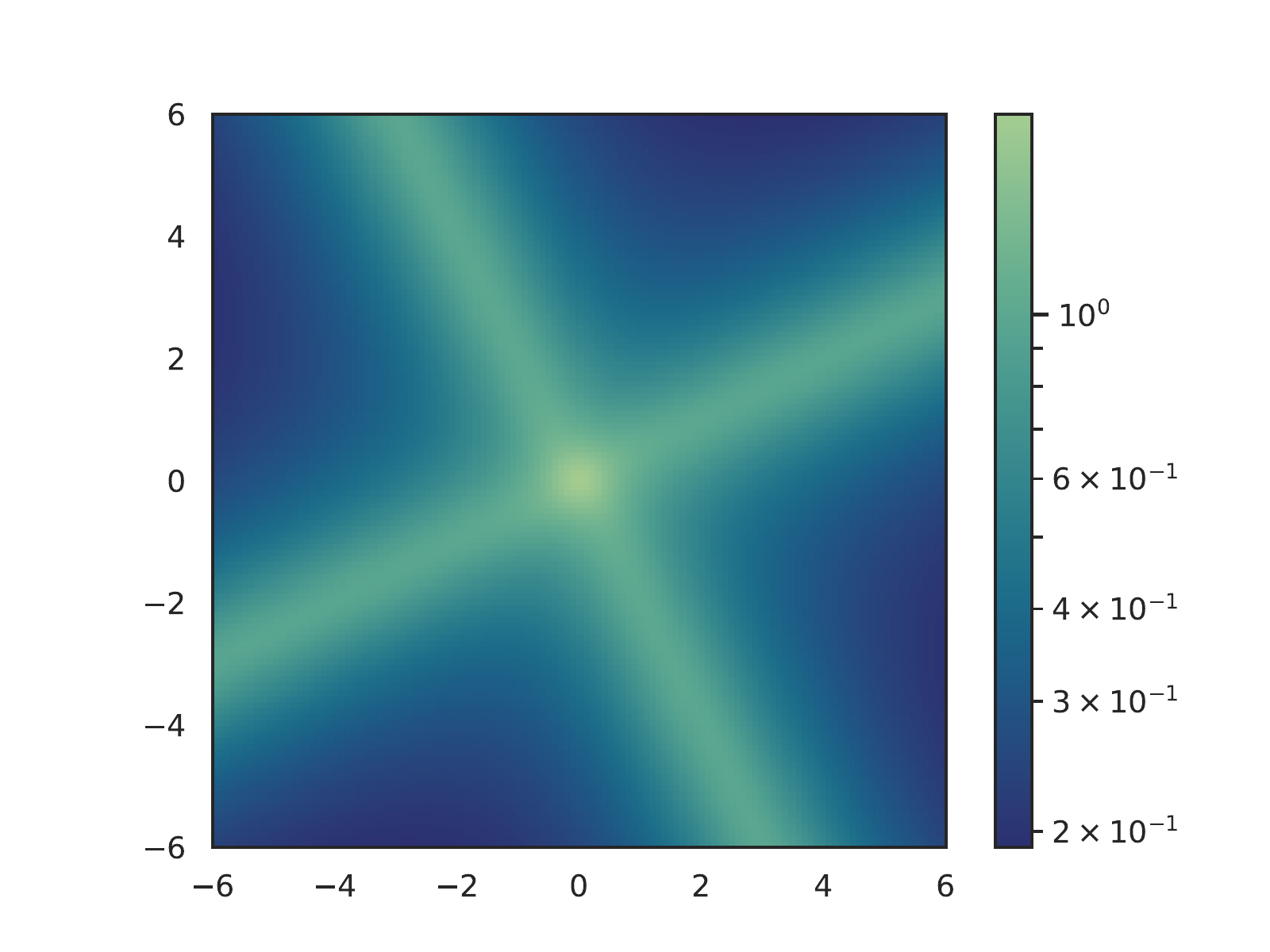} 
        \includegraphics[width=0.20\textwidth, trim=60 20 20 38, clip]{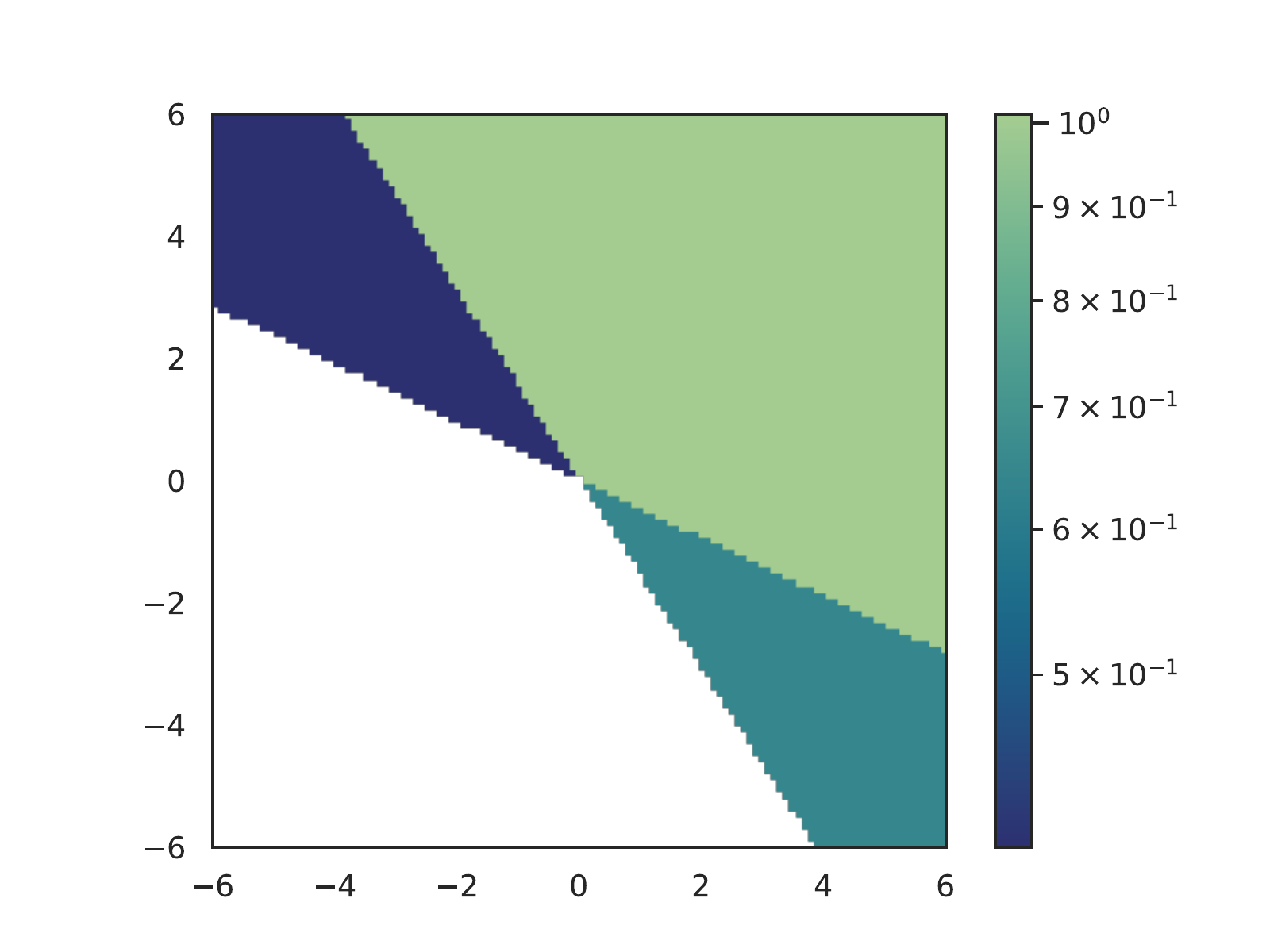} \\ \vspace{16pt}
        \includegraphics[width=0.20\textwidth, trim=60 20 20 38, clip]{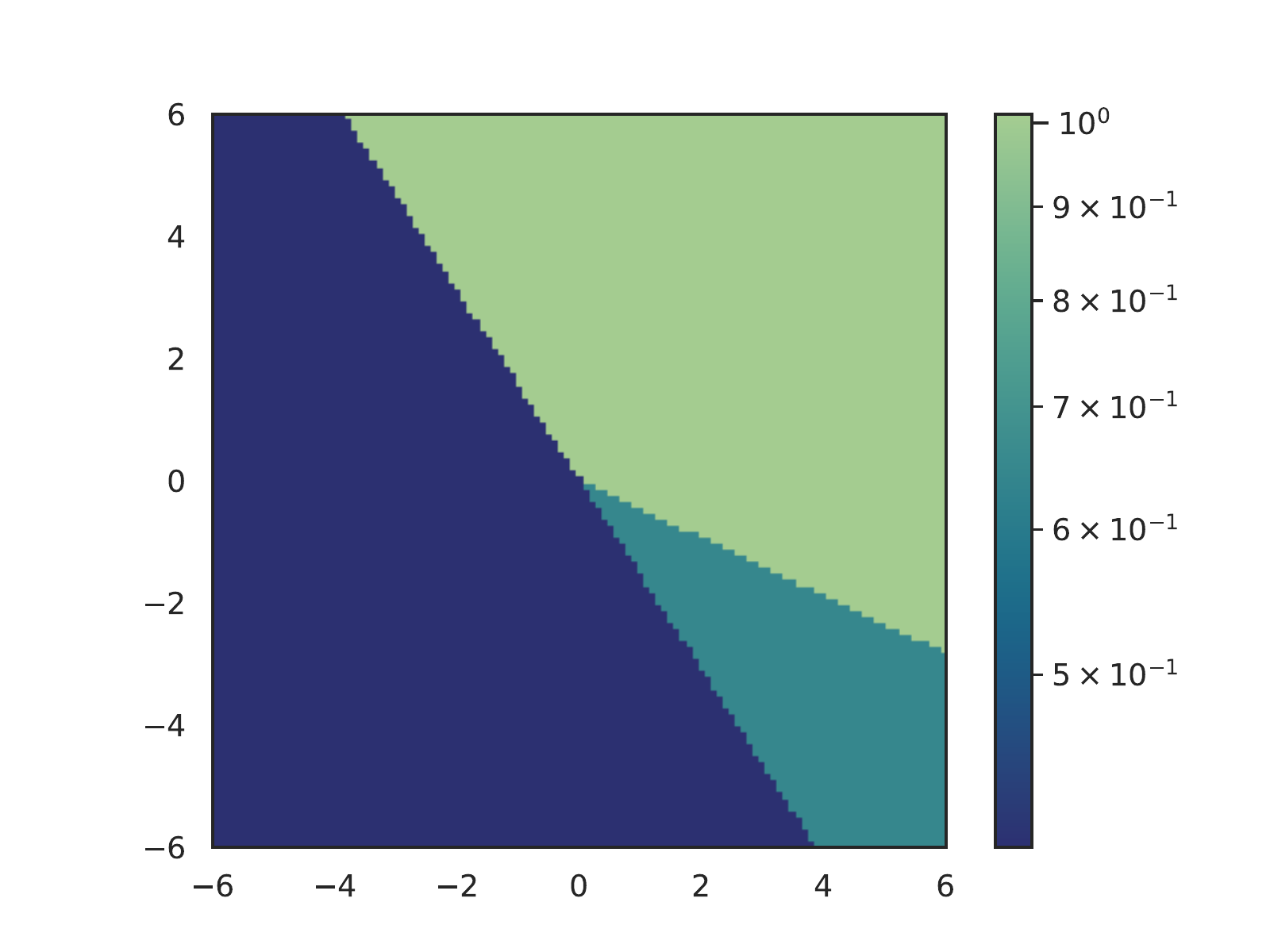} 
        \includegraphics[width=0.20\textwidth, trim=60 20 20 38, clip]{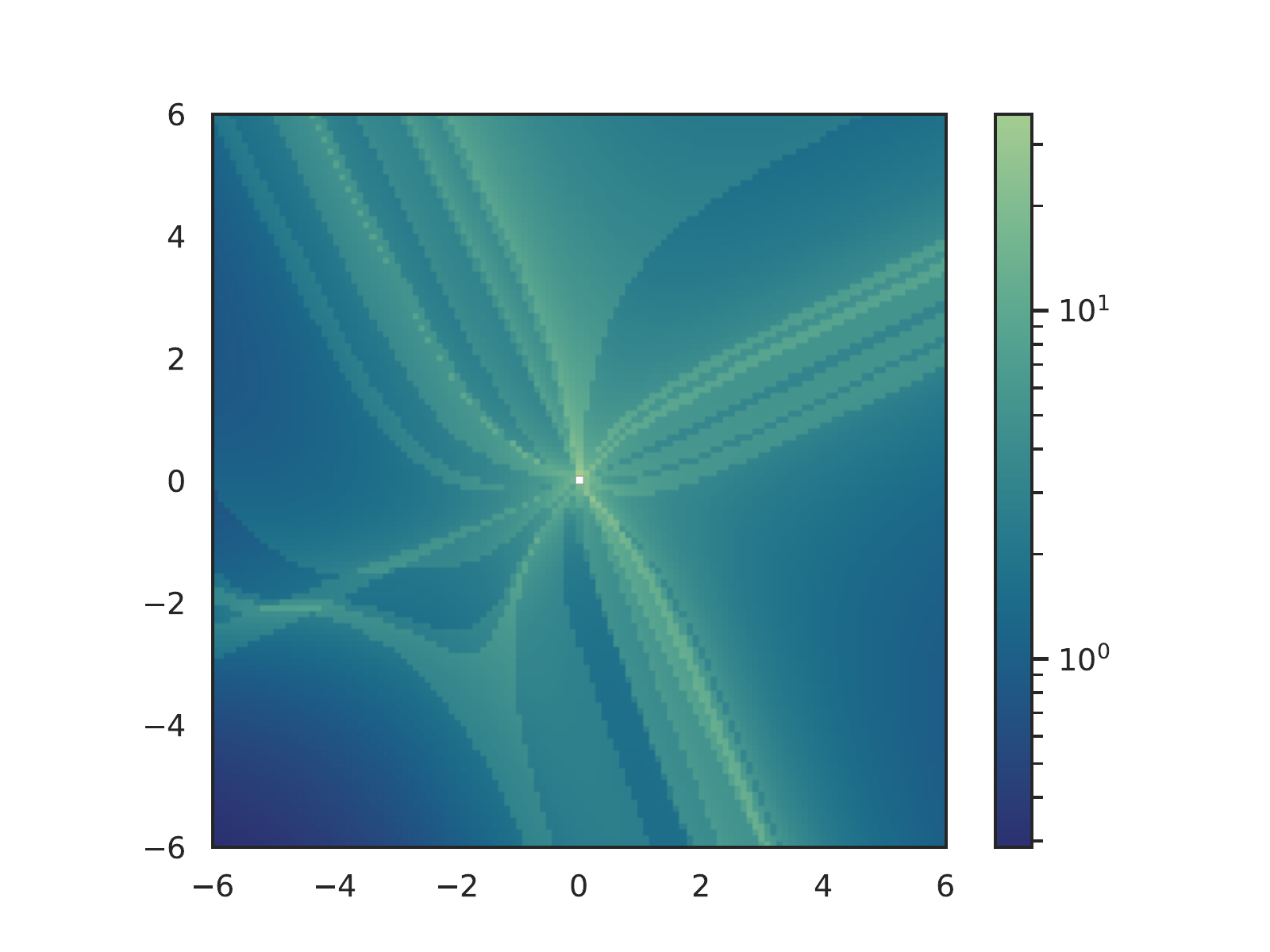} 
        \includegraphics[width=0.20\textwidth, trim=60 20 20 38, clip]{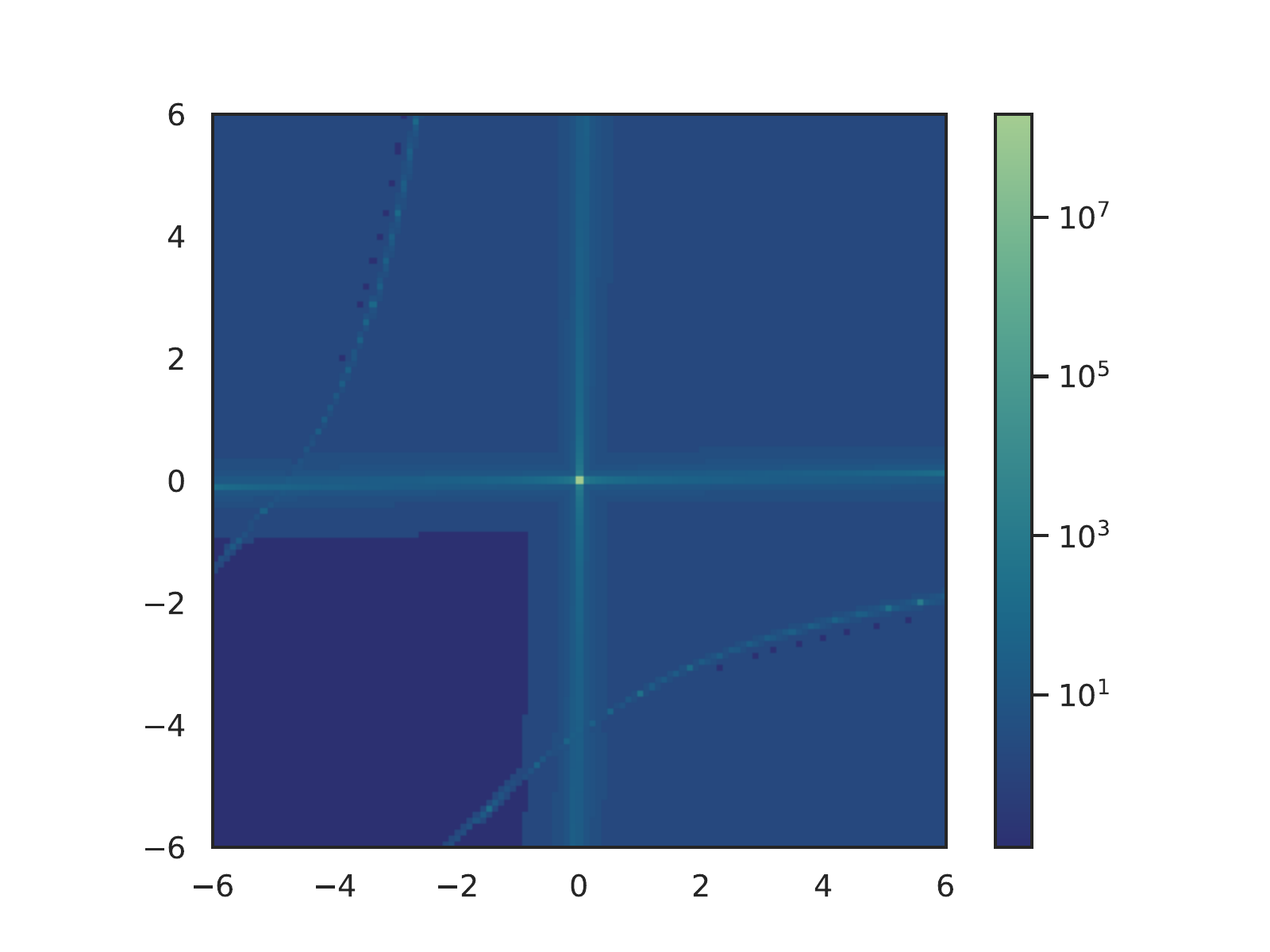} 
        \caption{State space regions represented by different values of local operator norms  corresponding to Fig. \ref{fig:attractors}.}
        \label{fig:AstarRegions}
    \end{figure*}
    \begin{figure*}[htbp!]
        \centering
         \hspace{-1.0cm}
        \includegraphics[width=0.20\textwidth, trim=50 0 60 0, clip]{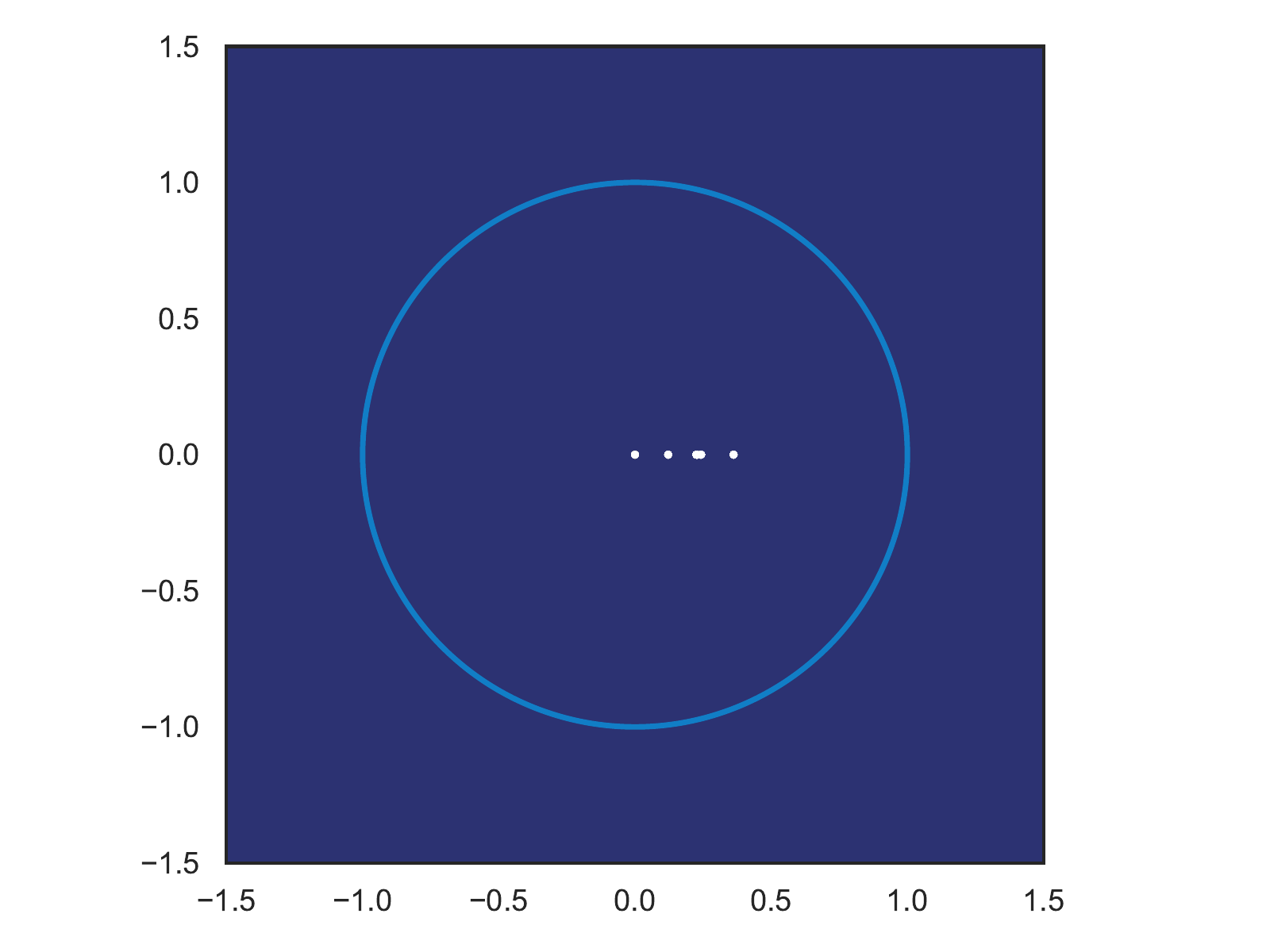} 
        \includegraphics[width=0.20\textwidth, trim=50 0 60 0, clip]{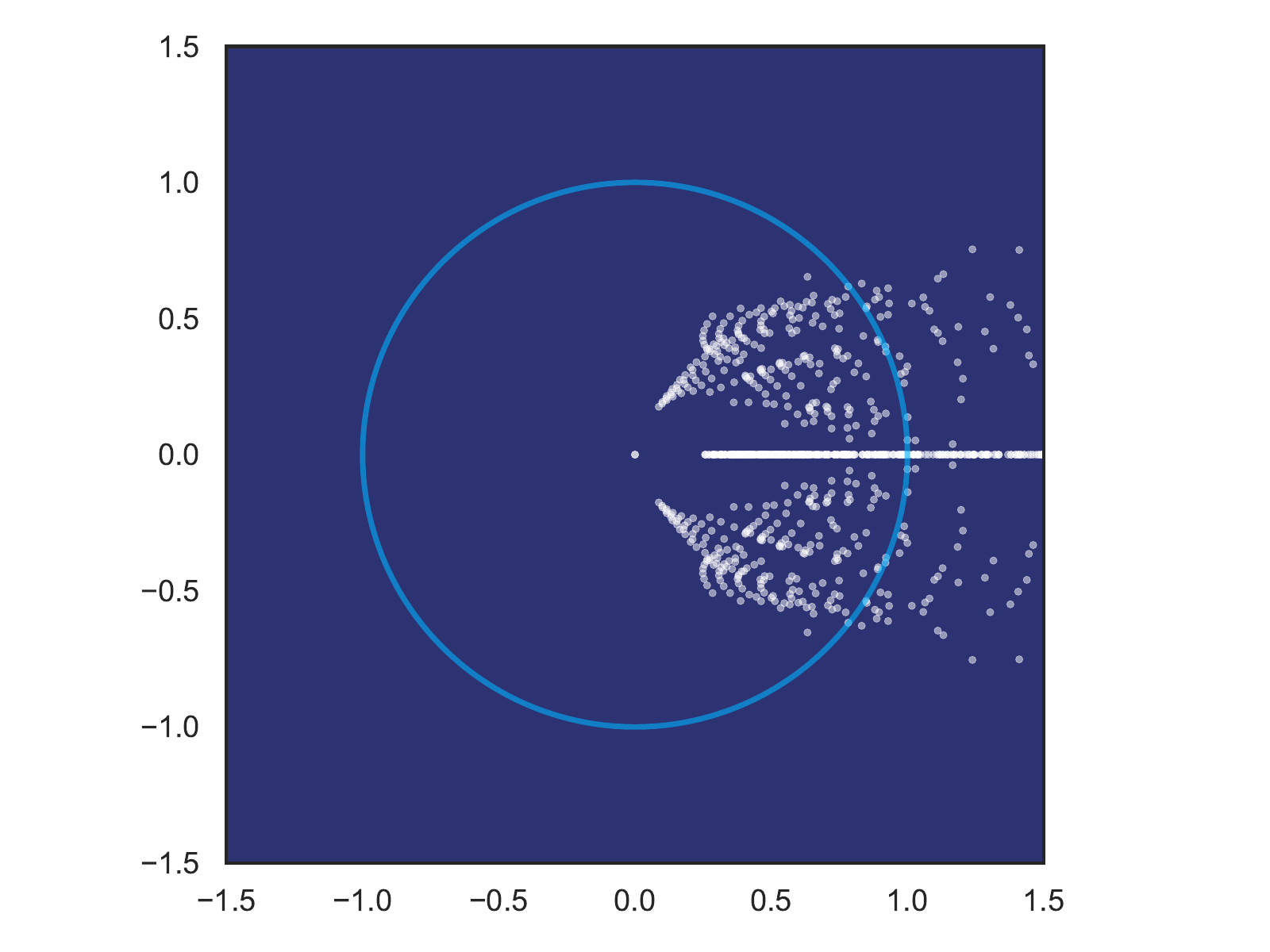} 
        \includegraphics[width=0.20\textwidth, trim=50 0 60 0, clip]{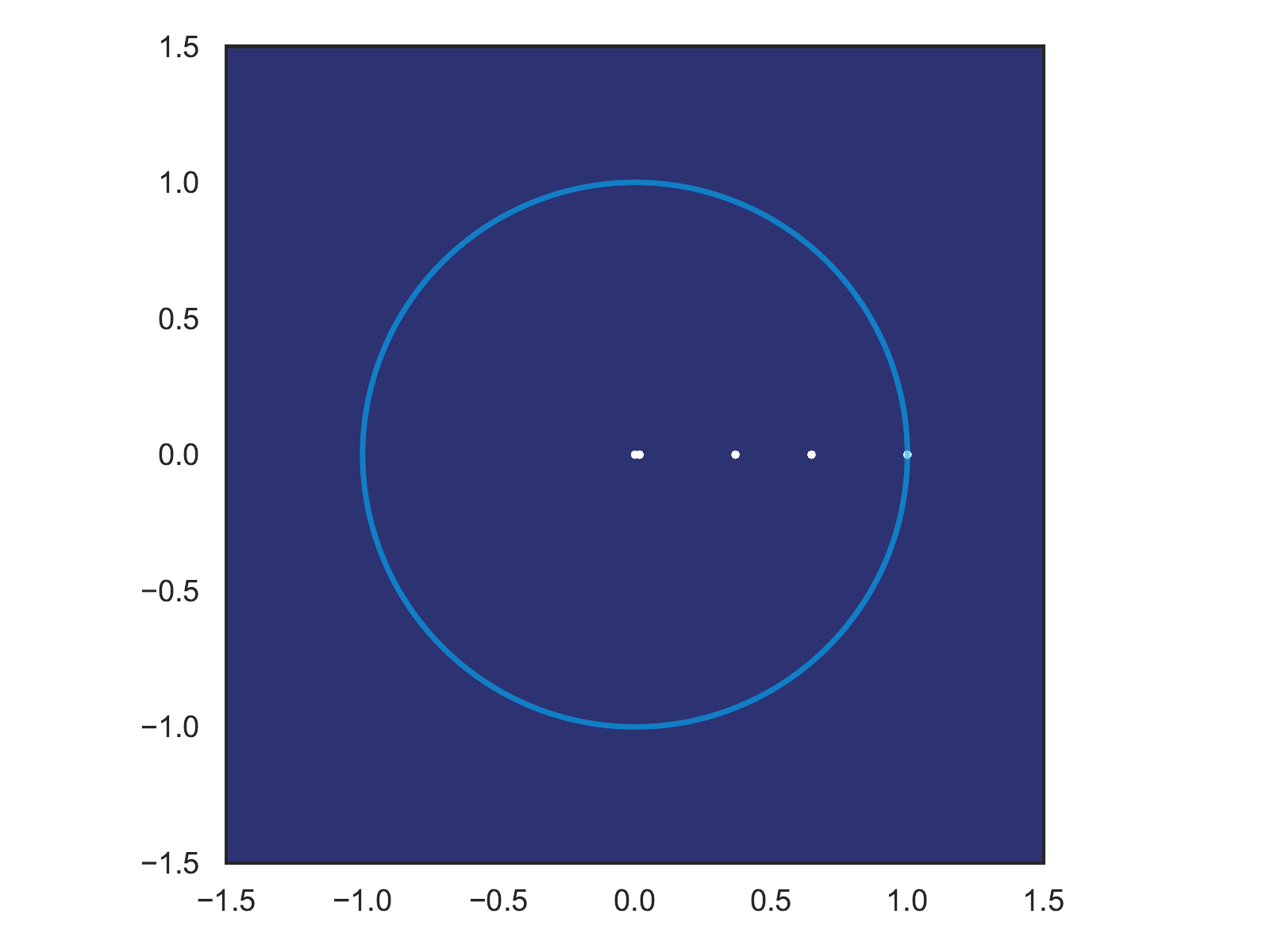} \\  \hspace{-1.0cm}
         \includegraphics[width=0.20\textwidth, trim=50 0 60 0, clip]{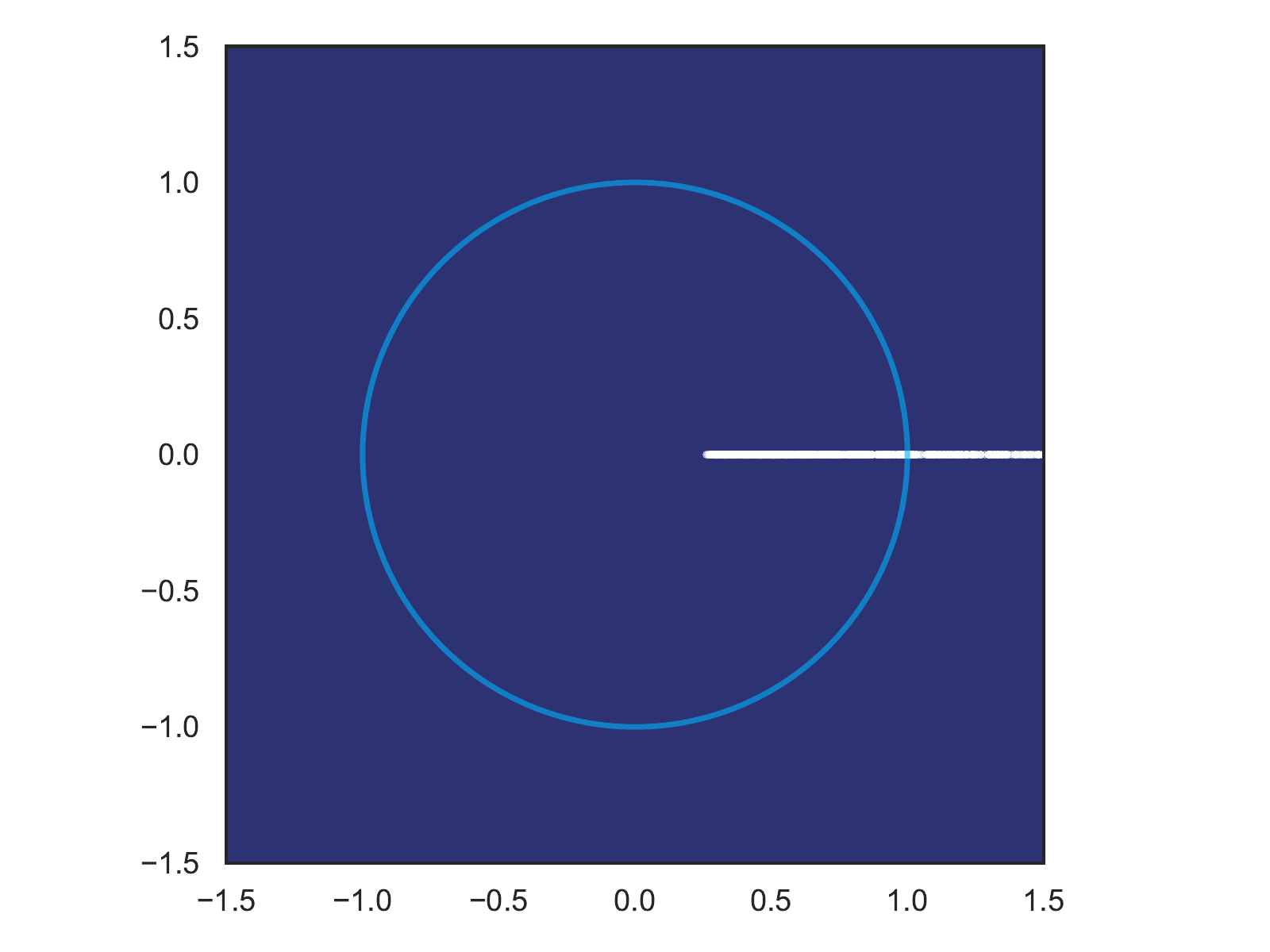} 
        \includegraphics[width=0.20\textwidth, trim=50 0 60 0, clip]{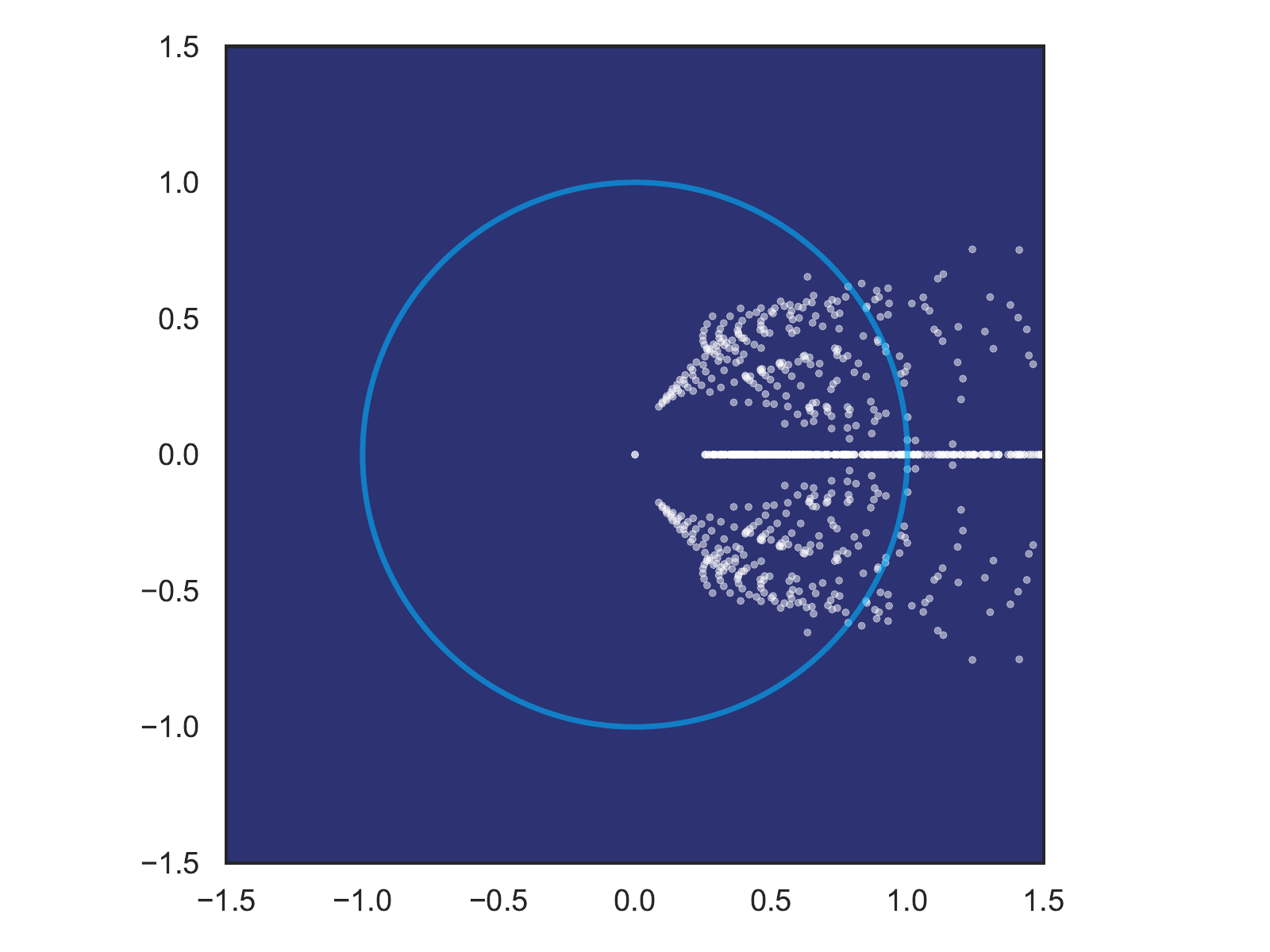} 
        \includegraphics[width=0.20\textwidth, trim=50 0 60 0, clip]{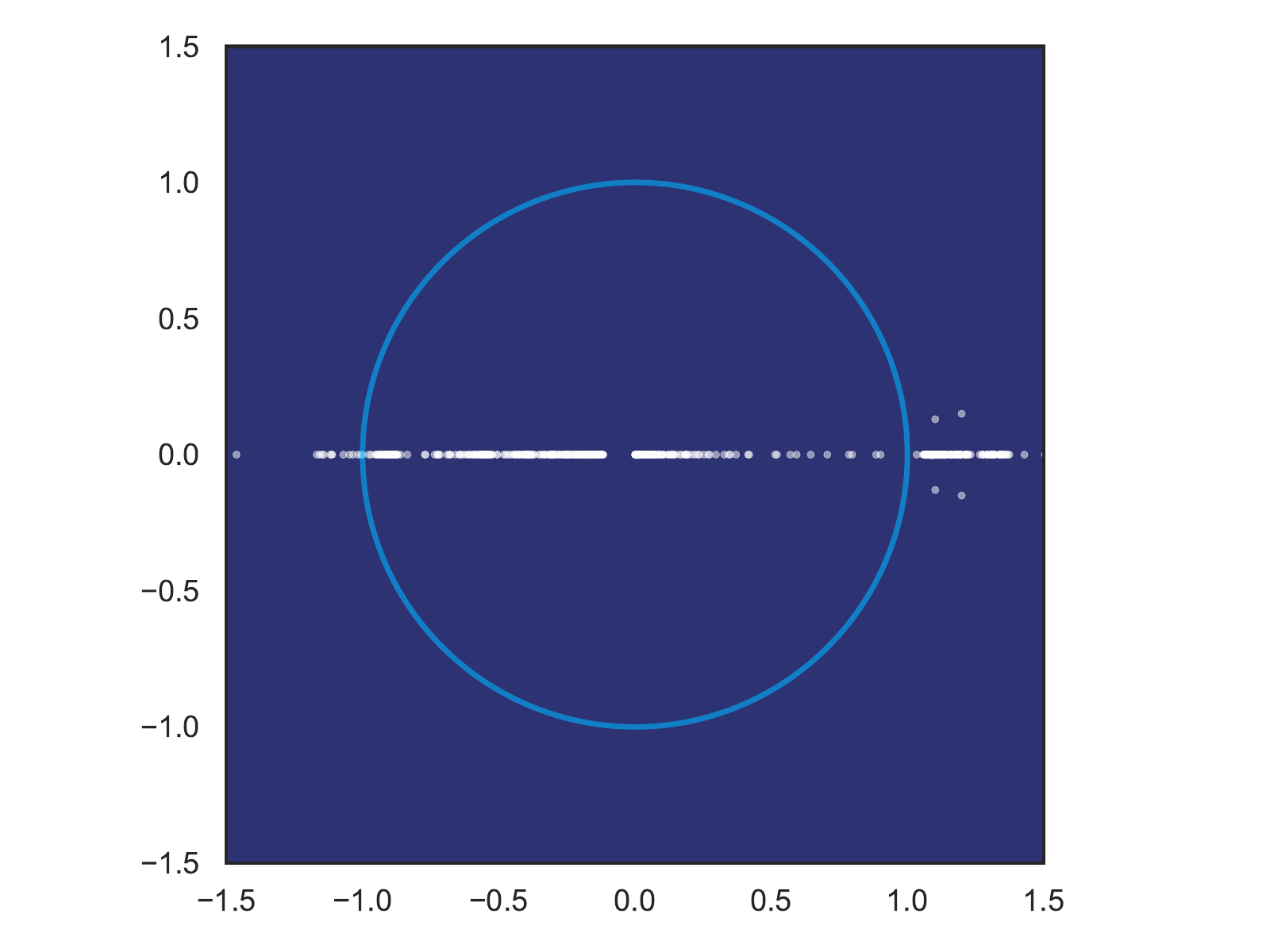} 
        \caption{Eigenvalue spectra in the complex plane associated with different attractor types in Fig.~\ref{fig:attractors}. }
        \label{fig:attractors_spectra}
    \end{figure*}
    
From our empirical analysis, we derive the following observations.
Single point attractors and limit cycles can be generated by dissipative deep neural networks with zero bias terms. 
Multi-point attractors can not be obtained with zero bias terms.
For generating more complex trajectories such as limit cycle attractors, the network must have both stable and unstable regions as given by Bendixson–Dulac theorem~\cite{mccluskey_bendixson_dulac_1998}. Thus more complex attractors must have both dissipative and non-dissipative regions in the state space.

\section{Conclusions}

In this work, we provide a dissipativity theory-based perspective on discrete-time neural dynamical systems.
 As the paper's main result, we pose sufficient conditions for the dissipativity of deep neural dynamical systems.
We do so by leveraging an exact local linearization of deep neural networks (DNNs), yielding pointwise affine maps (PWA).  
  The obtained PWA form yields a seamless way to analyze neural networks' dissipativity via energy stored in their local linear operator norms and energy supplied via their aggregate bias temrs. 
 Guided by the derived dissipativity conditions, we review a set of design practices for enforcing the stability of deep neural dynamics.
 We use the presented theory in numerical case studies to analyze the effects of weights, activations, bias terms, and depth on the dissipativity of overall dynamics of deep neural networks. 
  Additionally, we demonstrate the utility of the proposed method on the stability analysis of neural surrogate models of two nonlinear systems.
 We believe that the presented method can be a useful tool for designing and analyzing deep neural networks used for modeling and control of dynamical systems. 
 The focus of the future work will be on the closed-loop stability of systems with neural feedback policies.

\section*{Appendix}

\subsection{Mathematical Background} \label{sec:lin_alg}

\begin{definition}
Induced operator norm of a matrix $\mathbf{A} \in \mathbb{R}^{n\times m}$ is defined as:
 \begin{equation}
 ||\mathbf{A}||_p =  \max_{\mathbf{x} \neq 0} \frac{ ||\mathbf{A}\mathbf{x}||_p}{||\mathbf{x}||_p}  =   \max_{\|\mathbf{x}\|_p = 1} ||\mathbf{A}\mathbf{x}||_p, \ \ \forall \mathbf{x} \in \mathcal{X},
   \label{eq:operator_norm}
 \end{equation}
 where $\mathcal{X}$ is a compact normed vector space, and  $|| \cdot ||_p: \mathbb{R}^{n} \to \mathbb{R}$ represents vector $p$-norm inducing the matrix norm $ || \mathbf{A} ||_p: \mathbb{R}^{n \times m} \to \mathbb{R}$. 
 The matrix norm is sub-additive:
  \begin{equation}
  \label{eq:operator_norm_subadd}
|| \mathbf{A} + \mathbf{B} ||_p \le || \mathbf{A}||_p + ||\mathbf{B} ||_p.
\end{equation} 
\end{definition}
\begin{definition}
Induced $p$-norm $||\cdot||_p: \mathbb{R}^{n \times m} \to \mathbb{R}$ is called submultiplicative if it satisfies~\cite{matrix_norms1983}:
 \begin{equation}
  \label{eq:operator_norm_submultiplicative}
|| \mathbf{A} \mathbf{B} ||_p \le || \mathbf{A}||_p  ||\mathbf{B} ||_p.
\end{equation} 
\end{definition}
\begin{theorem}
Lets have a vector norm $||\cdot||_p: \mathbb{R}^{n} \to \mathbb{R}$ defined for all $n$ with corresponding induced matrix norm defined as~\eqref{eq:operator_norm},
 then the submultiplicatity of the matrix norm~  \eqref{eq:operator_norm_submultiplicative} is satisfied for any $\mathbf{A} \in \mathbb{R}^{m \times k}$ and
  $\mathbf{B} \in \mathbb{R}^{k \times n}$~\cite{LinAlg2020}.
  \label{thm:submult}
\end{theorem}

\begin{definition}
A function ${f}: \mathbf{R} \rightarrow \mathbf{R}$ is  Lipschitz continuous if there exists a positive real constant $K$ such that, for all real ${x}_1$ and ${x}_2$, following holds:
      \begin{equation}
  \label{eq:Lipschitz}
     |{f}({x}_1) - {f}({x}_2) | \le K | {x}_1 - {x}_2 |
 \end{equation}
\end{definition}

\begin{definition}
Given a metric space $(\mathcal{X},d)$, a mapping $T: \mathcal{X}\to \mathcal{X}$ is a contraction if there exists a constant $c \in [0, 1)$ and a metric $d$ such that following holds:
 \begin{equation}
  \label{eq:contraction}
      d(T( \mathbf{x_1}),T( \mathbf{x_2})) \le  c d( \mathbf{x_1}, \mathbf{x_2}), \  \forall  \mathbf{x_1},  \mathbf{x_2} \in   \mathcal{X}
\end{equation} 
\end{definition}


\begin{theorem}
\label{thm:banch}
Banach fixed-point theorem. Consider a  non-empty complete metric space $(\mathcal{X},d)$, then every contraction~\eqref{eq:contraction} converges towards a unique fixed point  $T( \mathbf{x}_{\text{ss}}) = \mathbf{x}_{\text{ss}}$.
\end{theorem}

\begin{definition}
{Asymptotic stability of a dynamical system}: 
for any bounded initial condition $\mathbf{x}_0$, the state of the dynamical system converges to its equilibrium point $\bar{\mathbf{x}}$:
  \begin{equation}
  \label{eq:asymptotic}
     || \mathbf{x}_0 - \bar{\mathbf{x}} || < \epsilon, \epsilon>0 \implies \lim_{t \rightarrow \infty} ||\mathbf{x}_t|| =  \bar{\mathbf{x}}
 \end{equation}
  \label{def:asymptotic}
\end{definition}


\subsection{Details on Experimental Setup} \label{sec:appendix_f}

As a holistic analysis of the dynamics and stability of neural networks, we employ a grid search to compare across the axes of structured linear maps, eigenvalue constraints, activation functions, network depths, and biases. For each combination of factorization, spectral constraints, and network depths, we swap out activation functions and toggle bias usage while keeping model parameters fixed to control for the effects of initializations when comparing how these latter two hyperparameters affect network dynamics and stability.
\textbf{Linear Map Factorizations and Constraints $(\lambda_\text{min}, \lambda_\text{max})$}
\begin{itemize}
    \item Gershgorin Disc, real eigenvalues: $(-1.50, -1.10)$, $(0.00, 1.00)$, $(0.99, 1.00)$, $(0.99, 1.01)$, $(0.99, 1.10)$, $(1.00, 1.01)$, $(1.10, 1.50)$
    \item Gershgorin Disc, complex eigenvalues: $(-1.50, -1.10)$, $(0.00, 1.00)$, $(0.99, 1.00)$, $(0.99, 1.01)$, $(0.99, 1.10)$, $(1.00, 1.01)$, $(1.10, 1.50)$
    \item Spectral: $(-1.50, -1.10)$, $(0.00, 1.00)$, $(0.99, 1.00)$, $(0.99, 1.01)$, $(0.99, 1.10)$, $(1.00, 1.01)$, $(1.10, 1.50)$
    \item Perron-Frobenius: $(1.00, 1.00)$
    \item Unstructured: no spectral constraints
\end{itemize}
\textbf{Network Depth} 1, 4, 8 layers \\
\textbf{Activation Functions} ReLU, SELU, GELU, Tanh, logistic sigmoid, Softplus \\
\textbf{Bias} Enabled, disabled
\\ 
\noindent Overall, we generated 828 different models for examination. From this sampling, we then selected the models which were most relevant to our analysis based on their dynamical behavior. Table \ref{tab:fig_hparams} outlines the hyperparameters of each of the curated examples shown in the figures.
\begin{table}[h]
    \centering
    \begin{tabular}{llllll} \toprule
    Figure                     & Weights   & Act. & Layer & Eigenvalue  & Bias \\ 
          &    &  &  &  Range &  \\ \midrule
    Fig. \ref{fig:weights}     & Gersh & Tanh       & 8         & (0.99, 1.00)     & N    \\
                               & Spectral        & Tanh       & 8         & (0.99, 1.10)     & N    \\
                               & Gersh & Softplus   & 8         & (1.00, 1.01)     & Y    \\ \midrule
    Fig. \ref{fig:spectral_radii} & Gersh & ReLU       & 4         & (0.00, 1.00)     & N    \\
                               & Gersh & Tanh       & 4         & (0.00, 1.00)     & N    \\
                               & Spectral        & SELU       & 4         & (0.00, 1.00)           & N    \\ 
                               & Gersh & Sigmoid    & 4         & (0.00, 1.00)           & N    \\ \midrule
    Fig. \ref{fig:bias}        & Gersh & ReLU       & 1         & (0.00, 1.00)           & Y    \\
                               & Gersh & ReLU       & 1         & (0.00, 1.00)           & N    \\ \midrule
    Fig. \ref{fig:depth}       & Gersh & GELU       & 1         & (0.00, 1.00)           & N    \\
                               & Gersh & GELU       & 4         & (0.00, 1.00)           & N    \\
                               & Gersh & GELU       & 8         & (0.00, 1.00)           & N    \\
                               & Gersh & GELU       & 1         & (0.99, 1.00)        & N    \\
                               & Gersh & GELU       & 4         & (0.99, 1.00)        & N    \\
                               & Gersh & GELU       & 8         & (0.99, 1.00)        & N    \\
                               & Gersh & GELU       & 1         & (1.20, 1.40)      & N    \\
                               & Gersh & GELU       & 4         & (1.20, 1.40)      & N    \\
                               & Gersh & GELU       & 8         & (1.20, 1.40)      & N    \\ \midrule
    Figs.  & Gersh & ReLU & 1 & (0.00, 1.00) & N \\
             \ref{fig:attractors},\ref{fig:AstarRegions},\ref{fig:attractors_spectra}      & Spectral        & Tanh       & 8         & (0.99, 1.10)      & N    \\
                & PF & ReLU       & 1         & (1.00, 1.00)           & N    \\
               & Gersh & SELU       & 1         & (-1.50, -1.10)     & Y    \\
                & Spectral        & SELU       & 8         & (0.99, 1.10)      & N    \\
                & Gersh & Softplus   & 1         & (0.99, 1.10)      & N    \\
    \bottomrule
    \end{tabular}
    \caption{Layer factorizations, activation functions, network depths, spectral constraints, and bias usage of the models depicted in each figure.}
    \label{tab:fig_hparams}
\end{table}

\section*{Acknowledgements}

We acknowledge the engineering work of our former colleague Mia Skomski who helped us with the empirical studies presented in this paper.
We want to thank Wenceslao Shaw Cortez and James Koch for reviewing the manuscript and helping to improve the technical quality of the presented ideas.
Also, we want to thank our anonymous reviewers for their constructive feedback and suggestions.

This research was supported by 
the Mathematics for Artificial Reasoning in Science (MARS) initiative
via the Laboratory Directed Research and Development (LDRD) investments at Pacific Northwest National Laboratory (PNNL). PNNL is a multi-program national laboratory
operated for the U.S. Department of Energy (DOE) by Battelle Memorial Institute under Contract
No. DE-AC05-76RL0-1830.

\bibliographystyle{IEEEtran}
\bibliography{references}

\begin{thebibliography}{10}
\providecommand{\url}[1]{#1}
\csname url@samestyle\endcsname
\providecommand{\newblock}{\relax}
\providecommand{\bibinfo}[2]{#2}
\providecommand{\BIBentrySTDinterwordspacing}{\spaceskip=0pt\relax}
\providecommand{\BIBentryALTinterwordstretchfactor}{4}
\providecommand{\BIBentryALTinterwordspacing}{\spaceskip=\fontdimen2\font plus
\BIBentryALTinterwordstretchfactor\fontdimen3\font minus
  \fontdimen4\font\relax}
\providecommand{\BIBforeignlanguage}[2]{{%
\expandafter\ifx\csname l@#1\endcsname\relax
\typeout{** WARNING: IEEEtran.bst: No hyphenation pattern has been}%
\typeout{** loaded for the language `#1'. Using the pattern for}%
\typeout{** the default language instead.}%
\else
\language=\csname l@#1\endcsname
\fi
#2}}
\providecommand{\BIBdecl}{\relax}
\BIBdecl

\bibitem{Byrnes94}
C.~Byrnes and W.~Lin, ``Losslessness, feedback equivalence, and the global
  stabilization of discrete-time nonlinear systems,'' \emph{IEEE Transactions
  on Automatic Control}, vol.~39, no.~1, pp. 83--98, 1994.

\bibitem{WILLEMS2007134}
J.~C. Willems, ``Dissipative dynamical systems,'' \emph{European Journal of
  Control}, vol.~13, no.~2, pp. 134--151, 2007.

\bibitem{HILL1980327}
D.~J. Hill and P.~J. Moylan, ``Dissipative dynamical systems: Basic
  input-output and state properties,'' \emph{Journal of the Franklin
  Institute}, vol. 309, no.~5, pp. 327--357, 1980.

\bibitem{Chellaboina2005}
V.~Chellaboina, W.~Haddad, and A.~Kamath, ``Dynamic dissipativity theory for
  stability of nonlinear feedback dynamical systems,'' in \emph{Proceedings of
  the 44th IEEE Conference on Decision and Control}, 2005, pp. 4748--4753.

\bibitem{Rajpurohit2017}
T.~Rajpurohit and W.~M. Haddad, ``Dissipativity theory for nonlinear stochastic
  dynamical systems,'' \emph{IEEE Transactions on Automatic Control}, vol.~62,
  no.~4, pp. 1684--1699, 2017.

\bibitem{Kottenstette2010}
N.~Kottenstette and P.~J. Antsaklis, ``Relationships between positive real,
  passive dissipative, amp; positive systems,'' in \emph{Proceedings of the
  2010 American Control Conference}, 2010, pp. 409--416.

\bibitem{Rawlings2012}
J.~B. Rawlings, D.~Angeli, and C.~N. Bates, ``Fundamentals of economic model
  predictive control,'' in \emph{2012 IEEE 51st IEEE Conference on Decision and
  Control (CDC)}, 2012, pp. 3851--3861.

\bibitem{Diehl2011}
M.~Diehl, R.~Amrit, and J.~B. Rawlings, ``A lyapunov function for economic
  optimizing model predictive control,'' \emph{IEEE Transactions on Automatic
  Control}, vol.~56, no.~3, pp. 703--707, 2011.

\bibitem{Muller2015}
M.~A. Müller, D.~Angeli, and F.~Allgöwer, ``On necessity and robustness of
  dissipativity in economic model predictive control,'' \emph{IEEE Transactions
  on Automatic Control}, vol.~60, no.~6, pp. 1671--1676, 2015.

\bibitem{Sosanya2022}
A.~Sosanya and S.~Greydanus, ``Dissipative hamiltonian neural networks:
  Learning dissipative and conservative dynamics separately,'' \emph{CoRR},
  vol. abs/2201.10085, 2022.

\bibitem{shamma1990LPV}
J.~S. {Shamma} and M.~{Athans}, ``Analysis of gain scheduled control for
  nonlinear plants,'' \emph{IEEE Transactions on Automatic Control}, vol.~35,
  no.~8, pp. 898--907, 1990.

\bibitem{boyd1994LDI}
S.~Boyd, L.~El~Ghaoui, E.~Feron, and V.~Balakrishnan, \emph{4. Linear
  Differential Inclusions}, 1994, pp. 51--59.

\bibitem{tanaka1996ldi}
K.~{Tanaka}, ``An approach to stability criteria of neural-network control
  systems,'' \emph{IEEE Transactions on Neural Networks}, vol.~7, no.~3, pp.
  629--642, 1996.

\bibitem{matusik2020ldiNN3}
R.~Matusik, A.~Nowakowski, S.~Plaskacz, and A.~Rogowski, ``Finite-time
  stability for differential inclusions with applications to neural networks,''
  \emph{SIAM Journal on Control and Optimization}, vol.~58, no.~5, pp.
  2854--2870, 2020.

\bibitem{he2014ldiNN}
X.~{He}, C.~{Li}, T.~{Huang}, C.~{Li}, and J.~{Huang}, ``A recurrent neural
  network for solving bilevel linear programming problem,'' \emph{IEEE
  Transactions on Neural Networks and Learning Systems}, vol.~25, no.~4, pp.
  824--830, 2014.

\bibitem{limanond1998ldiNN2}
S.~{Limanond} and J.~{Si}, ``Neural network-based control design: an lmi
  approach,'' \emph{IEEE Transactions on Neural Networks}, vol.~9, no.~6, pp.
  1422--1429, 1998.

\bibitem{arora2016ReLuDNN}
R.~Arora, A.~Basu, P.~Mianjy, and A.~Mukherjee, ``Understanding deep neural
  networks with rectified linear units,'' \emph{CoRR}, vol. abs/1611.01491,
  2016.

\bibitem{hanin2019complexity}
B.~Hanin and D.~Rolnick, ``Complexity of linear regions in deep networks,''
  ser. Proceedings of Machine Learning Research, K.~Chaudhuri and
  R.~Salakhutdinov, Eds., vol.~97.\hskip 1em plus 0.5em minus 0.4em\relax Long
  Beach, California, USA: PMLR, 09--15 Jun 2019, pp. 2596--2604.

\bibitem{hanin2019deep}
------, ``Deep relu networks have surprisingly few activation patterns,'' in
  \emph{Advances in Neural Information Processing Systems 32}, H.~Wallach,
  H.~Larochelle, A.~Beygelzimer, F.~d\textquotesingle Alch\'{e}-Buc, E.~Fox,
  and R.~Garnett, Eds.\hskip 1em plus 0.5em minus 0.4em\relax Curran
  Associates, Inc., 2019, pp. 361--370.

\bibitem{Wang2016}
S.~Wang, A.~rahman Mohamed, R.~Caruana, J.~A. Bilmes, M.~Philipose,
  M.~Richardson, K.~Geras, G.~Urban, and Özlem Aslan, ``Analysis of deep
  neural networks with extended data jacobian matrix,'' in \emph{ICML}, 2016,
  pp. 718--726.

\bibitem{gehr2018ai2}
T.~Gehr, M.~Mirman, D.~Drachsler-Cohen, P.~Tsankov, S.~Chaudhuri, and
  M.~Vechev, ``Ai2: Safety and robustness certification of neural networks with
  abstract interpretation,'' in \emph{2018 IEEE Symposium on Security and
  Privacy (SP)}.\hskip 1em plus 0.5em minus 0.4em\relax IEEE, 2018, pp. 3--18.

\bibitem{Robinson2020}
H.~Robinson, A.~Rasheed, and O.~San, ``Dissecting deep neural networks,''
  \emph{CoRR}, vol. abs/1910.03879, 2019.

\bibitem{ludwig2014eigenvalue}
O.~Ludwig, U.~Nunes, and R.~Araujo, ``Eigenvalue decay: A new method for neural
  network regularization,'' \emph{Neurocomputing}, vol. 124, pp. 33--42, 2014.

\bibitem{Schmidt2021}
D.~Schmidt, G.~Koppe, M.~Beutelspacher, and D.~Durstewitz, ``Inferring
  dynamical systems with long-range dependencies through line attractor
  regularization,'' \emph{CoRR}, vol. abs/1910.03471, 2019.

\bibitem{IMEXnet2019}
E.~Haber, K.~Lensink, E.~Treister, and L.~Ruthotto, ``Imexnet: {A} forward
  stable deep neural network,'' \emph{CoRR}, vol. abs/1903.02639, 2019.

\bibitem{NAISnet2018}
M.~Ciccone, M.~Gallieri, J.~Masci, C.~Osendorfer, and F.~J. Gomez, ``Nais-net:
  Stable deep networks from non-autonomous differential equations,''
  \emph{CoRR}, vol. abs/1804.07209, 2018.

\bibitem{HamiltonianDNNs2019}
S.~Greydanus, M.~Dzamba, and J.~Yosinski, ``Hamiltonian neural networks,'' in
  \emph{Advances in Neural Information Processing Systems}, H.~Wallach,
  H.~Larochelle, A.~Beygelzimer, F.~d\textquotesingle Alch\'{e}-Buc, E.~Fox,
  and R.~Garnett, Eds., vol.~32.\hskip 1em plus 0.5em minus 0.4em\relax Curran
  Associates, Inc., 2019.

\bibitem{cranmer2020lagrangian}
M.~Cranmer, S.~Greydanus, S.~Hoyer, P.~Battaglia, D.~Spergel, and S.~Ho,
  ``Lagrangian neural networks,'' 2020.

\bibitem{Rusch2021}
T.~K. Rusch and S.~Mishra, ``Coupled oscillatory recurrent neural network
  (cornn): An accurate and (gradient) stable architecture for learning long
  time dependencies,'' \emph{CoRR}, vol. abs/2010.00951, 2020.

\bibitem{engelken2020lyapunov}
R.~Engelken, F.~Wolf, and L.~Abbott, ``Lyapunov spectra of chaotic recurrent
  neural networks,'' \emph{arXiv preprint arXiv:2006.02427}, 2020.

\bibitem{vogt2020lyapunov}
R.~Vogt, M.~P. Touzel, E.~Shlizerman, and G.~Lajoie, ``On lyapunov exponents
  for rnns: Understanding information propagation using dynamical systems
  tools,'' \emph{arXiv preprint arXiv:2006.14123}, 2020.

\bibitem{gler2019robust}
B.~Güler, A.~Laignelet, and P.~Parpas, ``Towards robust and stable deep
  learning algorithms for forward backward stochastic differential equations,''
  2019.

\bibitem{haber2017stable}
E.~Haber and L.~Ruthotto, ``Stable architectures for deep neural networks,''
  \emph{Inverse Problems}, vol.~34, no.~1, p. 014004, 2017.

\bibitem{NIPS2019_9292}
G.~Manek and J.~Z. Kolter, ``Learning stable deep dynamics models,'' in
  \emph{Advances in Neural Information Processing Systems 32}, H.~Wallach,
  H.~Larochelle, A.~Beygelzimer, F.~d\textquotesingle Alch\'{e}-Buc, E.~Fox,
  and R.~Garnett, Eds.\hskip 1em plus 0.5em minus 0.4em\relax Curran
  Associates, Inc., 2019, pp. 11\,126--11\,134.

\bibitem{ghorbani2019investigation}
B.~Ghorbani, S.~Krishnan, and Y.~Xiao, ``An investigation into neural net
  optimization via hessian eigenvalue density,'' \emph{arXiv preprint
  arXiv:1901.10159}, 2019.

\bibitem{le1991eigenvalues}
Y.~Le~Cun, I.~Kanter, and S.~A. Solla, ``Eigenvalues of covariance matrices:
  Application to neural-network learning,'' \emph{Physical Review Letters},
  vol.~66, no.~18, p. 2396, 1991.

\bibitem{goel2017eigenvalue}
S.~Goel and A.~Klivans, ``Eigenvalue decay implies polynomial-time learnability
  for neural networks,'' in \emph{Advances in Neural Information Processing
  Systems}, 2017, pp. 2192--2202.

\bibitem{vrabie2009neural}
D.~Vrabie and F.~Lewis, ``Neural network approach to continuous-time direct
  adaptive optimal control for partially unknown nonlinear systems,''
  \emph{Neural Networks}, vol.~22, no.~3, pp. 237--246, 2009.

\bibitem{vamvoudakis2010neuroCntrl}
K.~G. Vamvoudakis and F.~L. Lewis, ``Online actor–critic algorithm to solve
  the continuous-time infinite horizon optimal control problem,''
  \emph{Automatica}, vol.~46, no.~5, pp. 878 -- 888, 2010.

\bibitem{vamvoudakis2015NNfbCntrl}
K.~G. Vamvoudakis, F.~Lewis, and S.~S. Ge, \emph{Neural Networks in Feedback
  Control Systems}.\hskip 1em plus 0.5em minus 0.4em\relax American Cancer
  Society, 2015, ch.~23, pp. 1--52.

\bibitem{Pennington46342}
J.~Pennington and P.~Worah, ``Nonlinear random matrix theory for deep
  learning,'' \emph{Journal of Statistical Mechanics: Theory and Experiment},
  2019.

\bibitem{louart2017random}
C.~Louart, Z.~Liao, and R.~Couillet, ``{A random matrix approach to neural
  networks},'' \emph{The Annals of Applied Probability}, vol.~28, no.~2, pp.
  1190 -- 1248, 2018.

\bibitem{liao2018dynamics}
Z.~Liao and R.~Couillet, ``The dynamics of learning: A random matrix
  approach,'' in \emph{Proceedings of the 35th International Conference on
  Machine Learning}, ser. Proceedings of Machine Learning Research, J.~Dy and
  A.~Krause, Eds., vol.~80.\hskip 1em plus 0.5em minus 0.4em\relax PMLR, 10--15
  Jul 2018, pp. 3072--3081.

\bibitem{Kozachkov2020}
L.~Kozachkov, M.~Lundqvist, J.-J. Slotine, and E.~K. Miller, ``Achieving stable
  dynamics in neural circuits,'' \emph{PLOS Computational Biology}, vol.~16,
  no.~8, pp. 1--15, 08 2020.

\bibitem{Revay2019}
M.~Revay and I.~R. Manchester, ``Contracting implicit recurrent neural
  networks: Stable models with improved trainability,'' \emph{CoRR}, vol.
  abs/1912.10402, 2019.

\bibitem{Revay2021}
M.~{Revay}, R.~{Wang}, and I.~R. {Manchester}, ``A convex parameterization of
  robust recurrent neural networks,'' \emph{IEEE Control Systems Letters},
  vol.~5, no.~4, pp. 1363--1368, 2021.

\bibitem{Pauli9319198}
P.~Pauli, A.~Koch, J.~Berberich, P.~Kohler, and F.~Allgöwer, ``Training robust
  neural networks using lipschitz bounds,'' \emph{IEEE Control Systems
  Letters}, vol.~6, pp. 121--126, 2022.

\bibitem{Fazlyab2019}
M.~Fazlyab, A.~Robey, H.~Hassani, M.~Morari, and G.~J. Pappas, \emph{Efficient
  and Accurate Estimation of Lipschitz Constants for Deep Neural
  Networks}.\hskip 1em plus 0.5em minus 0.4em\relax Advances in Neural
  Information Processing Systems, 2019.

\bibitem{erichson2021lipschitz}
N.~B. Erichson, O.~Azencot, A.~Queiruga, L.~Hodgkinson, and M.~W. Mahoney,
  ``Lipschitz recurrent neural networks,'' in \emph{International Conference on
  Learning Representations}, 2021.

\bibitem{Zhang2014}
H.~{Zhang}, Z.~{Wang}, and D.~{Liu}, ``A comprehensive review of stability
  analysis of continuous-time recurrent neural networks,'' \emph{IEEE
  Transactions on Neural Networks and Learning Systems}, vol.~25, no.~7, 2014.

\bibitem{laje_robust_2013}
R.~Laje and D.~V. Buonomano, ``Robust timing and motor patterns by taming chaos
  in recurrent neural networks,'' vol.~16, no.~7, pp. 925--933.

\bibitem{bonassi2020lstm}
F.~Bonassi, E.~Terzi, M.~Farina, and R.~Scattolini, ``Lstm neural networks:
  Input to state stability and probabilistic safety verification,'' 2020.

\bibitem{mhammedi2017efficient}
Z.~Mhammedi, A.~Hellicar, A.~Rahman, and J.~Bailey, ``Efficient orthogonal
  parametrisation of recurrent neural networks using householder reflections,''
  in \emph{International Conference on Machine Learning}.\hskip 1em plus 0.5em
  minus 0.4em\relax PMLR, 2017, pp. 2401--2409.

\bibitem{zhang2018stabilizing}
J.~Zhang, Q.~Lei, and I.~Dhillon, ``Stabilizing gradients for deep neural
  networks via efficient svd parameterization,'' in \emph{International
  Conference on Machine Learning}, 2018, pp. 5806--5814.

\bibitem{tuor2020constrained}
A.~Tuor, J.~Drgona, and D.~Vrabie, ``Constrained neural ordinary differential
  equations with stability guarantees,'' \emph{arXiv preprint
  arXiv:2004.10883}, 2020.

\bibitem{chang2019antisymmetricrnn}
B.~Chang, M.~Chen, E.~Haber, and E.~H. Chi, ``Antisymmetricrnn: A dynamical
  system view on recurrent neural networks,'' 2019.

\bibitem{rajan2006random}
K.~Rajan and L.~F. Abbott, ``Eigenvalue spectra of random matrices for neural
  networks,'' \emph{Phys. Rev. Lett.}, vol.~97, p. 188104, Nov 2006.

\bibitem{lechner2020gershgorin}
M.~Lechner, R.~Hasani, D.~Rus, and R.~Grosu, ``Gershgorin loss stabilizes the
  recurrent neural network compartment of an end-to-end robot learning
  scheme,'' in \emph{2020 International Conference on Robotics and Automation
  (ICRA). IEEE}, 2020.

\bibitem{drgona2021stochastic}
J.~Drgoňa, S.~Mukherjee, J.~Zhang, F.~Liu, and M.~Halappanavar, ``On the
  stochastic stability of deep markov models,'' in \emph{Advances in Neural
  Information Processing Systems}, 2021.

\bibitem{BFb0109870}
A.~Bemporad and M.~Morari, ``Robust model predictive control: A survey,'' in
  \emph{Robustness in identification and control}, A.~Garulli and A.~Tesi,
  Eds.\hskip 1em plus 0.5em minus 0.4em\relax London: Springer London, 1999,
  pp. 207--226.

\bibitem{WangLH15}
Z.~Wang, Q.~Ling, and T.~S. Huang, ``Learning deep encoders,'' \emph{CoRR},
  vol. abs/1509.00153, 2015.

\bibitem{Ng_norms2004}
A.~Y. Ng, ``Feature selection, l1 vs. l2 regularization, and rotational
  invariance,'' ser. ICML '04.\hskip 1em plus 0.5em minus 0.4em\relax New York,
  NY, USA: Association for Computing Machinery, 2004, p.~78.

\bibitem{Duchi2008}
J.~Duchi, S.~Shalev-Shwartz, Y.~Singer, and T.~Chandra, ``Efficient projections
  onto the l1-ball for learning in high dimensions,'' ser. ICML '08.\hskip 1em
  plus 0.5em minus 0.4em\relax New York, NY, USA: Association for Computing
  Machinery, 2008, p. 272–279.

\bibitem{L1_dnns2018}
S.~Wu, G.~Li, L.~Deng, L.~Liu, Y.~Xie, and L.~Shi, ``L1-norm batch
  normalization for efficient training of deep neural networks,'' \emph{CoRR},
  vol. abs/1802.09769, 2018.

\bibitem{NeyshaburTS15}
B.~Neyshabur, R.~Tomioka, and N.~Srebro, ``Norm-based capacity control in
  neural networks,'' \emph{CoRR}, vol. abs/1503.00036, 2015.

\bibitem{hoffer2019norm}
E.~Hoffer, R.~Banner, I.~Golan, and D.~Soudry, ``Norm matters: efficient and
  accurate normalization schemes in deep networks,'' 2019.

\bibitem{specNorm2018}
T.~Miyato, T.~Kataoka, M.~Koyama, and Y.~Yoshida, ``Spectral normalization for
  generative adversarial networks,'' in \emph{International Conference on
  Learning Representations}, 2018.

\bibitem{farnia2018generalizable}
F.~Farnia, J.~Zhang, and D.~Tse, ``Generalizable adversarial training via
  spectral normalization,'' in \emph{International Conference on Learning
  Representations}, 2019.

\bibitem{Pascanu2012}
R.~Pascanu, T.~Mikolov, and Y.~Bengio, ``Understanding the exploding gradient
  problem,'' \emph{CoRR}, vol. abs/1211.5063, 2012.

\bibitem{skomski2021constrained}
E.~Skomski, S.~Vasisht, C.~Wight, A.~Tuor, J.~Drgona, and D.~Vrabie,
  ``Constrained block nonlinear neural dynamical models,'' vol.
  arXiv:2101.01864.\hskip 1em plus 0.5em minus 0.4em\relax American Control
  Conference (ACC), 2021.

\bibitem{Varga_Gersgorin2004}
R.~Varga, \emph{Geršgorin and His Circles}.\hskip 1em plus 0.5em minus
  0.4em\relax Springer, Berlin, Heidelberg, 01 2004, vol.~36.

\bibitem{Kolen5264952}
J.~F. {Kolen} and S.~C. {Kremer}, \emph{Gradient Flow in Recurrent Nets: The
  Difficulty of Learning LongTerm Dependencies}, 2001, pp. 237--243.

\bibitem{Eliasmith2007}
C.~Eliasmith, ``{A}ttractor network,'' \emph{Scholarpedia}, vol.~2, no.~10, p.
  1380, 2007, revision \#91016.

\bibitem{mccluskey_bendixson_dulac_1998}
C.~C. {McCluskey} and J.~S. Muldowney, ``Bendixson-dulac criteria for
  difference equations,'' \emph{Journal of Dynamics and Differential
  Equations}, vol.~10, no.~4, pp. 567--575, 1998.

\bibitem{matrix_norms1983}
M.~Malek-Shahmirzadi, ``A characterization of certain classes of matrix
  norms,'' \emph{Linear and Multilinear Algebra}, vol.~13, no.~2, pp. 97--99,
  1983.

\bibitem{LinAlg2020}
R.~van~de Geijn and M.~Myers, \emph{Advanced Linear Algebra: Foundations to
  Frontiers}.\hskip 1em plus 0.5em minus 0.4em\relax Creative Commons
  NonCommercial (CC BY-NC), 2020.

\end{thebibliography}

\end{document}